\documentclass{article}

\PassOptionsToPackage{numbers, compress}{natbib}

\usepackage[preprint]{neurips_2025}

\usepackage[utf8]{inputenc} 
\usepackage[T1]{fontenc}    
\usepackage{hyperref}       
\usepackage{url}            
\usepackage{booktabs}       
\usepackage{amsfonts}       
\usepackage{nicefrac}       
\usepackage{microtype}      
\usepackage{xcolor}         
\usepackage[most]{tcolorbox}
\usepackage{multirow}       
\usepackage{wrapfig}        
\usepackage{amsthm}         
\usepackage{amsmath}        
\usepackage{enumitem}       
\usepackage{placeins}

\title{SGD as Free Energy Minimization: \\ A Thermodynamic View on Neural Network Training}

\author{
Ildus Sadrtdinov$\bf{}^{1}$\thanks{Equal contribution.}\:, 
Ivan Klimov$\bf{}^{1}$\footnotemark[1]\:, 
Ekaterina Lobacheva$^{2,3}$\:, 
Dmitry Vetrov$\bf{}^{1}$\\
${}^1$ Constructor University, Bremen \quad
${}^2$ Mila -- Quebec AI Institute, \quad
${}^3$ Université de Montréal\\
\texttt{\{isadrtdinov,iklimov,dvetrov\}@constructor.university}\\
\texttt{ekaterina.lobacheva@mila.quebec}
}

\begin{document}

\maketitle

\begin{abstract}
We present a thermodynamic interpretation of the stationary behavior of stochastic gradient descent (SGD) under fixed learning rates (LRs) in neural network training.
We show that SGD implicitly minimizes a free energy function $F=U-TS$, balancing training loss $U$ and the entropy of the weights distribution $S$, with temperature $T$ determined by the LR.
This perspective offers a new lens on why high LRs prevent training from converging to the loss minima and how different LRs lead to stabilization at different loss levels.
We empirically validate the free energy framework on both underparameterized (UP) and overparameterized (OP) models.
UP models consistently follow free energy minimization, with temperature increasing monotonically with LR, while for OP models, the temperature effectively drops to zero at low LRs, causing SGD to minimize the loss directly and converge to an optimum.
We attribute this mismatch to differences in the signal-to-noise ratio of stochastic gradients near optima, supported by both a toy example and neural network experiments.
\end{abstract}

\section{Introduction}

Modern neural networks (NNs) are typically trained using stochastic gradient descent (SGD) and its numerous variants~\cite{KingmaB14,loshchilov2018decoupled,chen2023symbolic}.
Understanding the behavior of these optimization methods is essential for developing high-performing models.
One of the most important hyperparameters in such iterative algorithms is the learning rate (LR), 
as it affects both the convergence speed and the properties of the final solution~\cite{bengio2012practical,li2019towards,you2019does,kodryan2022training,andriushchenko2023sgd}.
Notably, when using a fixed LR, SGD often reduces the training loss only to a certain level~\cite{kodryan2022training,andriushchenko2023sgd}; further improvement typically requires LR annealing~\cite{li2019towards,Sadrtdinov2024largelr}.
Interestingly, different fixed LRs cause SGD to plateau at different training loss levels (see Figure~\ref{fig:loss_entropy_iters} and~\cite{kodryan2022training}), indicating distinct stationary behaviors depending on the LR.

Direct theoretical analysis of individual training trajectories in complex models is intractable, so we adopt an alternative approach and introduce a thermodynamic framework to describe stationary behavior of SGD under different LRs.
Specifically, we show that training with SGD minimizes a form of \emph{free energy} given by $F=U-TS$, which determines a trade-off between the training loss $U$ and the entropy of weights distribution $S$ with the temperature parameter $T$, controlled by the LR.
This perspective offers a thermodynamic interpretation for two commonly observed phenomena:
\begin{enumerate}
\item Why training fails to converge to an optimum at high LRs --- \emph{because SGD minimizes the free energy rather than the training loss alone}.
\item Why different LRs yield different final training losses --- \emph{because such stationary distributions minimize the free energy at temperatures corresponding to these LRs}.
\end{enumerate}
Although these research questions have been partly discussed in the literature (see Section~\ref{sec:related}), our work presents a novel, self-consistent perspective, supported by extensive empirical results.

\begin{figure}
    \centering
    \addtolength{\tabcolsep}{-0.4em}

    \begin{tabular}{ccccc}
        \multicolumn{2}{c}{Undeparameterized (UP)} & \multicolumn{2}{c}{Overparameterized (OP)} & \multirow{2}{*}{\includegraphics[width=0.09\textwidth]{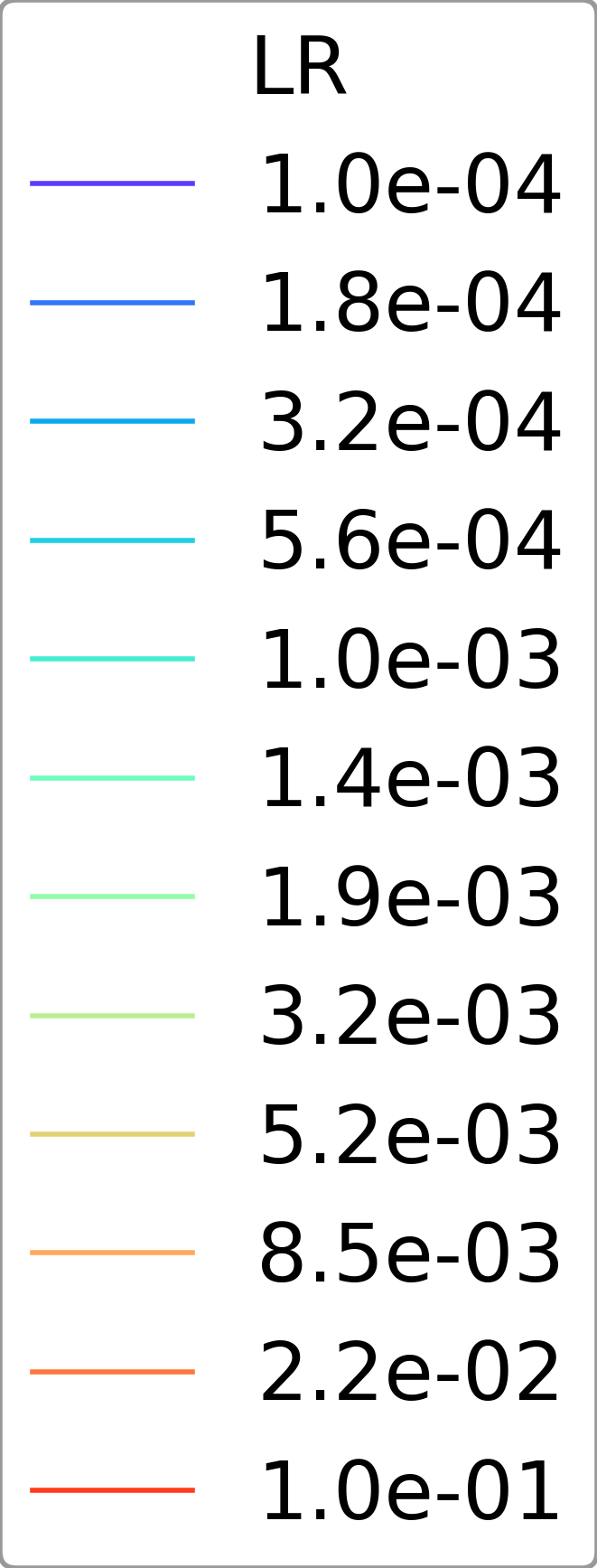}}\\
        \includegraphics[width=0.21\textwidth]{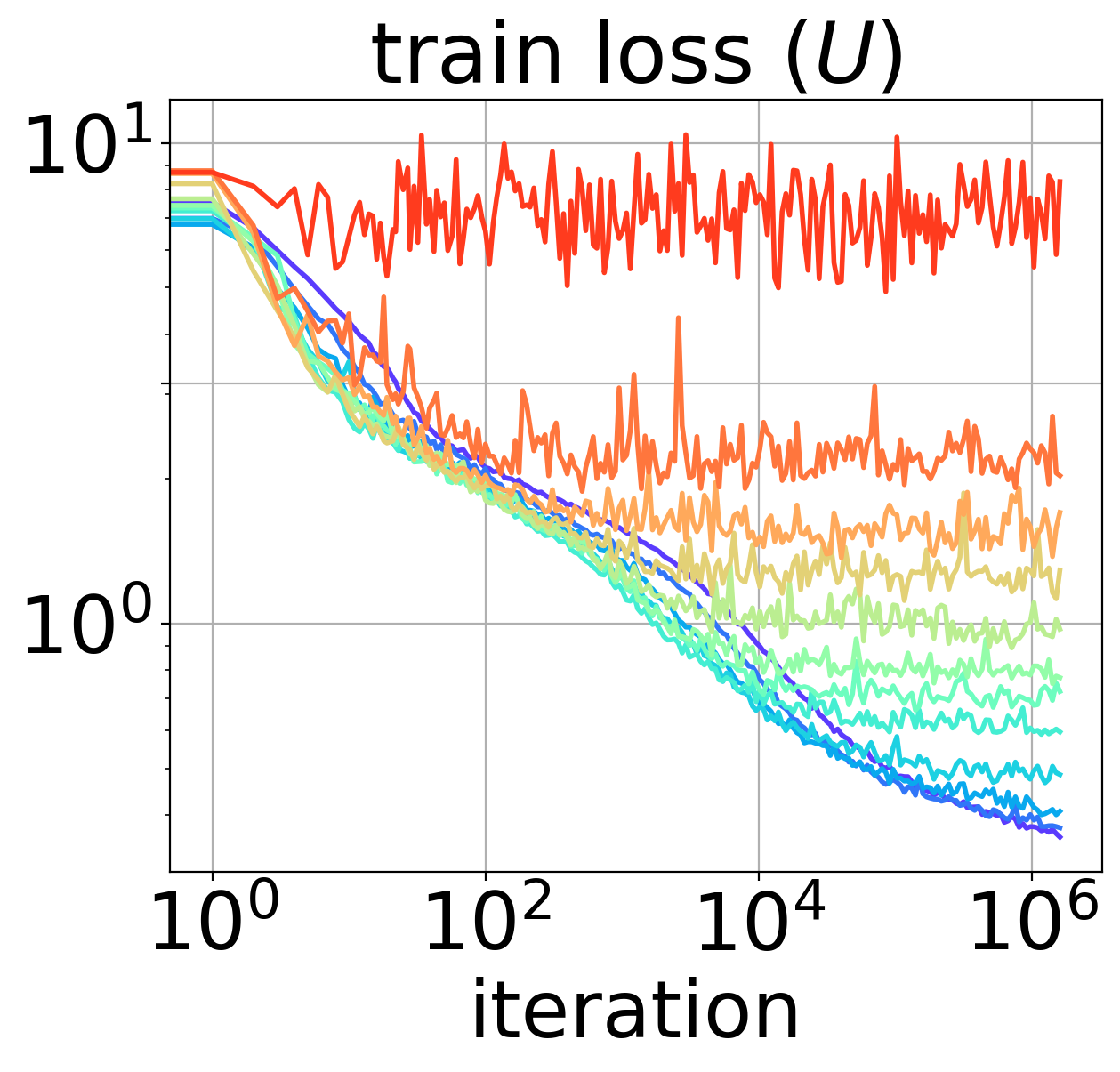} & 
        \includegraphics[width=0.203\textwidth]{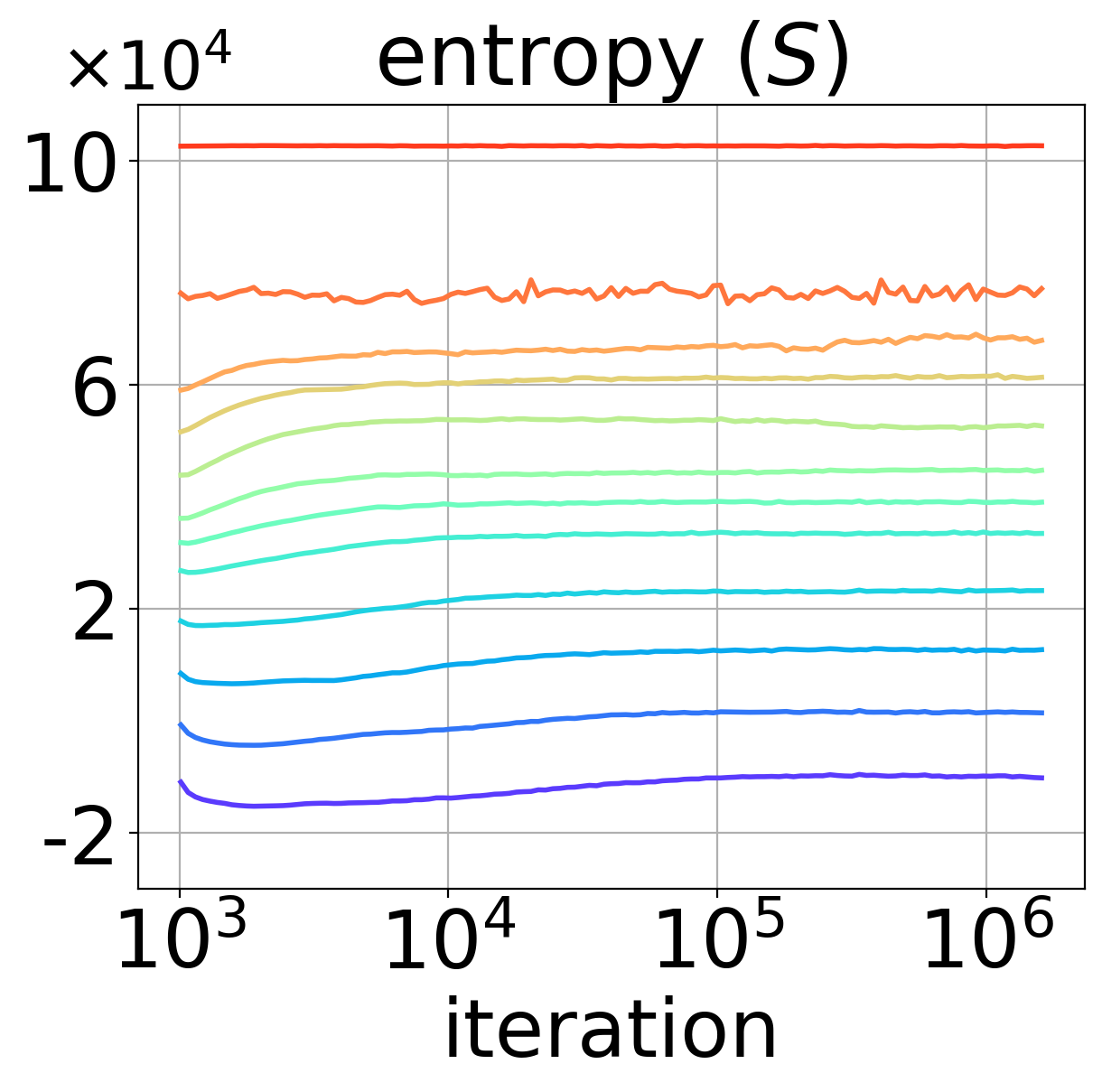} & 
        \includegraphics[width=0.218\textwidth]{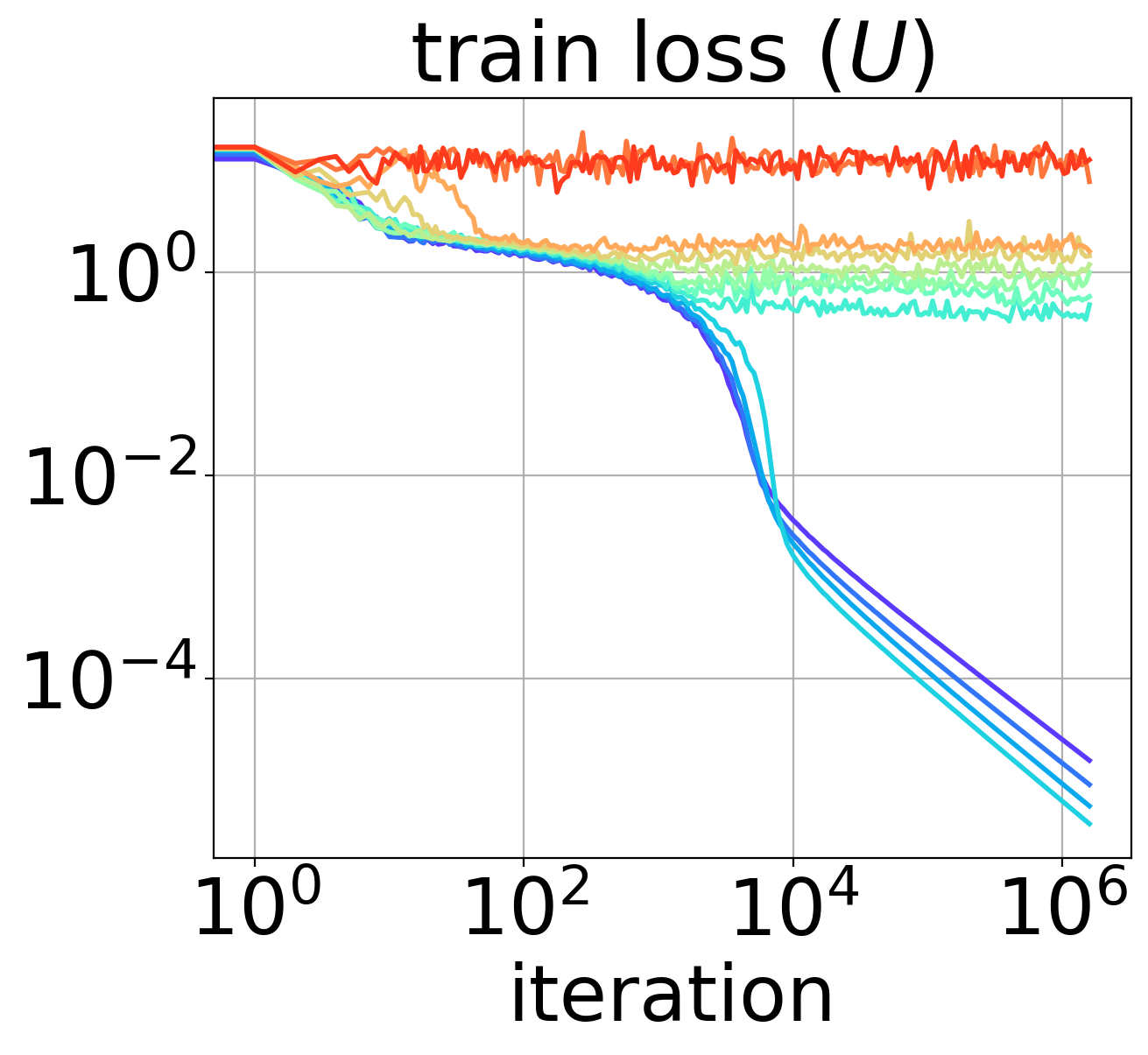} & 
        \includegraphics[width=0.208\textwidth]{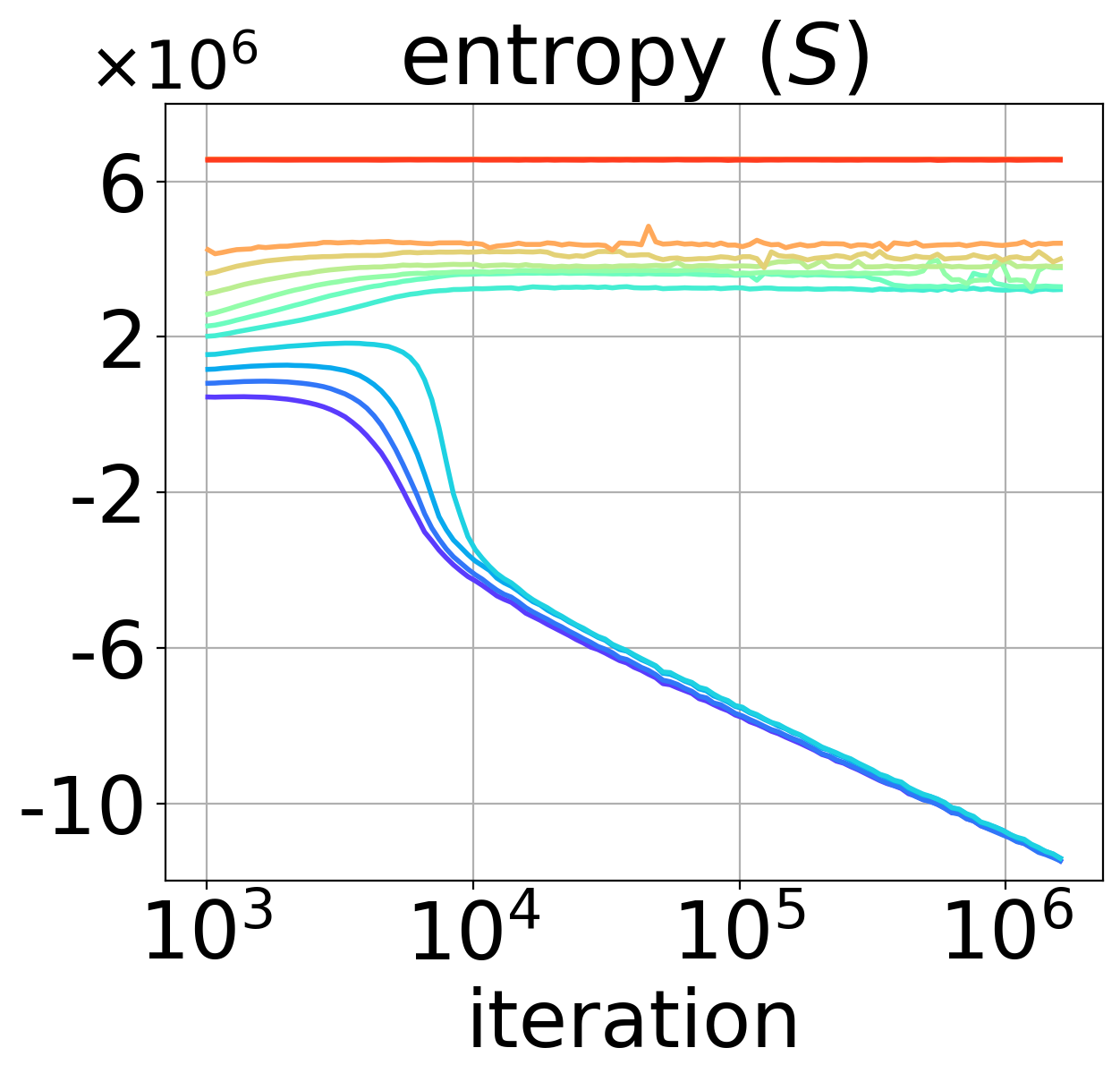} & 
    \end{tabular}
    \caption{Stationary loss and entropy for UP and OP settings. ConvNet on CIFAR-10.}
    \label{fig:loss_entropy_iters}
\end{figure}

\begin{figure}[!b]
\end{figure}

Since the behavior of a model depends heavily on its ability to fit the entire training dataset~\cite{interp_sgd}, we consider NNs of different size.
In practice, many modern models—especially large-scale foundation models~\cite{wav2vec, gpt_3, bert, clip}—are \emph{underparameterized (UP)}, meaning they lack capacity to memorize all training samples due to the vast size of the datasets involved.
Despite this, UP models exhibit remarkable properties, such as adaptability to new tasks and domains~\cite{adaptation2, adaptation1}, as well as zero-shot inference~\cite{gpt_3, lrfm_benchmark, clip}, making them widely adopted in both research and industry.
Conversely, \emph{overparameterized (OP)} models contain significantly more parameters than training examples, enabling them to achieve perfect accuracy on the training set.
Such models also appear in practice—for instance, during transfer learning~\cite{big_transfer}, where large pretrained models are fine-tuned on small datasets.
OP models demonstrate phenomena not typically observed in UP models, such as grokking~\cite{grokking}, double descent~\cite{double_descent2, double_descent}, and mode connectivity~\cite{draxler2018essentially, garipov2018loss}. 
In this work, we compare the stationary behavior of UP and OP models through our thermodynamic lens.
Our \textbf{key contributions} are:
\begin{enumerate}
    \item We propose and empirically validate a free energy minimization framework that characterizes the stationary behavior of SGD.
    \item We empirically construct the temperature function for both under- and overparameterized models, showing that in the underparameterized case, temperature remains positive and increases smoothly with learning rate, while in the overparameterized case, it drops to zero at small learning rates, enabling convergence to an optimum.
    \item We explain this mismatch between the parameterization regimes as a consequence of differences in the signal-to-noise ratio (SNR) of stochastic gradients near optima.
\end{enumerate}
Our code is available at: \url{https://github.com/isadrtdinov/sgd-free-energy}

\section{Related works}
\label{sec:related}

{\bf SGD dynamics} 
SGD training differs from loss minimization via gradient flow in two key ways: stochasticity and the use of a finite non-zero LR. Both have been shown to introduce implicit regularization, preventing SGD from minimizing the original loss directly. A finite LR leads to regularization on the norm of the full gradient~\cite{barrett2021implicit} and weight norms~\cite{pmlr-v139-liu21ad,ziyin2022strength}, with stronger effects at higher LRs. Meanwhile, stochasticity induces regularization on the variance of stochastic gradients, which becomes more pronounced with smaller batch sizes~\cite{smith2021on}. In practice, these implicit regularization effects help SGD converge to flatter, more generalizable optima~\cite{barrett2021implicit,smith2021on,10.5555/3524938.3525778}, and may explain its superior generalization compared to full-batch gradient descent~\cite{geiping2022stochastic}. Our thermodynamic perspective is closely related to this view, as we show that SGD optimizes a free energy function that includes a noise-dependent entropy term, rather than minimizing the loss alone.

In OP models where zero loss is achievable, SGD with a sufficiently small finite LR can converge to a minimum point~\cite{4576,pmlr-v89-nacson19a,zou2018stochasticgradientdescentoptimizes,du2018gradient}. 
However, at larger LRs, which are common in practice, SGD instead reaches a stationary regime where the loss stabilizes at a non-zero level determined by the LR~\cite{kodryan2022training,andriushchenko2023sgd,NEURIPS2024_29496c94}. In UP settings, even small LRs lead to stationarity due to irreducible gradient noise~\cite{pmlr-v84-chee18a}. This behavior has been studied by modeling SGD as a stochastic differential equation (SDE) and analyzing its stationary distribution~\cite{yaida2018fluctuationdissipation,pmlr-v139-liu21ad,10.1162/neco_a_01626,jastrzebski2017three}, which reflects a modified loss influenced by both gradient noise and curvature—an effect similar to implicit regularization. Empirically, loss stabilization caused by large LRs in the beginning of training has been shown to improve generalization by guiding training toward regions of the loss landscape containing only wide optima~\cite{Sadrtdinov2024largelr}, and promoting learning more stable and low-rank features~\cite{li2019towards,andriushchenko2023sgd,chen2023stochastic}.

{\bf Thermodynamic perspective} 
Several prior works interpret the stochasticity of SGD through the lens of thermodynamics by introducing a temperature parameter defined as $T\propto\eta/B$, where $\eta$ is the learning rate and $B$ is the batch size. This temperature, grounded in an optimization perspective, controls the magnitude of noise in the updates and leads to different training regimes~\cite{sgd_regimes_thermo} and convergence to minima of different sharpness~\cite{bayes_sgd_thermo}. \citet{chaudhari2018stochastic} further incorporate this notion into a free energy framework for SGD. In contrast, we develop a thermodynamic analogy based on the principle of free energy minimization, where the temperature $T$ is not predefined but emerges from the stationary behavior of the learning dynamics. Although empirically linked to the learning rate (and implicitly to batch size, which we hold fixed), our $T$ is defined through the trade-off between training loss and the entropy of the weight distribution at stationarity. This framework enables us to empirically test the free energy minimization hypothesis and analyze how the learning rate influences the temperature, and consequently, the training dynamics in UP and OP models.

Beyond SGD dynamics, several works explore thermodynamic perspectives in the context of generalization and define temperature based on the parameter-to-data ratio $N/P$~\cite{Zhang17112018}, or treat it as an explicit hyperparameter to force convergence to more generalizable minima~\cite{chaudhari2017entropysgd}. 
Thermodynamic analogies also appear in the context of representation learning~\cite{alemi2019therml,gao2020}, particularly through the lens of the Information Bottleneck (IB) principle~\cite{tishby2015learning}, which formalizes a trade-off between compression and predictive power of learned features.
Finally, Bayesian perspectives on learning, such as viewing SGD as approximate Bayesian inference~\cite{Mandt2017,nagayasu2023bayesianfreeenergydeep,Welling2011}, share structural similarities with thermodynamic formulations 
by focusing on the balance between a data-fit and a complexity term.

\section{SGD, thermodynamics and free energy}

In this section, we present the motivation for an analogy between training NNs with SGD and thermodynamics.
We begin by conceptually comparing SGD to its continuous counterpart---the gradient flow.
In the case of a full (i.~e., non-stochastic) gradient flow with an infinitesimal LR, it is well established that the dynamics converge to a stationary point of the loss function (i.e., a minimum or a saddle point)~\cite{grad_flow}.
This behavior resembles that of classical mechanical systems, where a system in stable equilibrium settles into a configuration that (locally) minimizes its total \emph{potential energy}.

However, the behavior of SGD differs significantly due to two key factors: (1) a finite non-zero LR and (2) stochasticity of gradients.
These factors introduce noise into the training, causing SGD to stabilize in a distribution near a stationary point rather than exactly at the minimum.
We denote the stationary distribution of model weights $w$ as $p(w)$.
In essence, the ''noise'' introduced by (1) and (2) prevents SGD from fully minimizing the loss function.
A similar phenomenon arises in thermodynamics, which accounts for the intrinsic thermal fluctuations of microscopic particles—another form of ''noise''.
Due to these fluctuations, the evolution of thermodynamic systems cannot be described solely by the minimization of \emph{internal energy}~$U$ (analogous to potential energy in mechanics).

Our central idea is to adopt this thermodynamic perspective to analyze the stationary behavior of SGD.
In our experiments, we fix the learning rate and batch size, which jointly control the noise level in SGD.
This setup corresponds to a thermodynamic system at constant temperature, where temperature intuitively represents the magnitude of chaotic thermal fluctuations. 
Specifically, we focus on the (Helmholtz) free energy, which is minimized at equilibrium in systems with fixed temperature~$T$ and volume~$V$. This quantity is defined as $F=U-TS$, where $S$ denotes the thermodynamic entropy.
Free energy shows how much of a system's internal energy is available to do useful work, excluding the energy lost due to entropy (i.~e., disorder in the system).
This concept is similar to stochastic optimization, where the noise prevents the training loss from decreasing to its minimum, so we employ the same formula to describe such kind of behavior.
In our analogy, we define:
\begin{itemize}
    \item the expected \emph{training loss} as internal energy $U=\mathbb{E}_{p(w)} [L(w)]$, where $L(w)$ is the training loss for weights $w$
    \item the differential \emph{entropy} of the distribution as thermodynamic entropy $S=-\mathbb{E}_{p(w)} [\log p(w)]$
\end{itemize}
The temperature~$T$ is treated as a function of the optimizer hyperparameters.
Specifically, we focus on the impact of the learning rate, while maintaining the batch size fixed across all experiments.
We further discuss the choice of Helmholtz free energy as an appropriate potential in Appendix~\ref{app:free_energy}.

The free energy $F$ captures a trade-off between minimizing the training loss and maximizing the entropy of the weight distribution.
At zero temperature, this reduces to pure loss minimization, corresponding to the behavior of gradient flow.
For low positive temperatures (i.~e., small LRs), the loss term dominates, and the resulting stationary distribution is localized around a loss minimum, yielding low values of both $U$ and $S$.
In contrast, for higher temperatures (i.e., larger LRs), the entropy term becomes more significant, leading to broader distributions with higher values of both $U$ and $S$.
These observations motivate what we refer to as the \textbf{free energy hypothesis}:
\begin{tcolorbox}[colback=blue!5!white,colframe=blue!75!black]
\centering
SGD with a fixed learning rate~$\eta$ minimizes the free energy $F = U - T S$, where $U$ is the expected training loss, $S$ is the entropy of the weight distribution, and $T = T(\eta)$ is an effective temperature that increases monotonically with the learning rate.
\end{tcolorbox}

We verify this hypothesis empirically through the following steps:
\begin{enumerate}
    \item We measure the training loss $U(\eta)$ and estimate the entropy $S(\eta)$ of the stationary distribution resulting from training with various fixed LRs~$\eta$.
    \item We demonstrate the existence of a monotonically increasing function~$T(\eta)$ such that the free energy $F = U - T(\eta) S$ attains its minimum at the observed values $U = U(\eta)$ and $S = S(\eta)$, among all stationary distributions corresponding to other LRs.
    \item Finally, we discuss the scope of applicability of the proposed framework.
\end{enumerate}

\section{Experimental setup and methodology}
\label{sec:methodology}

The goal of this study requires us to carefully fix the LR.
However, this is non-trivial in modern NNs, which often include normalization layers.
These layers induce \emph{scale-invariance} on the weights of the preceding layer, meaning the network's output depends only on the direction of the weight vector, but not its norm.
Despite this, the norm of scale-invariant parameters still influences training dynamics by altering the \emph{effective learning rate} (ELR)~\cite{lobacheva2021periodic,kodryan2022training}, even when the LR remains constant.
To control for this effect, we adopt the experimental setup proposed in prior work~\cite{kodryan2022training}: we train fully scale-invariant networks on the unit sphere using projected SGD with a fixed LR, which ensures that ELR is held constant throughout training.
In this setup, training dynamics follow one of the three regimes: (1) convergence, observed only for OP models, as it requires reaching 100\% training accuracy, (2) chaotic equilibrium with the stabilization of training loss, and (3) divergence, which is characterized by near random quality (as the fixed weight norm prevents overflow, what usually happens with regular networks at large LRs).
We train the ConvNet model on the CIFAR-10 dataset~\cite{cifar10} with $8$ and $64$ channels, for UP and OP setups, respectively.
We use a fixed batch size $B=128$ with independent (i.~e., non-epoch) sampling of batches and train for $t=1.6 \cdot 10^6$ iterations.
In Appendix~\ref{app:exp_results}, we also provide results for the ResNet-18~\cite{deep_resnet} architecture and the CIFAR-100 dataset~\cite{cifar100}.

While computing training loss is straightforward, estimating entropy is far more challenging.
In general, accurate entropy estimation requires a number of samples that grows exponentially with the dimension $D$ of the weight vector, making it infeasible for NNs. 
Instead, we use an approximate estimator that is assumed to correlate with the true entropy.
Specifically, we adapt an estimator based on a $k$-nearest neighbor ($k$-NN) graph constructed from the samples of the desired distribution~\cite{entropy}:
\[
S = D \left( \log L_k - \frac{D-1}{D} \log N \right),
\]
where $L_k$ is the total edge length in the $k$-NN graph and $N$ is the number of samples.
This formula estimates entropy up to an additive constant independent of the distribution, which, as we show later, is sufficient for inferring temperature.
In practice, we apply this estimator to segments of SGD trajectories using a sliding window of size $N=1000$, and the number of neighbors set to $k=50$.
This approach captures the contributions of both finite LR and stochastic noise to the entropy.

We now describe our protocol for temperature estimation.
For each learning rate $\eta$, we compute the loss $U(\eta)$ and entropy $S(\eta)$ at the stationary distribution.
Then, for a given LR $\eta^*$, we seek a scalar temperature $T(\eta^*)$ such that the function $F(\eta) = U(\eta) - T(\eta^*)  S(\eta)$ achieves its minimum at $\eta^*$.
This procedure corresponds to setting the temperature value to $T(\eta^*) = dU(\eta^*)/dS(\eta^*)$\footnote{A similar definition of temperature $T=(\partial U/\partial S)_V$ is well-established as the most general in physics~\cite{survey_td}.}.
The value $dS$ does not depend on the additive constant, so our entropy estimate is suitable here.
In practice, we determine $T(\eta^*)$ such that $F(\eta^*) \le \min_\eta F(\eta) + \varepsilon$, with a small $\varepsilon > 0$, to accommodate estimation inaccuracies and yield a confidence interval for $T$.
This way, we define a function $T(\eta)$ that satisfies the free energy minimization condition by construction.
To validate our free energy hypothesis, we need to show that this function (1) is well-defined over a broad range of LRs and (2) increases monotonically with LR.
A more detailed description of the experimental setup and metrics evaluation protocol is available in Appendix~\ref{app:exp_setup}.

\section{Free energy in neural networks}

\begin{figure}
    \centering
    \addtolength{\tabcolsep}{-0.4em}

    \begin{tabular}{ccc}
        \multicolumn{3}{c}{Underparameterized (UP)} \\
        \includegraphics[width=0.33\textwidth]{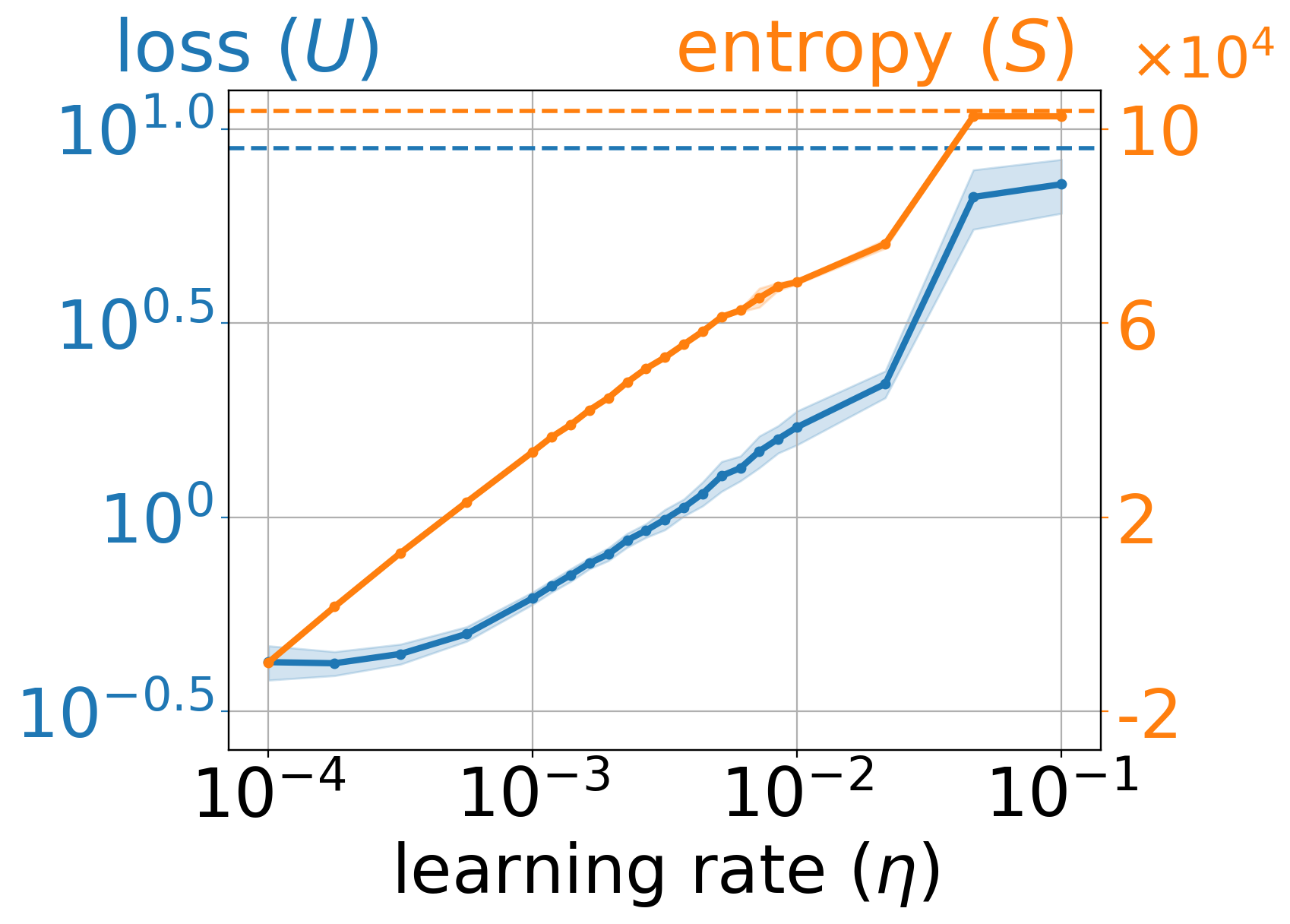} &
        \includegraphics[width=0.31\textwidth]{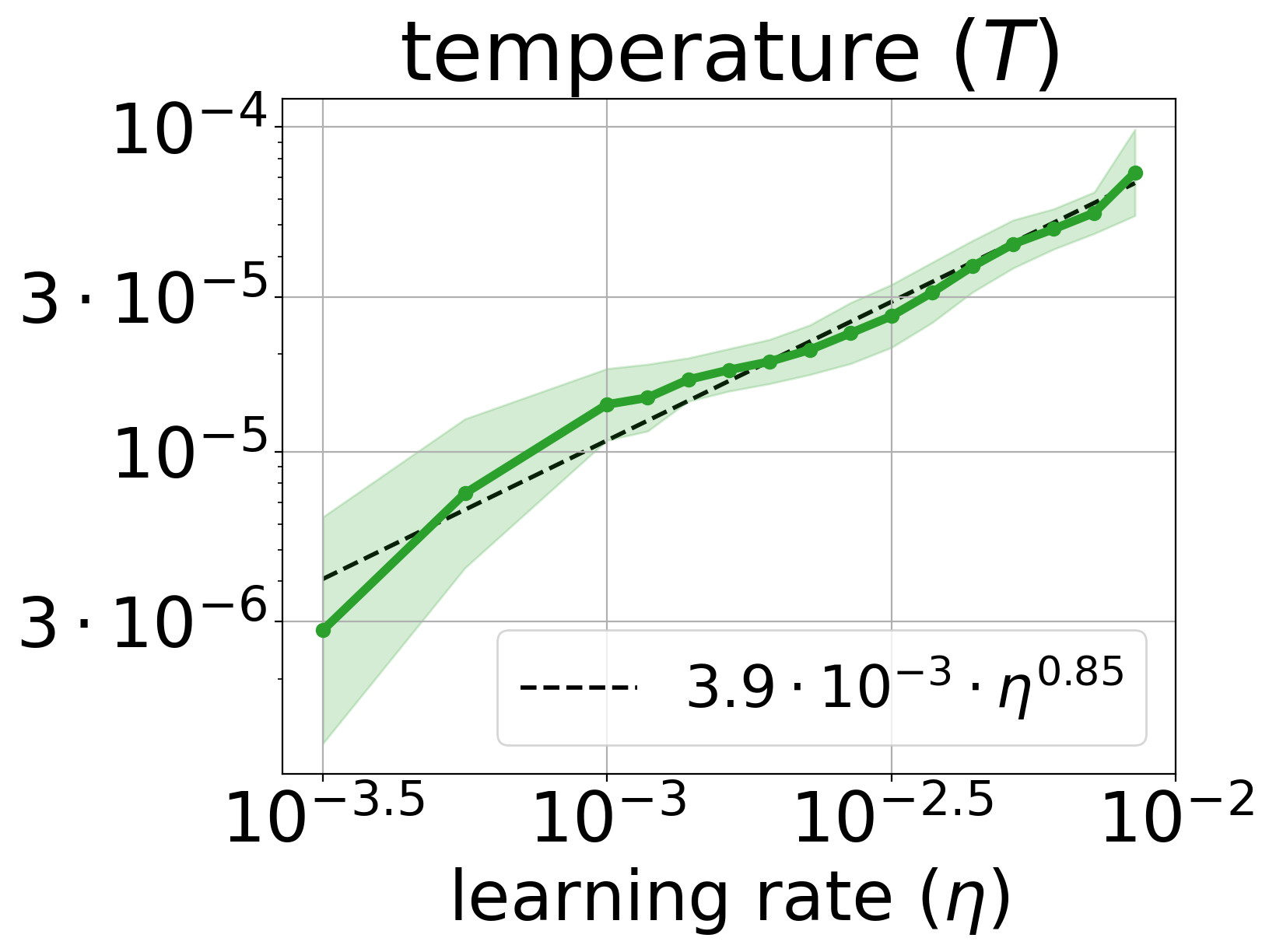} & 
        \includegraphics[width=0.35\textwidth]{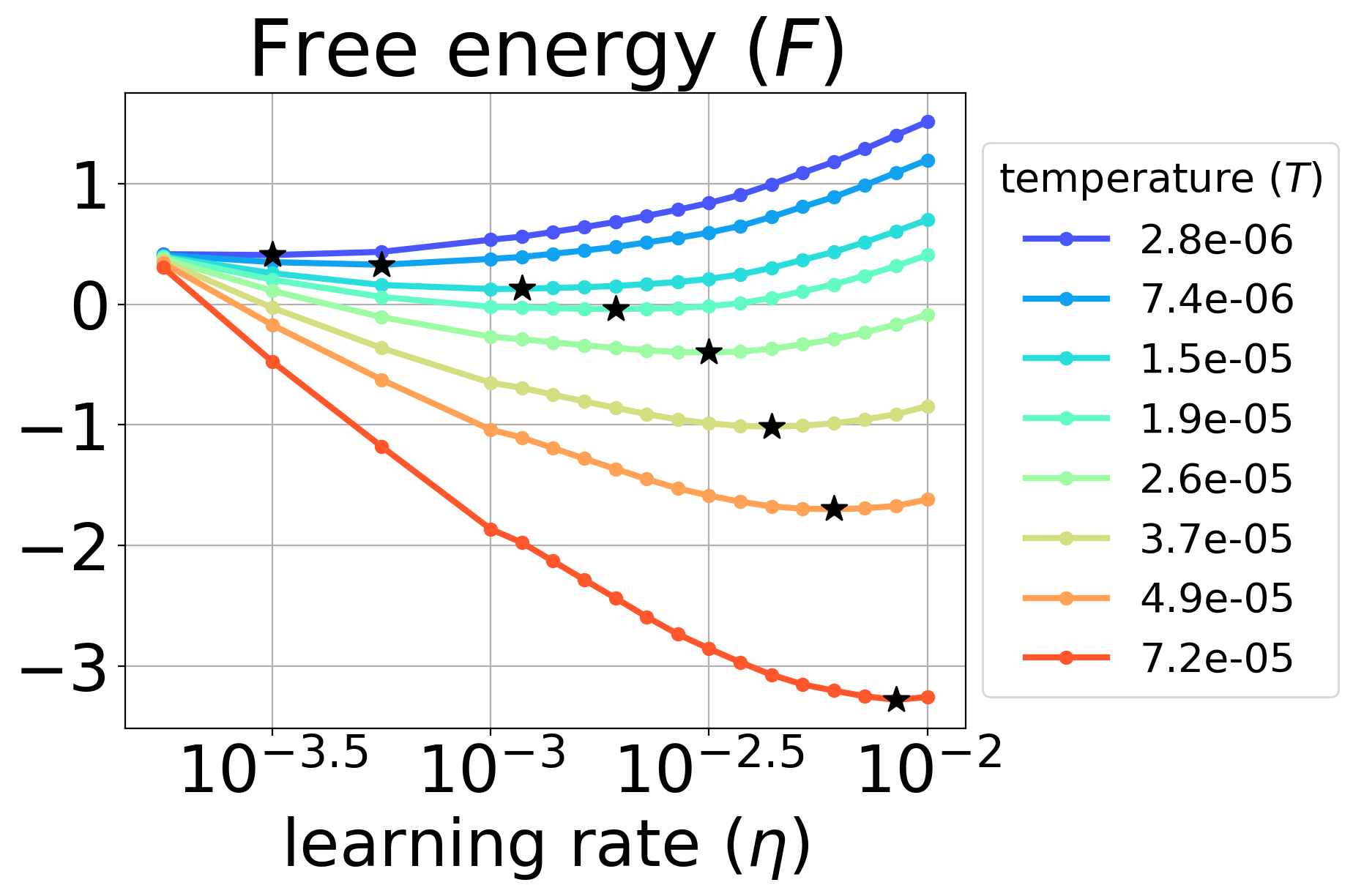} \\
        \multicolumn{3}{c}{Overparameterized (OP)} \\
        \includegraphics[width=0.33\textwidth]{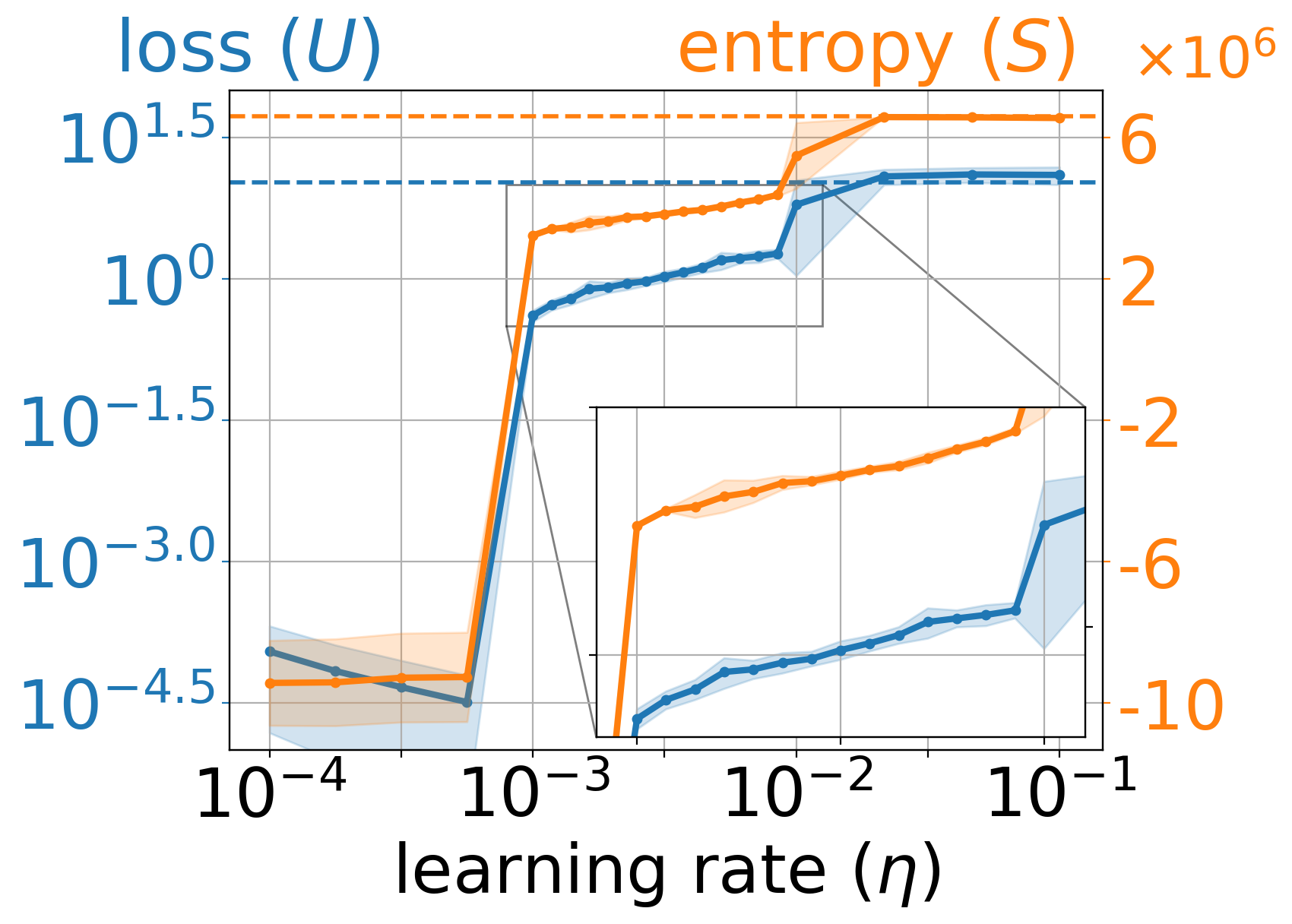} &
        \includegraphics[width=0.307\textwidth]{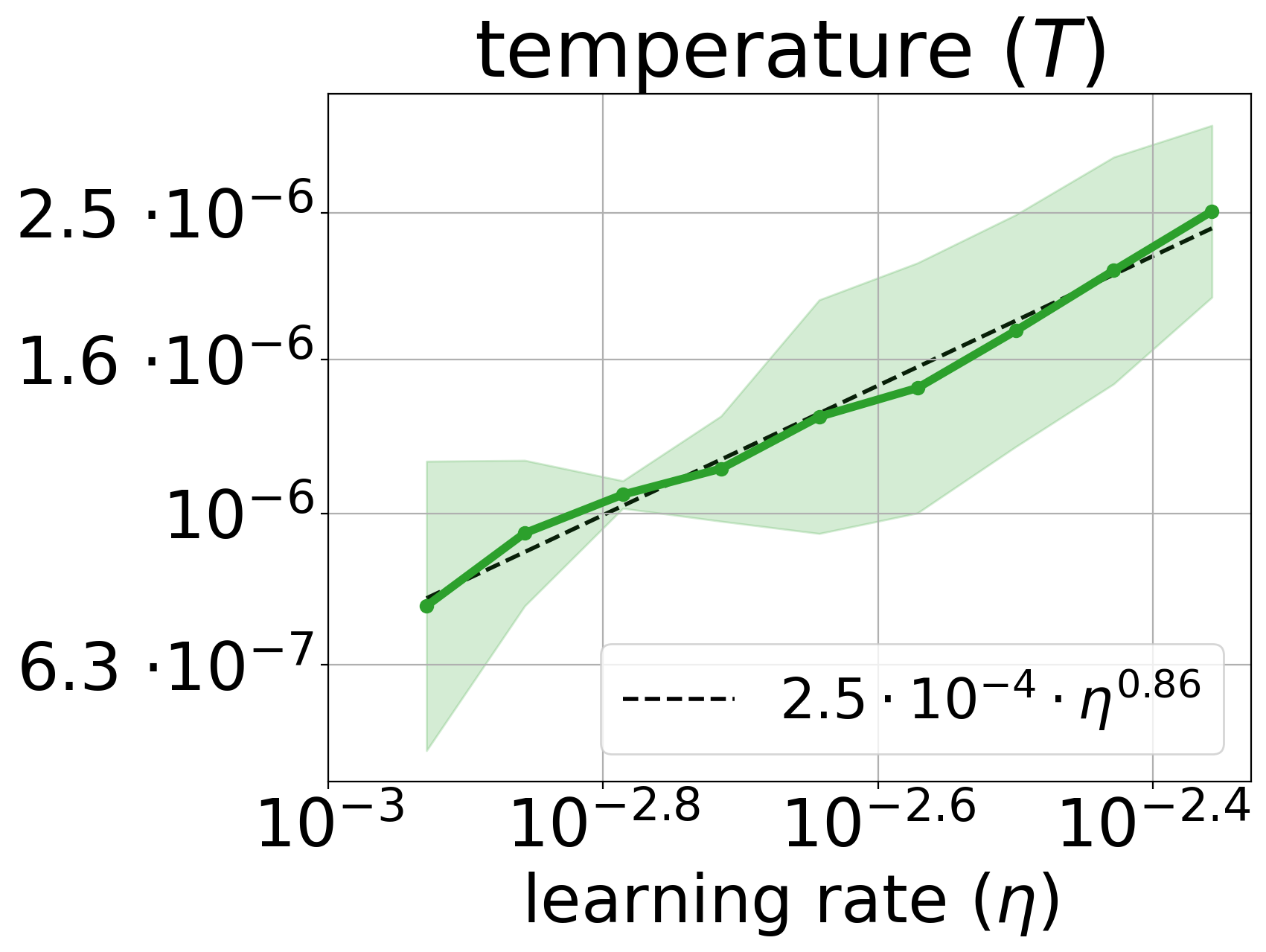} & 
        \includegraphics[width=0.35\textwidth]{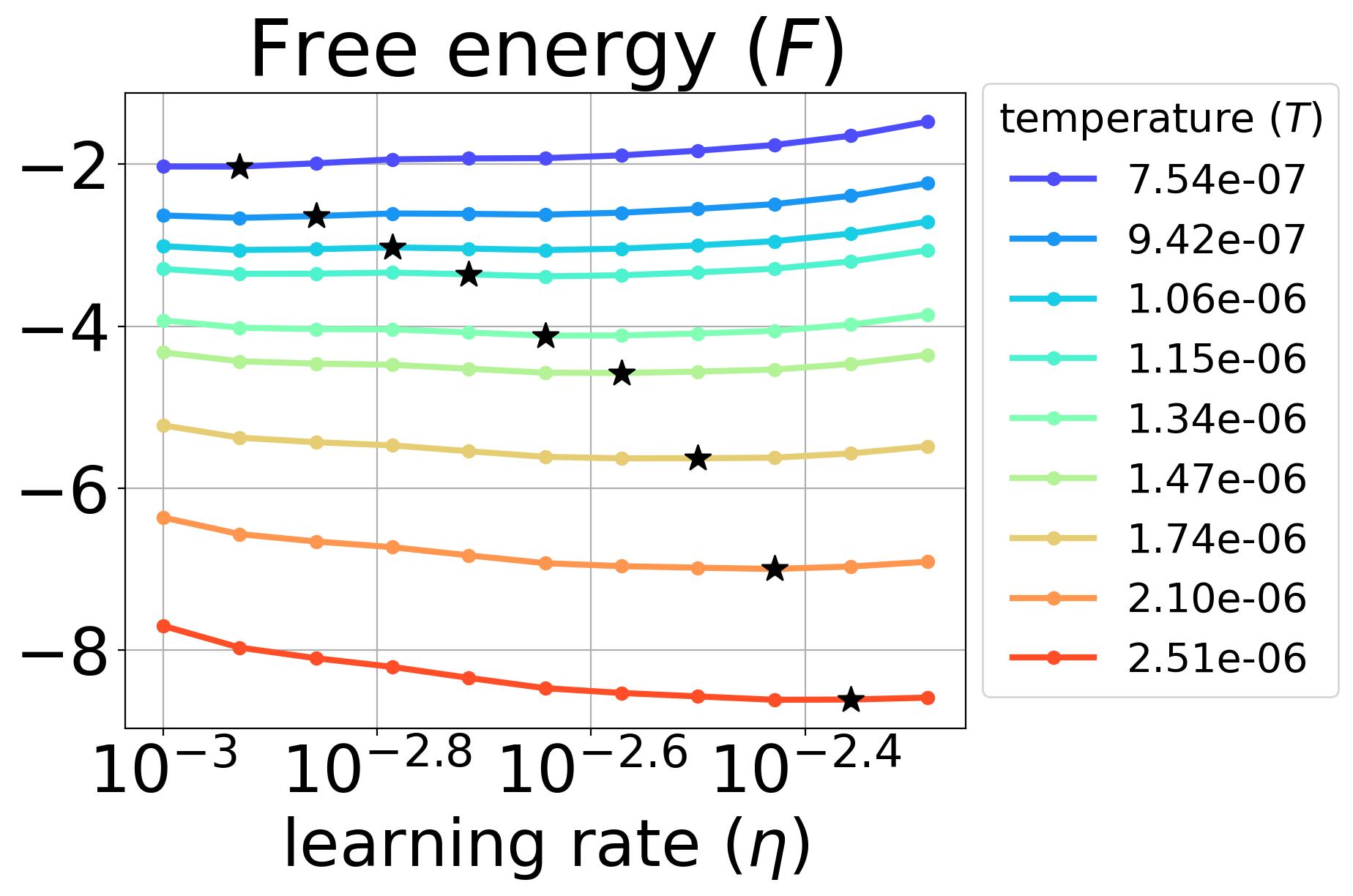} \\
    \end{tabular}
    \caption{Loss and entropy (left), estimated temperature (center) and free energy for different temperatures (right) vs. LR for UP (top row) and OP (bottom row) ConvNet on CIFAR-10. Orange and blue dashed lines denote the entropy and loss estimates of a uniform distribution on the sphere surface. Black dashed lines show power-law approximation to the temperature. Black stars indicate minima of free energy, which are achieved at LRs, corresponding to temperature values. Standard deviation for loss and entropy is computed over several iterations of the same training run, while translucent filling in the temperature plot indicates confidence intervals.}
    \label{fig:temperature_up_op}
\end{figure}

\subsection{Underparameterized setup}

We begin by discussing the UP setting.
As shown in the left part of Figure~\ref{fig:loss_entropy_iters}, both the training loss and entropy stabilize by the end of training, except for the smallest LRs, which require more iterations to reach the stationary state.
The top left subplot of Figure~\ref{fig:temperature_up_op} presents the stationary values of these two metrics across different LRs.
Both training loss and entropy increase monotonically with LR, eventually approaching the values expected under a uniform distribution over the sphere.
This behavior corresponds to regime 3 described in~\cite{kodryan2022training}, which is to some extent similar to a random walk on the unit sphere.
For smaller LRs, the $U$ and $S$ curves behave such that their combination—the free energy $F = U - T S$—is convex, with its minimum shifting depending on the value of $T$ (Figure~\ref{fig:temperature_up_op}, top right).
To construct the temperature function, we exclude certain LRs: (1) very small LRs that have not yet converged, and (2) large LRs corresponding to regime 3, where entropy has saturated and further increases in LR no longer affect the metrics.
We also smooth the $U$ and $S$ curves to reduce estimation noise (details provided in Appendix~\ref{app:exp_setup}).
The resulting temperature curve (Figure~\ref{fig:temperature_up_op}, top center) is well-defined and increases monotonically with LR, supporting the free energy hypothesis.
When the suitable temperature is substituted into the free energy formula, its minimum occurs at the corresponding LR (Figure~\ref{fig:temperature_up_op}, top right).

\subsection{Overparameterized setup}

\begin{wrapfigure}{r}{0.3\textwidth}
\vspace{-0.5cm}
  \begin{center}
    \includegraphics[width=0.3\textwidth]{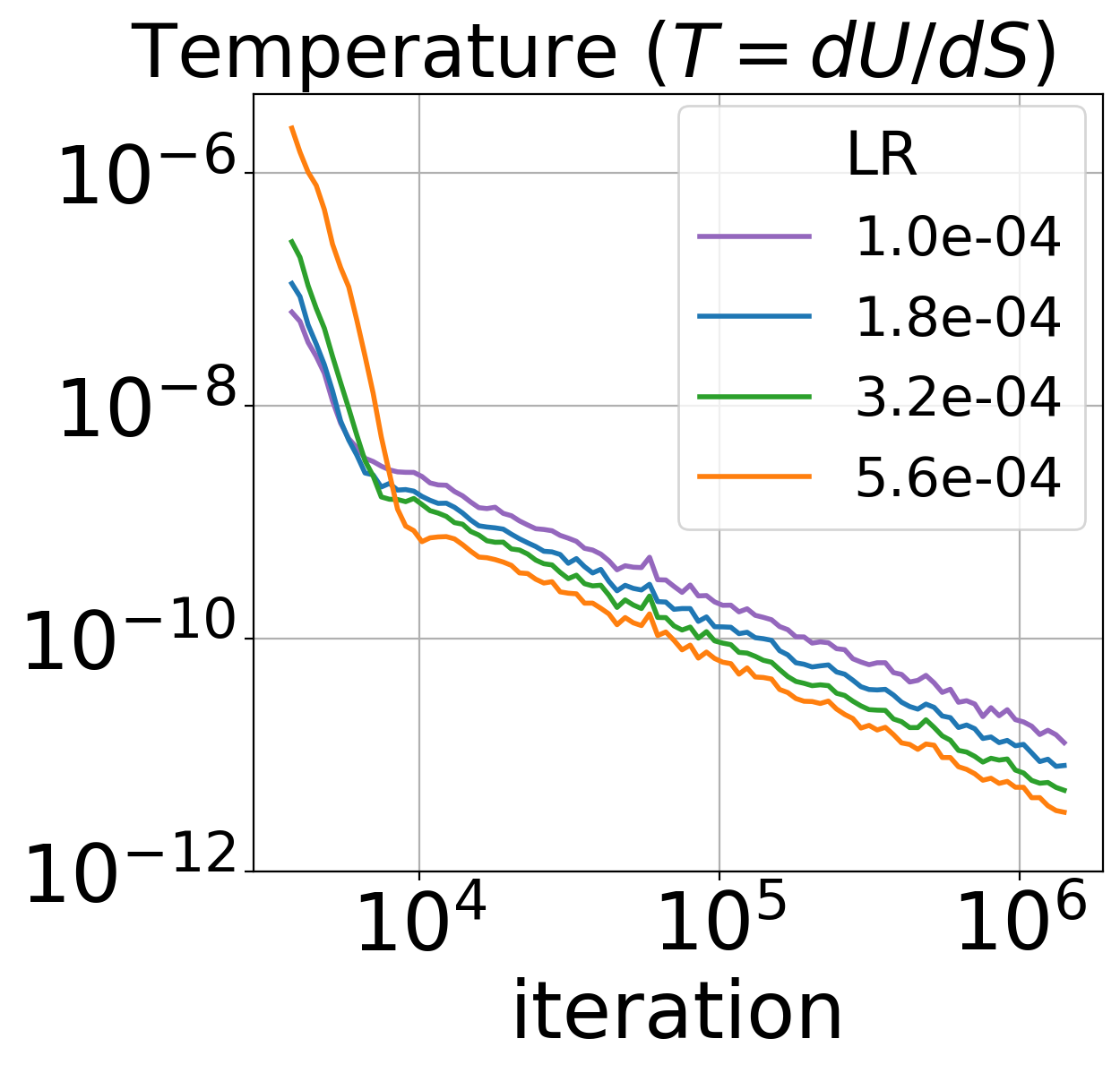}
  \end{center}
  \caption{Temperature decay for small LRs of the OP ConvNet on CIFAR-10.}
  \label{fig:temperature_reg1}
\end{wrapfigure}
We now turn to the analysis of the OP setting. 
For moderate LRs (from $10^{-3}$ to $10^{-2}$), the behavior closely resembles that observed in the UP setup: both training loss and entropy stabilize over iterations (Figure~\ref{fig:loss_entropy_iters}, right), and both metrics increase with LR (Figure~\ref{fig:temperature_up_op}, bottom left).
Since these runs reach a stationary state, we apply the temperature estimation protocol from the previous section (excluding several high LR values).
The resulting temperature function (Figure~\ref{fig:temperature_up_op}, bottom center) closely resembles that of the UP setting.
Even the exponents in the power-law fits are similar, though the proportionality constant is significantly smaller due to much higher entropy in the OP setting.
The corresponding free energy curves (Figure~\ref{fig:temperature_up_op}, bottom right) are also nearly convex, with their minima shifting across different LRs as temperature varies.

However, the dynamics change significantly at low LRs.
Unlike in the UP model, the OP setup converges to an optimum, marked by a consistent decrease in both loss and entropy.
These metrics do not stabilize within a reasonable number of iterations, making the established temperature estimation algorithm inapplicable.
Instead, we again adopt the thermodynamic definition of temperature, $T=dU/dS$, and compute the discrete derivatives of energy and entropy with respect to training time (i.e., iteration steps) to track how temperature evolves during training.
The results are shown in Figure~\ref{fig:temperature_reg1}.
We observe that, for all low LR values, temperature steadily decays towards zero.
In other words, the same reduction in entropy yields a progressively smaller decrease in loss as training progresses.
From a thermodynamic viewpoint, vanishing temperature implies that the free energy effectively collapses to the training loss, and noise no longer prevents the optimizer from converging to a minimum.
The stationary state for such dynamics is the delta-distribution at the minimum point, which suggests that temperature for small LRs experiences a sharp phase transition to zero values\footnote{
Curiously enough, this effect is similar to the superconductivity phenomenon, when the resistance of a material drops abruptly to zero for small absolute temperatures (see, for instance, Figure 4 in~\cite{superconductivity})
}.
In the next section, we investigate which aspect of the OP setup drives this distinctive behavior.

\section{Explaining difference between UP and OP setups}

\subsection{Behavior of stochastic gradients near optimum}

In this section, we investigate why OP setting helps SGD to converge to an optimum.
We begin by defining OP strictly as a setup in which all stochastic gradients vanish at the global minimum of the loss function.
In other words, the optimum over the entire dataset also minimizes the loss for each individual data point.
For a differentiable, non-negative loss function, achieving zero loss at the global optimum is sufficient to ensure this condition, since a zero mean implies that each individual loss—and thus each corresponding gradient—is also zero.
This criterion matches a commonly used definition of OP as a regime, where a model can interpolate training data~\cite{interp_sgd}.
Conversely, if the stochastic gradients do not vanish at the global optimum, the setup is classified as UP.
Importantly, in both cases, the full gradient $\overline{g}$ (i.~e., the average of all stochastic gradients $g_i$) is zero at the optimum, since this is a necessary condition for optimality.
To quantify the behavior of stochastic gradients near the optimum, we use the Signal-to-Noise Ratio (SNR) metric, defined as:
\[
\text{SNR} = \frac{\|\overline{g}\|}{\sqrt{\mathbb{E} \|g_i - \overline{g}\|^2}},
\]
Intuitively, SNR measures the norm of the mean gradient relative to the variance of stochastic gradients. A similar metric has been previously studied in connection to generalization~\cite{Liu2020Understanding,SunSYLMZ023,jiang2023acceleratinglargebatchtraining}.

We summarize the contrasting behavior of gradients in OP and UP settings in Table~\ref{tab:up_vs_op}.
In the UP setting, the stochastic gradients remain non-zero at the optimum, while the full gradient vanishes. As a result, the SNR at the optimum is zero.
In contrast, in the OP setting, both the stochastic and full gradients are zero at the optimum, making the SNR formally undefined.
To better understand the behavior of SNR, we consider its evolution along trajectories approaching the optimum, i.~e., under full gradient flow.
As the flow converges to a stationary point, the loss and gradients should converge to their limiting values.
This implies that in the UP setting, the SNR decreases to zero.
However, in the OP setting, it can be shown that the SNR instead stabilizes at a non-zero constant (see Appendix~\ref{app:snr_proof} for details).
This fundamental difference highlights the distinct behavior of SNR in UP vs. OP settings near the optimum.
Nevertheless, the picture changes when we take into account the finite LR of SGD.
In the UP setting, the large stochastic gradients prevent the optimizer from reaching the exact minimum, causing it to fluctuate around a typical distance from the optimum.
This results in stable gradient norms (both full and stochastic) and, consequently, a constant SNR.
Meanwhile, in the OP setting, both the full and stochastic gradients decay to zero at the same rate along the SGD trajectory, again leading to a constant SNR.
Therefore, although SNR stabilizes in both settings when using SGD with a small finite LR, it does so for \emph{entirely different reasons}.

\begin{table}[t]
    \newcommand{\specialcell}[2][c]{%
    \begin{tabular}[#1]{@{}c@{}}#2\end{tabular}}
    \renewcommand{\arraystretch}{1.1}
    \centering

    \caption{Comparison of underparameterized (UP) and overparameterized (OP) setups: behavior of loss ($L$), full gradient ($\overline{g}$), stochastic gradients ($g_i$) and SNR near optimum and along trajectories of full gradient flow and SGD. The character $\to$ denotes convergence to some value.}
    \label{tab:up_vs_op}

    \begin{tabular}{cc|ccc}
        \toprule
        setup & variable & at optimum $w^*$ & \specialcell{trajectory of full grad. \\ flow  (infinitesimal LR)} & \specialcell{trajectory of SGD \\ (small finite LR)} \\
        \midrule
        \multirow{4}{*}{UP} & loss & $L(w^*)>0$ & $\to L(w^*)$ & $\text{const} > L(w^*)$\\
        & full grad. & $\|\overline{g}\| = 0$ & $\|\overline{g}\| \to 0$ & $\mathbb{E} \|\overline{g}\| = \text{const} > 0$ \\ 
        & stoch grad. & $\mathbb{E} \|g_i\| > 0$ & $\mathbb{E} \|g_i\| > 0$ & $\mathbb{E} \|g_i\| = \text{const} > 0$ \\ 
        & SNR & $=0$ & $\to 0$ & $\text{const} > 0$ \\ 
        \midrule
        \multirow{4}{*}{OP} & loss & $L(w^*)\ge 0$ & $\to L(w^*)$ & $\to L(w^*)$ \\
        & full grad. & $\|\overline{g}\| = 0$ & $\|\overline{g}\| \to 0$ & $\|\overline{g}\| \to 0$ \\ 
        & stoch grad. & $\mathbb{E} \|g_i\| = 0$ & $\mathbb{E} \|g_i\| \to 0$ & $\mathbb{E} \|g_i\| \to 0$ \\ 
        & SNR & $=0/0$ & $\text{const}>0$ & $\text{const}>0$ \\ 
        \bottomrule
    \end{tabular}
\end{table}

\subsection{Toy example on 3D sphere}
\label{sec:toy}

\begin{wrapfigure}{r}{0.35\textwidth}
    \vspace{-0.5cm}
    \addtolength{\tabcolsep}{-0.4em}
    \begin{center}
    \begin{tabular}{cc}
        (a) UP & (b) OP \\
        \includegraphics[width=0.16\textwidth]{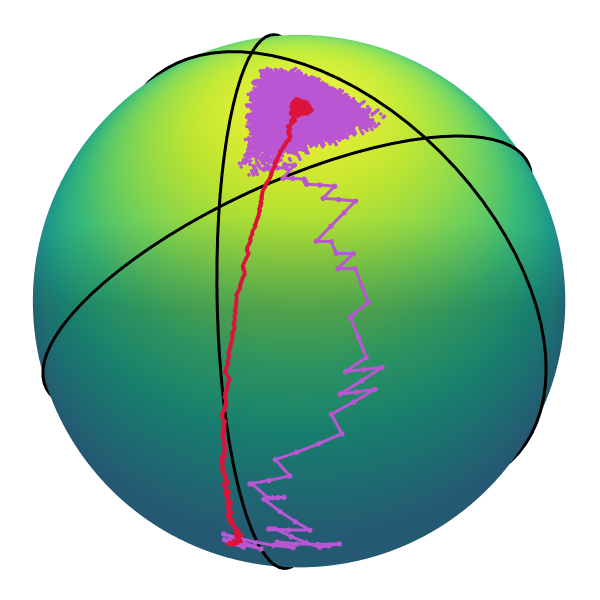} & 
        \includegraphics[width=0.16\textwidth]{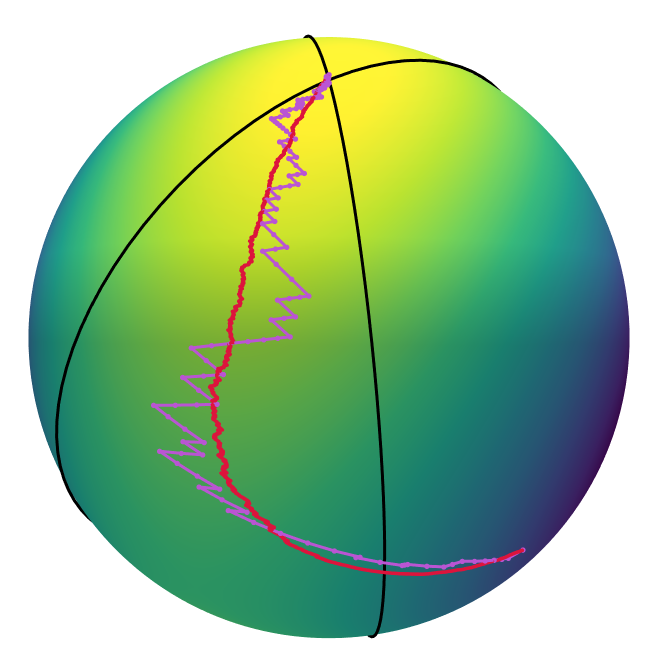} \\
        \multicolumn{2}{c}{\includegraphics[width=0.32\textwidth]{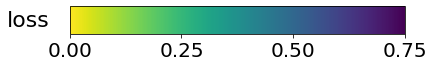}}
    \end{tabular}
    \end{center}
    \caption{Loss surfaces of UP and OP setups and training trajectories of SGD for 3D toy example. Red lines correspond to a low LR ($2.4 \cdot 10^{-3}$), purple --- to a higher LR ($6.9\cdot 10^{-2}$). Black lines indicate great circles.}
    \label{fig:sphere_trajectory}
\end{wrapfigure}
To validate our thesis about the behavior of stochastic gradients and SNR under small LRs, we design a toy example with the following properties:
(1) the loss is scale-invariant and defined on a unit 3D sphere, allowing us to simulate NN training dynamics on
a spherical manifold and enabling intuitive visualizations;
(2) the loss is defined as the mean of several component functions, mimicking a dataset of training examples;
(3) the model can switch between UP and OP settings by varying the number of “examples” in the dataset.
To construct this example, we leverage a geometric property of great circles on a 3D sphere: any two distinct great circles intersect at exactly two points, whereas three arbitrary great circles generally do not intersect at a common point.
We associate each great circle with a distinct object and define the loss $L_i$ for an object $i$ as:
\[
L_i(x, y, z) = \frac{(A_ix + B_iy + C_i z)^2}{x^2 + y^2 + z^2},
\]
where $(A_i, B_i, C_i)$ is the normal vector to the plane of a great circle and denominator insures the scale-invariance of the function.
The value in the nominator is the squared distance to this plane, meaning that zero loss for the object $i$ is achieved at the points of the corresponding great circle. 
In this setup, using two examples defines the OP case, since a point at the intersection of the two great circles achieves zero full loss.
In contrast, with three examples, no common intersection exists in general, resulting in a UP scenario where zero full loss is unattainable.
A visualization of these loss surfaces on the sphere is provided in Figure~\ref{fig:sphere_trajectory}.
We run projected SGD with batch size $1$ varying the LR value.
For a more detailed description of the experimental setup, see Appendix~\ref{app:sphere_result}.

\begin{figure}
  \addtolength{\tabcolsep}{-0.4em}
  \centering
  \centerline{
  \begin{tabular}{ccc}
      Underparameterized (UP) & Overparameterized (OP) \\
      \includegraphics[width=0.43\textwidth]{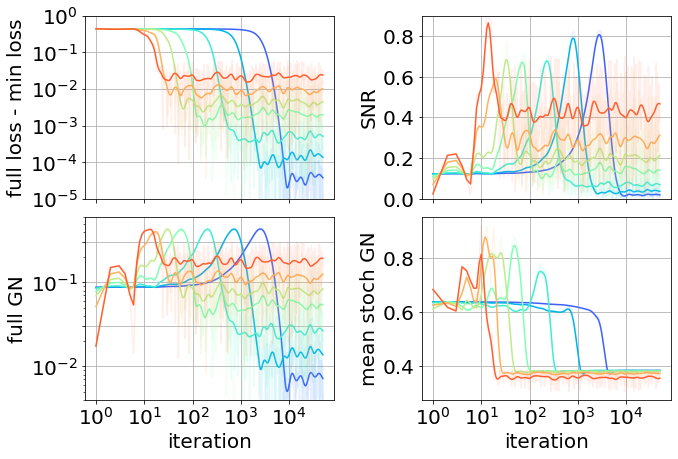} & \includegraphics[width=0.43\textwidth]{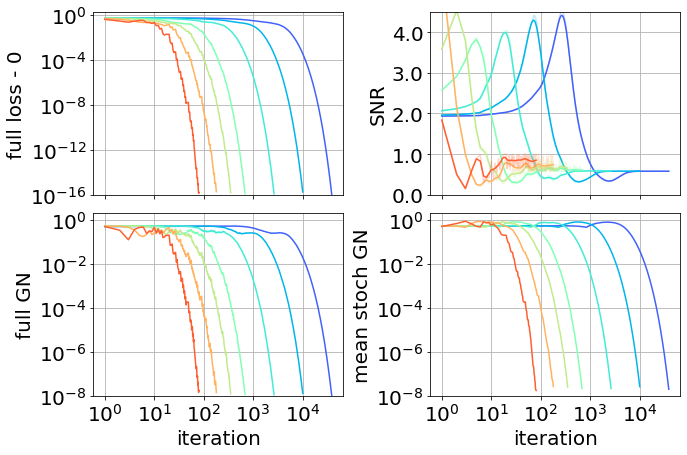} & 
      \raisebox{0.43em}{\includegraphics[width=0.13\textwidth]{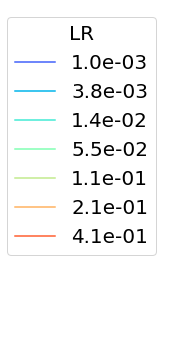}}
  \end{tabular}
  }
  \caption{Training metrics for various LRs in UP and OP setups on a 3D sphere. To plot the full loss, we subtract the loss value at the minimum (positive for UP and zero for OP).}
  \label{fig:sphere_snr}
\end{figure}

Figure~\ref{fig:sphere_snr} shows the evolution of the loss, SNR, and gradient norms.
In the UP setting, SGD stabilizes around the minimum (Figure~\ref{fig:sphere_trajectory}a), with the loss and gradient norms fluctuating around non-zero values.
As the LR increases, the loss and full gradient norm plateau at higher levels, reflecting larger deviations from the minimum.
Meanwhile, the stochastic gradient norm remains roughly constant, leading to a higher SNR.
As the LR decreases, the limiting SNR value approaches zero, gradually mirroring the behavior of full gradient flow.

In contrast, in the OP setting, training converges precisely to the minimum (Figure~\ref{fig:sphere_trajectory}b).
We terminate optimization when the loss reaches $10^{-16}$, to avoid numerical instability due to machine precision.
Both full and stochastic gradients vanish, while the SNR stabilizes, though it initially increases.
Intuitively, as the optimizer approaches the minimum, the stochastic gradients align more closely, increasing the SNR.
In Appendix~\ref{app:snr_proof}, we provide a rigorous proof showing that (1) when approaching the optimum along a fixed meridian\footnote{If we treat the two minima of the loss as the poles, then the two great circles corresponding to the examples act as meridians—just like any other great circle passing through the poles (e.~g., the central meridian).}, the SNR value converges upward to a constant limit, (2) the minimum among these SNR limits is achieved on the central meridian and is strictly positive.
Since smaller LRs result in trajectories that stay closer to the central meridian, they correspondingly yield lower (but still non-zero) SNR values, as observed in our numerical simulation.

\subsection{SNR in neural networks}
\label{sec:snr_nn}

To verify the SNR intuition for NNs, we plot the same set of metrics for the training of ConvNet (Figure~\ref{fig:snr_convnet}).
As expected in the UP setting, all three metrics stabilize once optimization reaches a stationary state. Notably, the asymptotic value of the SNR increases with the LR.
In contrast, the SNR behavior in the OP setting diverges from the trend observed in the toy example: here, SNR steadily declines throughout training.
While both the full and stochastic gradient norms decrease over time, the stochastic gradients decay more slowly, causing the SNR to diminish.
To quantify this effect, we plot a phase diagram of the full and stochastic gradient norms and fit their relationship using a power law (Figure~\ref{fig:grad_phase}).
For the baseline experiment with $50$k training samples, we obtain an exponent of approximately 0.78, substantially less than 1, indicating that stochastic gradients decay more slowly than full gradients.
This result suggests that the setup is \emph{not sufficiently overparameterized}, as an exponent below 1 violates the theoretical definition of the OP regime (see Appendix~\ref{app:snr_proof}).
In other words, stochastic gradients should either stabilize like in the UP case or decay at the same rate as full gradients, resulting in an exponent approaching $1$.
Determining which of these two options holds for this model would require an impractically large number of training iterations, potentially even beyond the limits of numerical precision.
However, we can probe this behavior by varying the degree of overparameterization via the size of the training dataset and observing how the gradient decay rates respond.
Indeed, our experiments show that smaller datasets lead to higher exponents, eventually approaching 1.
This suggests that the gradient decay exponent could serve as a measure of overparameterization, which is an intriguing direction for future investigation.

\begin{figure}
\centering
\begin{minipage}[c]{0.68\textwidth}
    \addtolength{\tabcolsep}{-0.4em}
    \begin{tabular}{ccc}
        \multicolumn{3}{c}{Underparameterized (UP)} \\
        \includegraphics[width=0.31\textwidth]{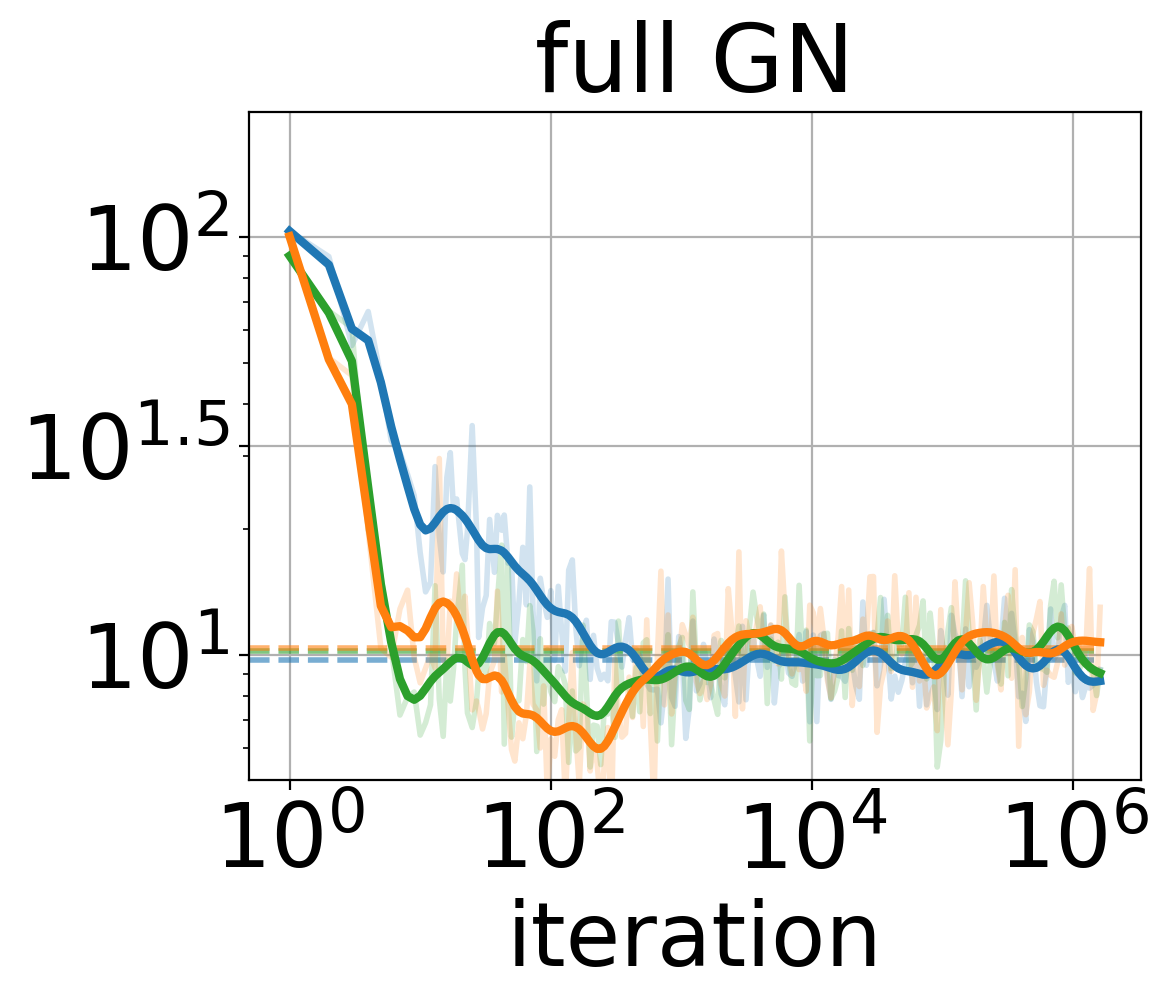} &
        \includegraphics[width=0.31\textwidth]{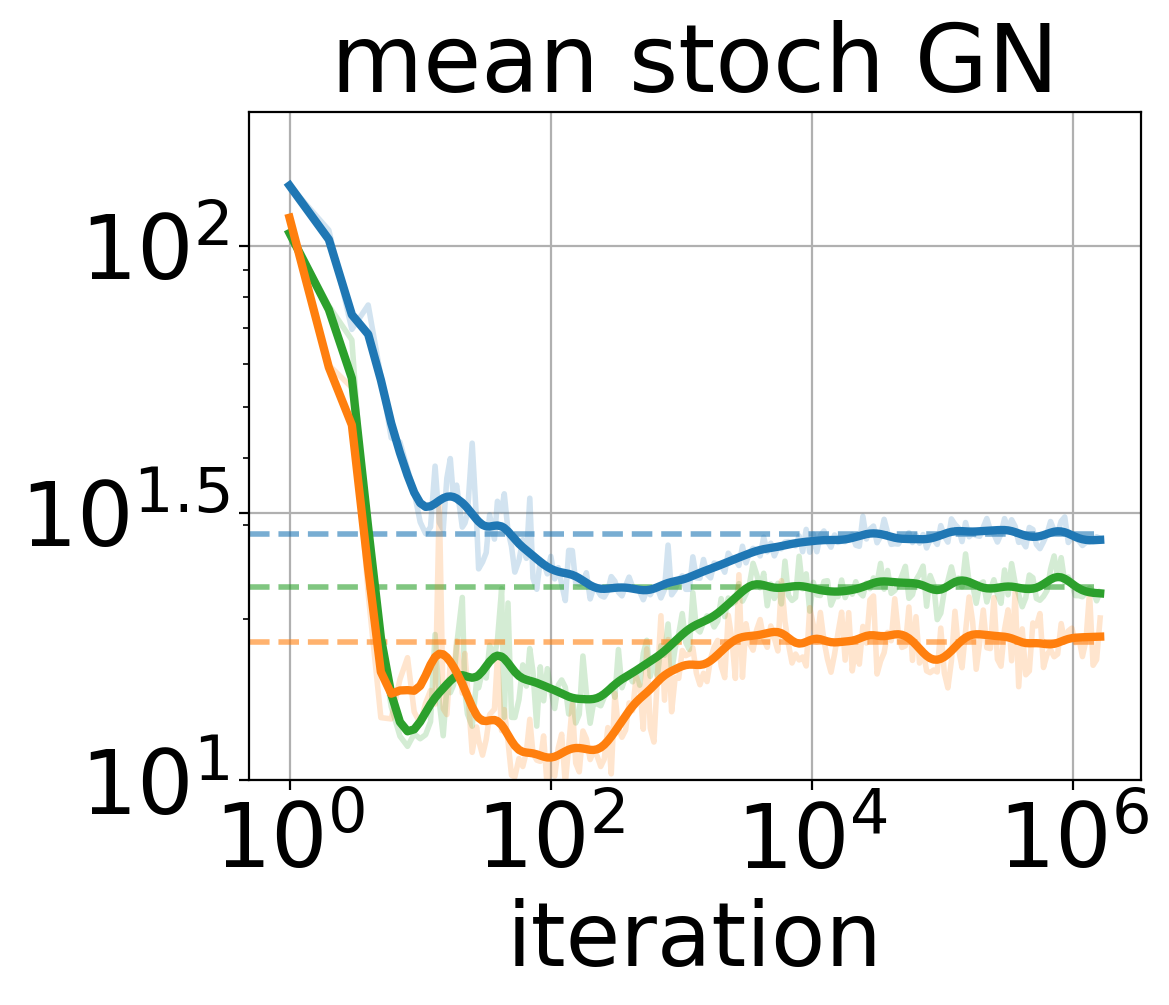} & 
        \includegraphics[width=0.325\textwidth]{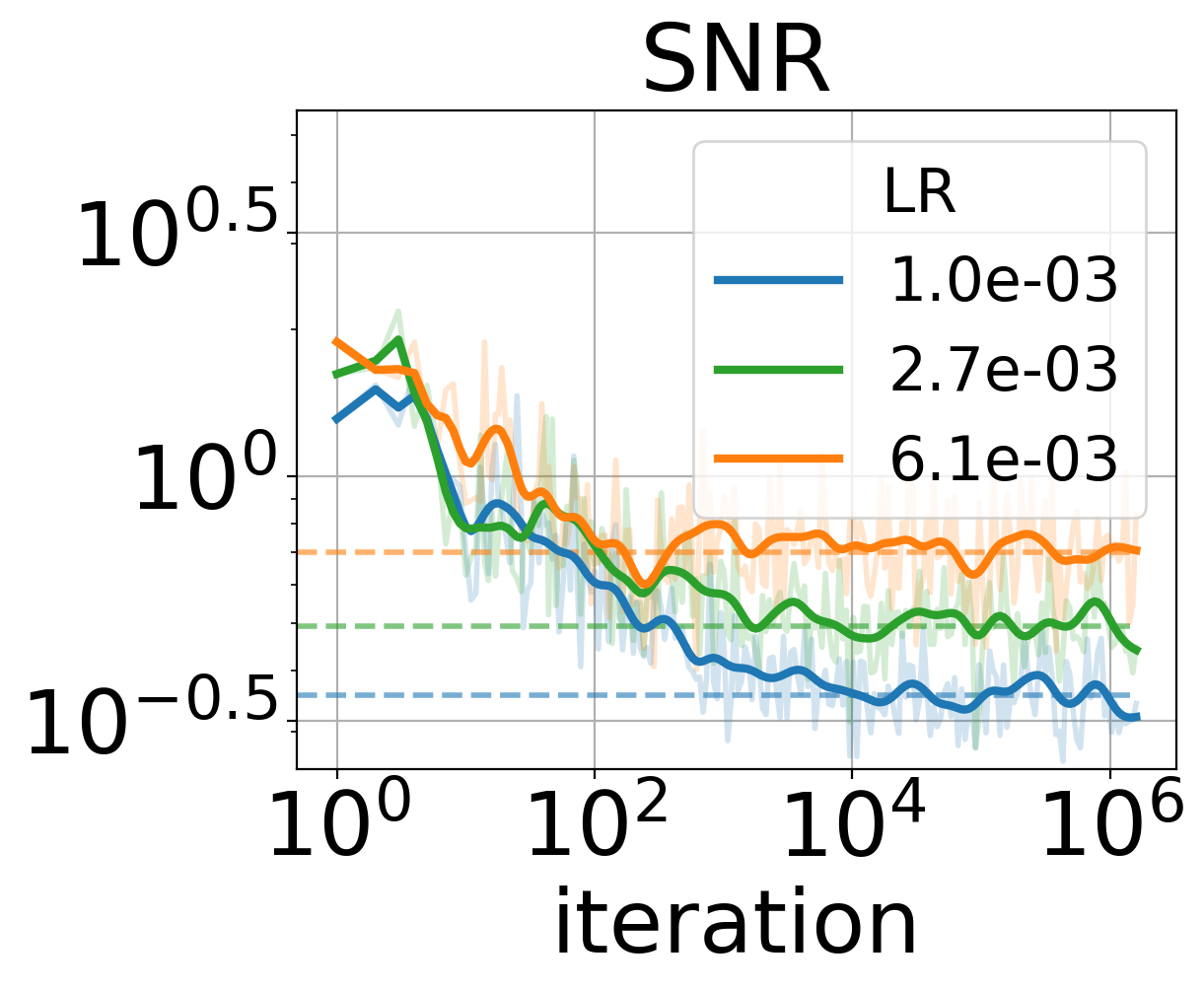} \\
        \multicolumn{3}{c}{Overparameterized (OP)} \\
        \includegraphics[width=0.31\textwidth]{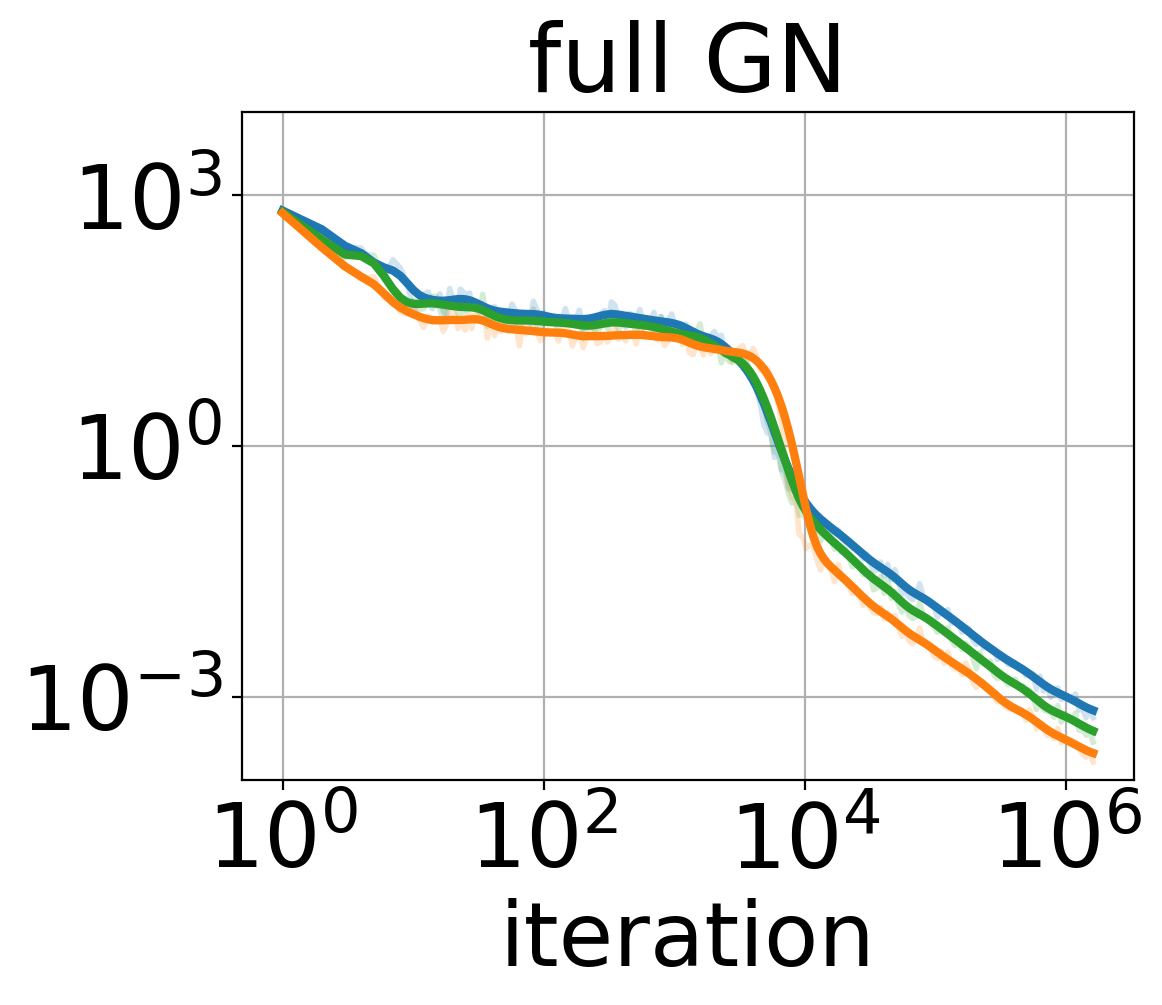} &
        \includegraphics[width=0.31\textwidth]{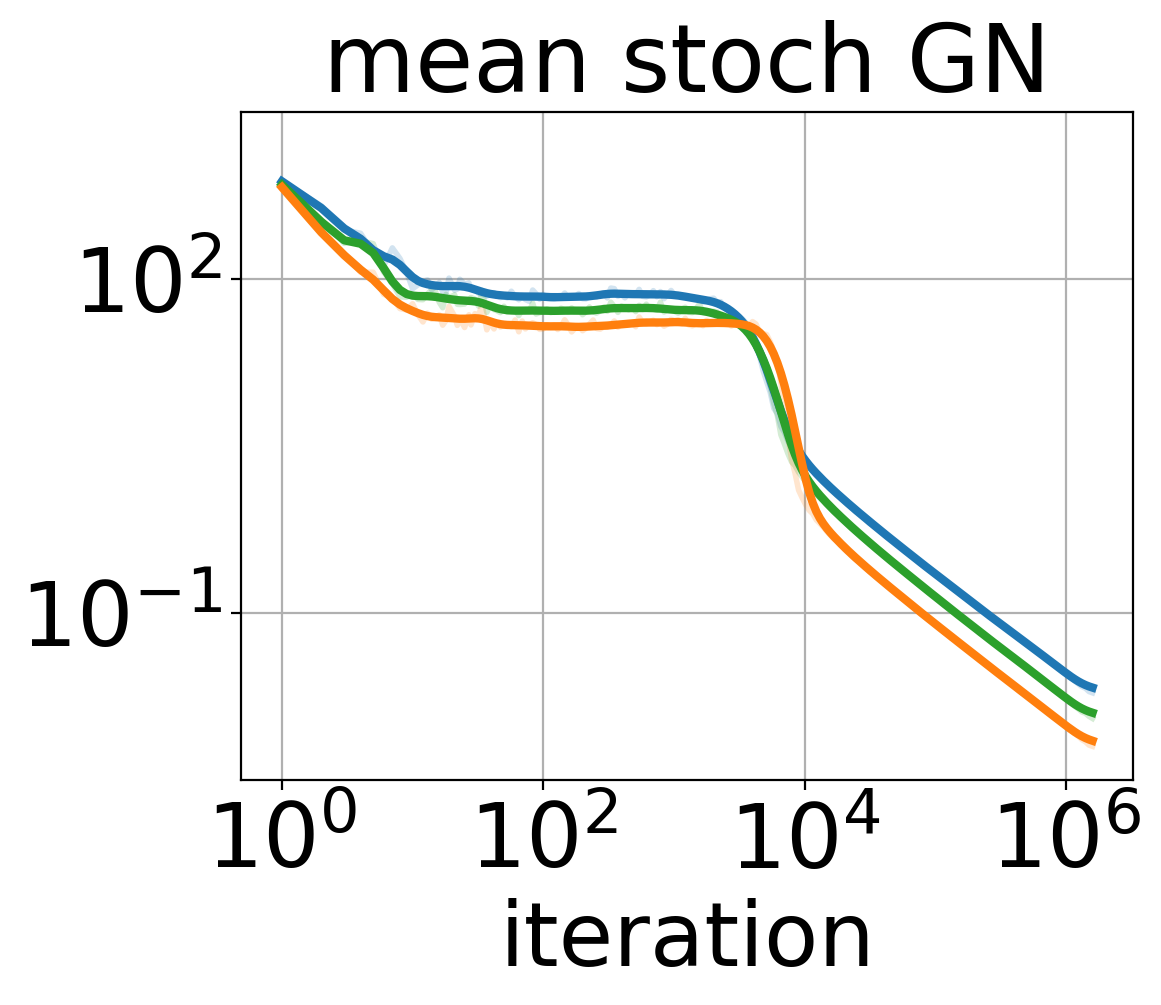} & 
        \includegraphics[width=0.325\textwidth]{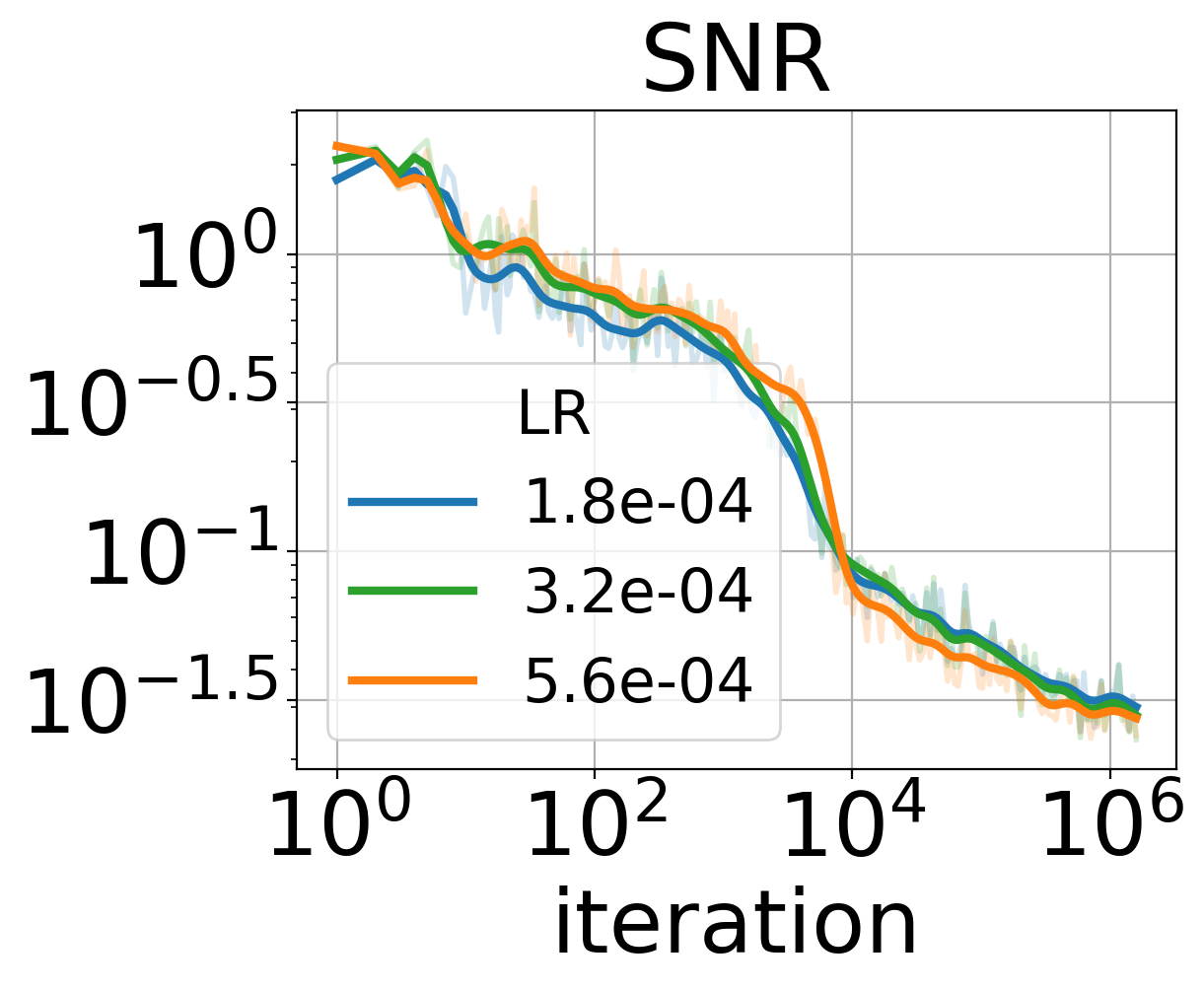} \\
    \end{tabular}
    \caption{Norm of full gradient (left), mean norm of stoch. gradient (center) and SNR (right) for small LRs of UP and OP ConvNet on CIFAR-10. Dashed lines indicate stationary values for the UP setup.}
    \label{fig:snr_convnet}
\end{minipage}\hfill
\begin{minipage}[c]{0.3\textwidth}
    \vspace{0.53cm}
    \begin{center}
        \includegraphics[width=\textwidth]{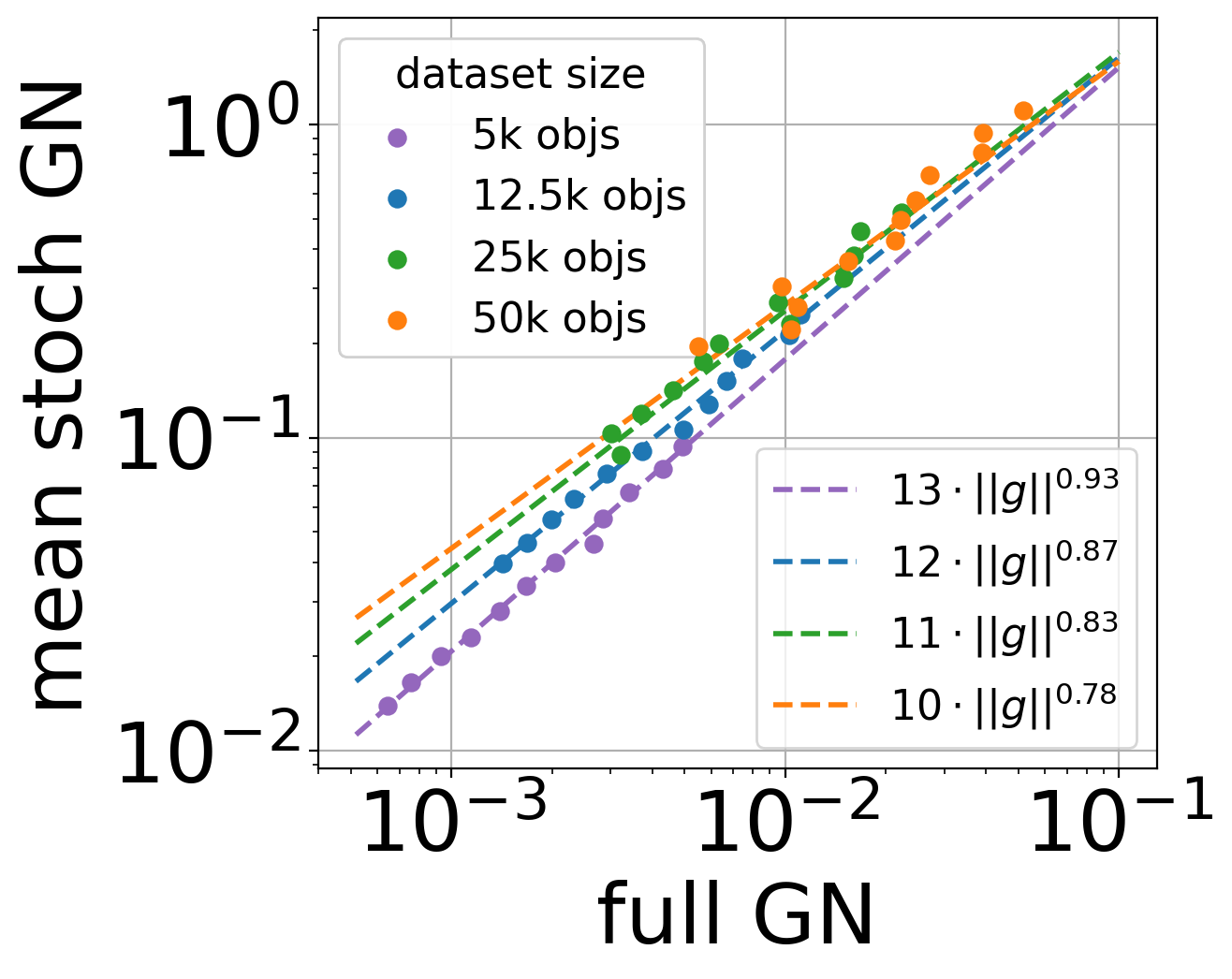}
    \end{center}
    \caption{Phase diagram of mean norm of stochastic grad. vs. full grad. norm for OP ConvNet on CIFAR-10 subsets of different size. For all four datasets we use LR equal to $2 \cdot 10^{-5}$. Dashed lines show power law approximation.}
    \label{fig:grad_phase}
\end{minipage}
\end{figure}

\section{Discussion and future work}

We have demonstrated that our thermodynamic framework effectively captures the stationary distributions of SGD under varying LRs.
This perspective not only explains gradual changes in the optimization system, such as increasing loss and entropy, but also provides a lens through which to interpret abrupt transitions, such as convergence to a loss minimum in the OP setting.
More broadly, NN training dynamics are influenced by many factors, including learning rate~\cite{cohen2021gradient, kodryan2022training}, batch size~\cite{interp_sgd, phase_batch_size}, number of training iterations~\cite{double_descent2, kaplan2020scaling, double_descent}, and the parameter-to-data ratio~\cite{double_descent2, kaplan2020scaling}.
Changes in any of these can lead to either smooth behavioral shifts (as in scaling laws~\cite{kaplan2020scaling}) or sharp phase transitions (such as double descent~\cite{double_descent2, double_descent}, grokking~\cite{liu2022towards, grokking}, or emergent abilities of large language models (LLMs)~\cite{wei2022emergent,srivastava2023beyond}).
Although our analysis focuses on a simple system, where the only varied factor is the LR, we believe this thermodynamic viewpoint can help unify and interpret a wide range of phenomena in modern deep learning.

One more interesting direction is to extend the proposed thermodynamic framework to analyze not only the stationary behavior at the end of training, but also the quasi-stationary dynamics observed during training. 
Neural networks often exhibit prolonged plateaus in their loss curves, followed by sudden drops associated with the acquisition of new capabilities~\cite{chen2024sudden,gopalani2024abrupt}. 
A thermodynamic analysis of these plateau phases could improve our understanding of where and why training stagnates, and help explain how such quasi-stationary regimes eventually transition into more effective solutions.

\section{Conclusion}
\label{sec:conclusion}

In this work, we introduced and validated a thermodynamic framework for understanding the stationary behavior of SGD.
We showed that SGD implicitly minimizes free energy, defined as a trade-off between training loss and weight entropy, with temperature controlled by the learning rate.
In the UP setup, temperature changes smoothly with LR, while in the OP setup, it rapidly drops to zero at low LRs.
This contrast stems from differences in stochastic gradient behavior near global minima, enabling OP models to converge directly.
Our framework provides a novel perspective on the optimization of neural networks and can be extended to analyze other phenomena in deep learning.

\paragraph{Limitations}

The main limitation of our work is its exclusive focus on LR, despite batch size being another important parameter that influences the magnitude of noise in SGD.
Second, our analysis is restricted to plain SGD, leaving out widely used alternatives like SGD with momentum or Adam~\cite{KingmaB14}.
Finally, we rely on a non-parametric entropy estimate, which might be biased given the high dimensionality of neural network weight spaces.

\section*{Acknowledgments}
We would like to thank Mikhail Burtsev for the valuable comments and discussions.
We also thank Alexander Shabalin for insightful conversations and personal support.
We are grateful to Dmitry Pobedinsky's popular-science video on various definitions of temperature in physics, which inspired part of our perspective.
Ekaterina Lobacheva was supported by IVADO and the Canada First Research Excellence Fund.
The empirical results were enabled by compute resources and technical support provided by Mila - Quebec AI Institute (mila.quebec) and the Constructor Research Platform~\cite{cub_research}.

\medskip

\bibliographystyle{plainnat}
\bibliography{ref}


\newpage
\appendix

\section{Why Helmholtz free energy?}
\label{app:free_energy}

In this section, we further justify the use of Helmholtz free energy to describe the stationary behavior of SGD.
In thermodynamics, choosing the appropriate potential requires fixing certain \emph{natural variables}.
Commonly, the number of particles $N$ is held constant, assuming no exchange with the environment. Drawing an analogy, we treat the weight space dimensionality $D$ as an equivalent to $N$, since both thermodynamic and differential entropy scale linearly with these quantities. 
In our experiments, $D$ is tied to the network architecture and remains fixed.
As noted in the main text, we also fix the learning rate and batch size during training, which we interpret as fixing the temperature $T$.
This leaves us with two choices: fixing either volume $V$ or pressure $p$, which correspond to using Helmholtz ($F=U-TS$) or Gibbs ($G=U-TS+pV$) free energy, respectively.
We adopt the Helmholtz formulation for three main reasons:
\begin{enumerate}
    \item Defining pressure in the context of neural networks is unclear, while the parameter space, which remains fixed (e.~g., unit sphere in our case), can be interpreted as fixed volume.
    \item Calculating Helmholtz free energy does not require explicit definitions of pressure or volume.
    \item Free energy in the form $F=U-TS$ is sometimes used to describe systems of non-thermodynamic nature, such as the Ising model\footnote{\url{https://universalitylectures.wordpress.com/wp-content/uploads/2013/05/lecture_5_ising_model_and_phase_transitions_v1.pdf}}.
\end{enumerate}
An additional justification for using the Helmholtz free energy arises from the stationary distributions of SGD.
\citet{jastrzebski2017three} show that, under the assumption of isotropic gradient noise, SGD converges to a Gibbs-Boltzmann distribution:
\[
p(w) \propto \exp\left(-\frac{L(w)}{2T}\right), \quad T = \frac{\eta\sigma^2}{B},
\]
where $\eta$ is the learning rate, $B$ is the batch size, and $\sigma^2$ is the gradient variance. 
The Gibbs-Boltzmann distribution is known to minimize the Helmholtz free energy over all possible distributions\footnote{https://physics.stackexchange.com/questions/54692/why-is-the-free-energy-minimized-by-the-boltzmann-distribution}.
While the isotropic noise assumption does not hold in practical neural networks, this example still motivates the use of Helmholtz free energy in our analysis. 
Altogether, these considerations support Helmholtz free energy as a natural thermodynamic potential for modeling the stationary behavior of SGD in our framework.

\section{Experimental setup details}
\label{app:exp_setup}

\paragraph{Training details}
We perform experiments using two network architectures --- a simple 4-layer convolutional network (ConvNet) and ResNet-18 --- on the CIFAR-10 and CIFAR-100 datasets. For each architecture-dataset pair, we evaluate both underparameterized (UP) and overparameterized (OP) settings.
For ConvNets, we use width factors $k = 8$ (UP) and $k = 64$ (OP) on CIFAR-10, and $k = 8$ (UP) and $k = 96$ (OP) on CIFAR-100.
For ResNet-18, the width factors are $k = 4$ (UP) and $k = 32$ (OP) on CIFAR-10, and $k = 4$ (UP) and $k = 48$ (OP) on CIFAR-100.

All models are trained with projected SGD using a batch size of $128$ and independent (i.~e., non-epoch) batch sampling.
ConvNet models are trained for $1.6 \cdot 10^6$ iterations, while ResNet-18 models are trained for $10^6$ iterations on CIFAR-10 and $4.8 \cdot 10^5$ iterations on CIFAR-100.
Following Kodryan et al.~\cite{kodryan2022training}, we use scale-invariant versions of both models by fixing the last layer’s parameters at initialization and setting their norm to $10$.
We project the vector of all scale-invariant parameters onto the unit sphere after initialization and after each optimization step.

Overall, we use the following grid of LRs (there are $28$ different values in total):
\begin{itemize}
    \item $3$ logarithmically-spaced LRs from $10^{-5}$ to $10^{-4}$: $\{1.0, 2.2, 4.6\} \cdot 10^{-5}$
    \item $4$ logarithmically-spaced LRs from $10^{-4}$ to $10^{-3}$: $\{1.0, 1.8, 3.2, 5.6\}\cdot10^{-4}$
    \item $14$ logarithmically-spaced LRs from $10^{-3}$ to $10^{-2}$: $\{1.0, 1.2, 1.4, 1.6, 1.9, 2.3, 2.7, 3.2, 3.7,$ $4.4,5.2, 6.1, 7.2, 8.5\}\cdot10^{-3}$. A more dense grid for this segment is used to ensure a sufficient number of LRs for temperature estimation in OP setups.
    \item $7$ logarithmically-spaced LRs from $10^{-2}$ to $10^{0}$ (including the upper boundary): $\{1.0, 2.2, 4.6, 10, 22, 46, 100\}\cdot10^{-2}$
\end{itemize}

\paragraph{Temperature estimation}
For temperature estimation, we exclude several small LR values (the ones that did not stabilize in UP setups or the ones that converge to minima in OP setups), as well as some large LRs, which either enter regime 3 or show edge effects associated with the phase transition from regime 2 to regime 3 (see Section 5.3 in~\cite{kodryan2022training} and Appendix D in~\cite{Sadrtdinov2024largelr}).
This procedure results in the following LR segments for temperature estimation (both boundaries are included):
\begin{itemize}
    \item UP ConvNet on CIFAR-10 --- $[1.8 \cdot 10^{-4}, 1.0 \cdot 10^{-2}]$
    \item OP ConvNet on CIFAR-10 --- $[1.0 \cdot 10^{-3}, 5.2 \cdot 10^{-3}]$
    \item UP ConvNet on CIFAR-100 --- $[5.6 \cdot 10^{-4}, 2.2 \cdot 10^{-2}]$
    \item OP ConvNet on CIFAR-100 --- $[5.6 \cdot 10^{-4}, 6.1 \cdot 10^{-3}]$
    \item UP ResNet-18 on CIFAR-10 --- $[1.8 \cdot 10^{-4}, 1.0 \cdot 10^{-2}]$
    \item OP ResNet-18 on CIFAR-10 --- $[1.0 \cdot 10^{-3}, 7.2 \cdot 10^{-3}]$
    \item UP ResNet-18 on CIFAR-100 ---  $[5.6 \cdot 10^{-4}, 2.2 \cdot 10^{-2}]$
    \item OP ResNet-18 on CIFAR-100 --- $[1.0 \cdot 10^{-3}, 6.1 \cdot 10^{-3}]$
\end{itemize}

To smooth the $U(\eta)$ and $S(\eta)$ curves, we apply kernel smoothing using a triangular kernel in the logarithmic domain (here we set $h=0.3$):
\[
K(\eta_1, \eta_2) = \max\left(h - \left|\log \eta_1 - \log \eta_2\right|, 0\right)
\]
Boundary values of LRs are preserved to avoid artifacts from smoothing. We choose a triangular kernel due to the non-uniform LR grid, which is denser in regions where OP models fail to converge. The LR grid is shared across UP and OP setups, so using kernels with infinite support (e.g., Gaussian) would bias smoothed values of $L(\eta)$ and $S(\eta)$ for small LRs in the UP setting. This happens because the grid is more dence for larger LRs, leading to higher values of the metrics for small LRs after smoothing.

For estimating temperature confidence intervals, we use the following threshold values $\varepsilon$: 
UP ConvNet CIFAR-10 --- $10^{-2}$, 
OP ConvNet CIFAR-10 --- $4 \cdot 10^{-2}$, 
UP ConvNet CIFAR-100 --- $10^{-2}$,
OP ConvNet CIFAR-100 --- $4 \cdot 10^{-2}$,
UP ResNet-18 CIFAR-10 --- $10^{-2}$, 
OP ResNet-18 CIFAR-10 --- $10^{-2}$, 
UP ResNet-18 CIFAR-100 --- $10^{-2}$, 
and OP ResNet-18 CIFAR-100 --- $3 \cdot 10^{-2}$.
Larger values of $\varepsilon$ for several experiments are caused by more noisy behavior of training loss and entropy.
In the central column of Figures~\ref{fig:temperature_up_op}, \ref{fig:app_temperature_up_op_resnet}, and \ref{fig:app_temperature_up_op_convnet}, we omit both low and high borderline LRs, since only bounds (upper and lower, respectively) and not precise estimates on temperature can be obtained for these values.

For low LRs in the OP setups, we estimate temperature during training using a finite difference approximation:
\[
T_i = \frac{\Delta U_i}{\Delta S_i} = \frac{U_{i+dt} - U_{i-dt}}{S_{i+dt} - S_{i-dt}},
\]
where $dt = 2$ for all OP setups. 
Here, the subscript $i$ refers not to the raw iteration number but to the logging step index, which is uniformly spaced in the logarithmic domain.

\paragraph{Metrics smoothing}
We use kernel smoothing for metrics in Figures~\ref{fig:sphere_snr},\ref{fig:snr_convnet} with a Gaussian kernel in the log-time domain:
\[
K(t_1, t_2) = \exp\left(-\frac{(\log t_1 - \log t_2)^2}{2\sigma^2}\right),
\]
We use $\sigma=0.1$ for the toy setup on 3D sphere (Figure~\ref{fig:sphere_snr}) and $\sigma=0.2$ for the rest of experiments (Figures~\ref{fig:snr_convnet},\ref{fig:app_temperature_reg1},\ref{fig:app_snr_resnet},\ref{fig:app_snr_convnet}).

\paragraph{Compute resources}
For all experiments, we use NVIDIA Tesla V100 GPU's. The amount of compute spent on training of a single model is approximately 4 GPU hours for the UP setup and 20 GPU hours for the OP setup.
In total, the amount of compute, spent on all experiments is approximately 2500 GPU hours.
Training of the toy example on 3D sphere takes $<1$ minute on a single CPU, so we omit these negligible computational costs.

\section{Additional results on free energy}
\label{app:exp_results}
In this section, we present additional experimental results on the validation of the \emph{free energy hypothesis}.
Figure~\ref{fig:app_loss_entropy_iters} shows the evolution of training loss and entropy over iterations for different models and datasets.
As in Figure~\ref{fig:loss_entropy_iters} from the main text, varying the LR causes these metrics to stabilize at different levels, except for small LRs in the OP setting, where both loss and entropy steadily decrease, indicating convergence to a minimum.
Figures~\ref{fig:app_temperature_up_op_resnet} and \ref{fig:app_temperature_up_op_convnet} present the final values of loss and entropy (left), estimated temperature (center), and free energy (right).
Empirically estimated temperature function $T(\eta)$ is well-defined and monotonically increases for all considered LR values, complementing the results from Figure~\ref{fig:temperature_up_op} in the main text.
Figure~\ref{fig:app_temperature_reg1} shows the evolution of temperature over training iterations for small LRs in the OP setting, complementing Figure~\ref{fig:temperature_reg1} from the main text. The observed decay in temperature is consistent across different dataset–architecture pairs, suggesting that it is fundamentally tied to convergence behavior.

\begin{figure}[h!]
    \centering
    \addtolength{\tabcolsep}{-0.4em}

    \begin{tabular}{cccc}
        \multicolumn{2}{c}{UP ResNet-18 on CIFAR-10} & \multicolumn{2}{c}{OP ResNet-18 on CIFAR-10} \\
        \includegraphics[width=0.235\textwidth]{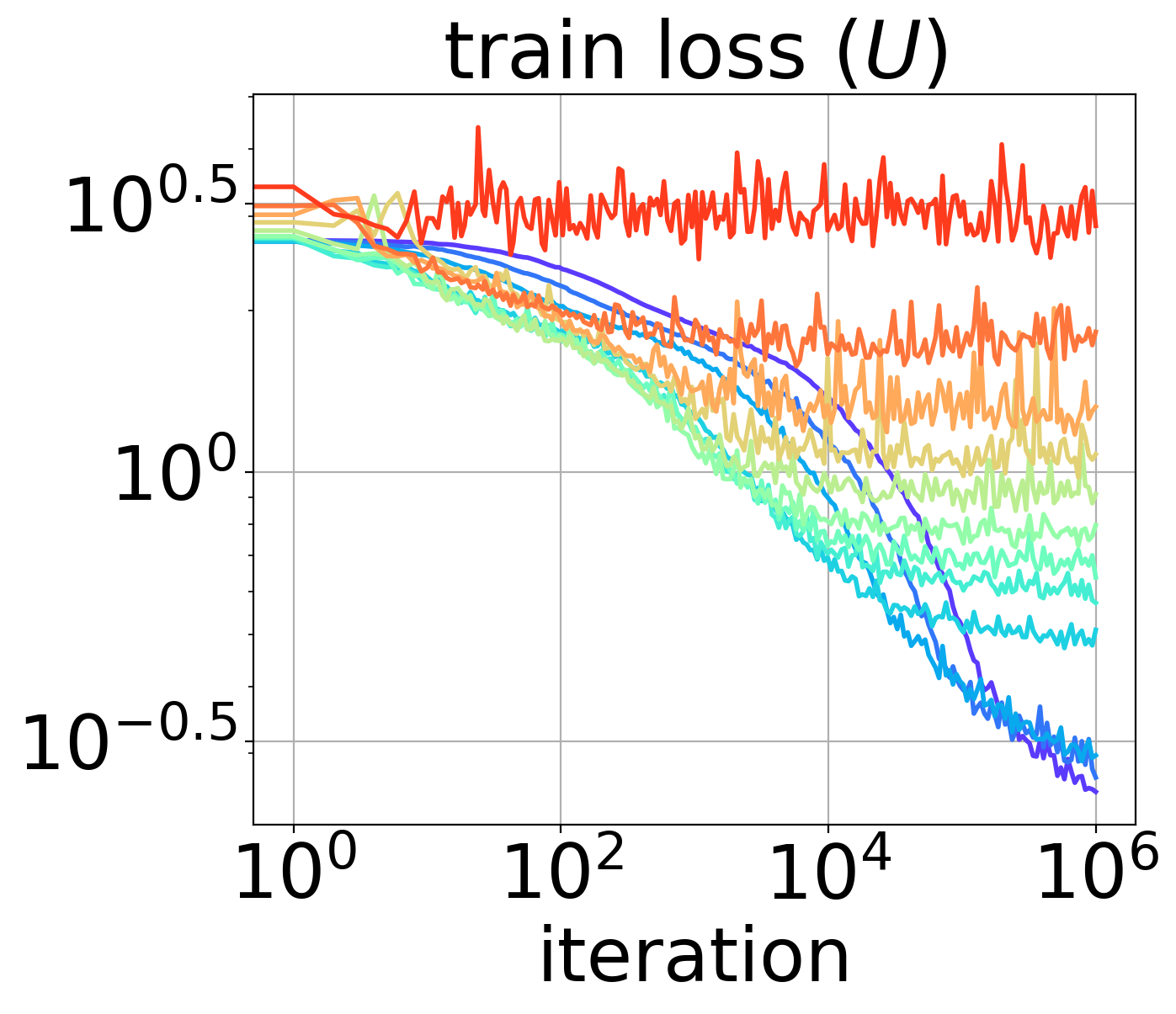} & 
        \includegraphics[width=0.203\textwidth]{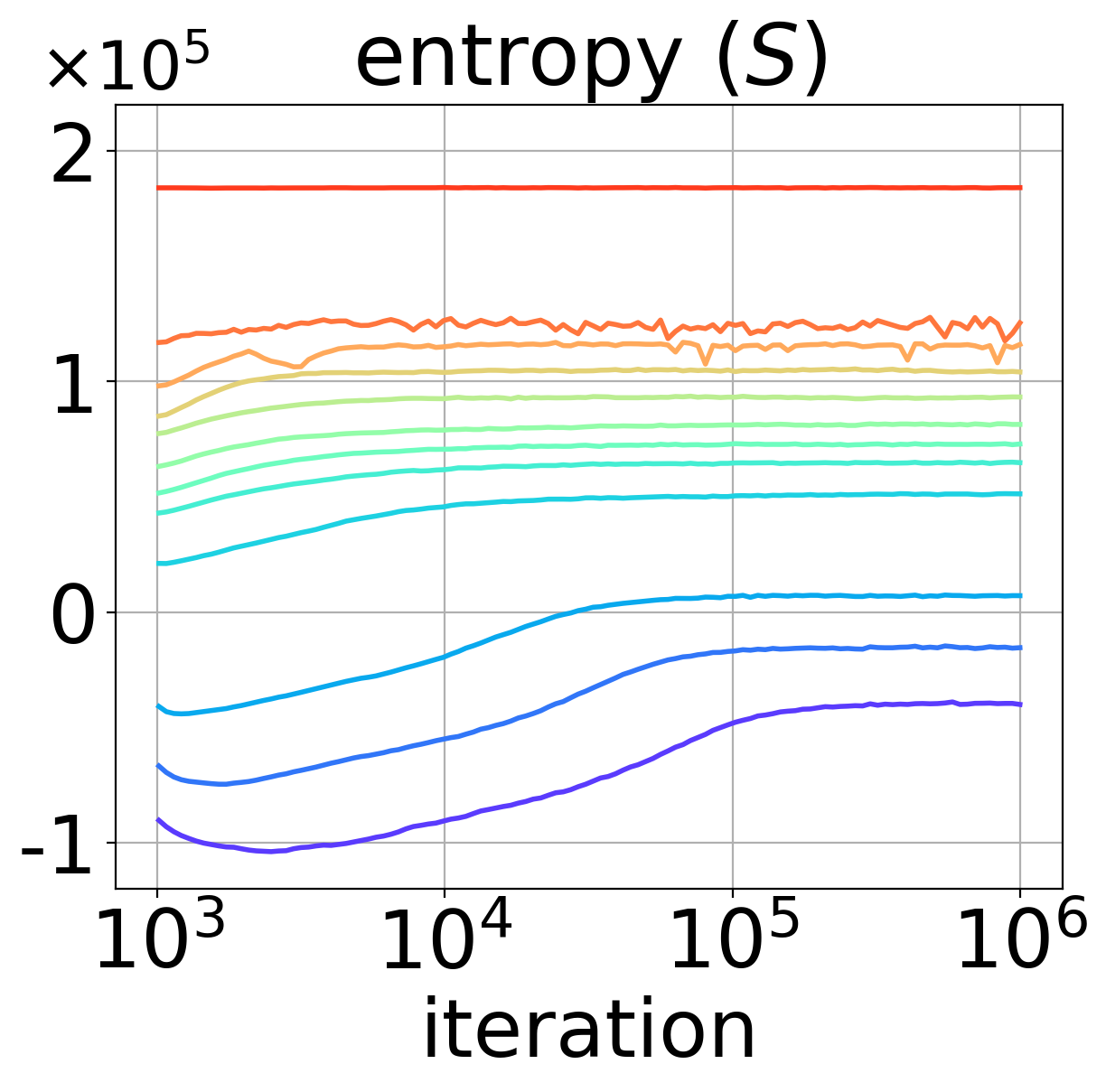} & 
        \includegraphics[width=0.225\textwidth]{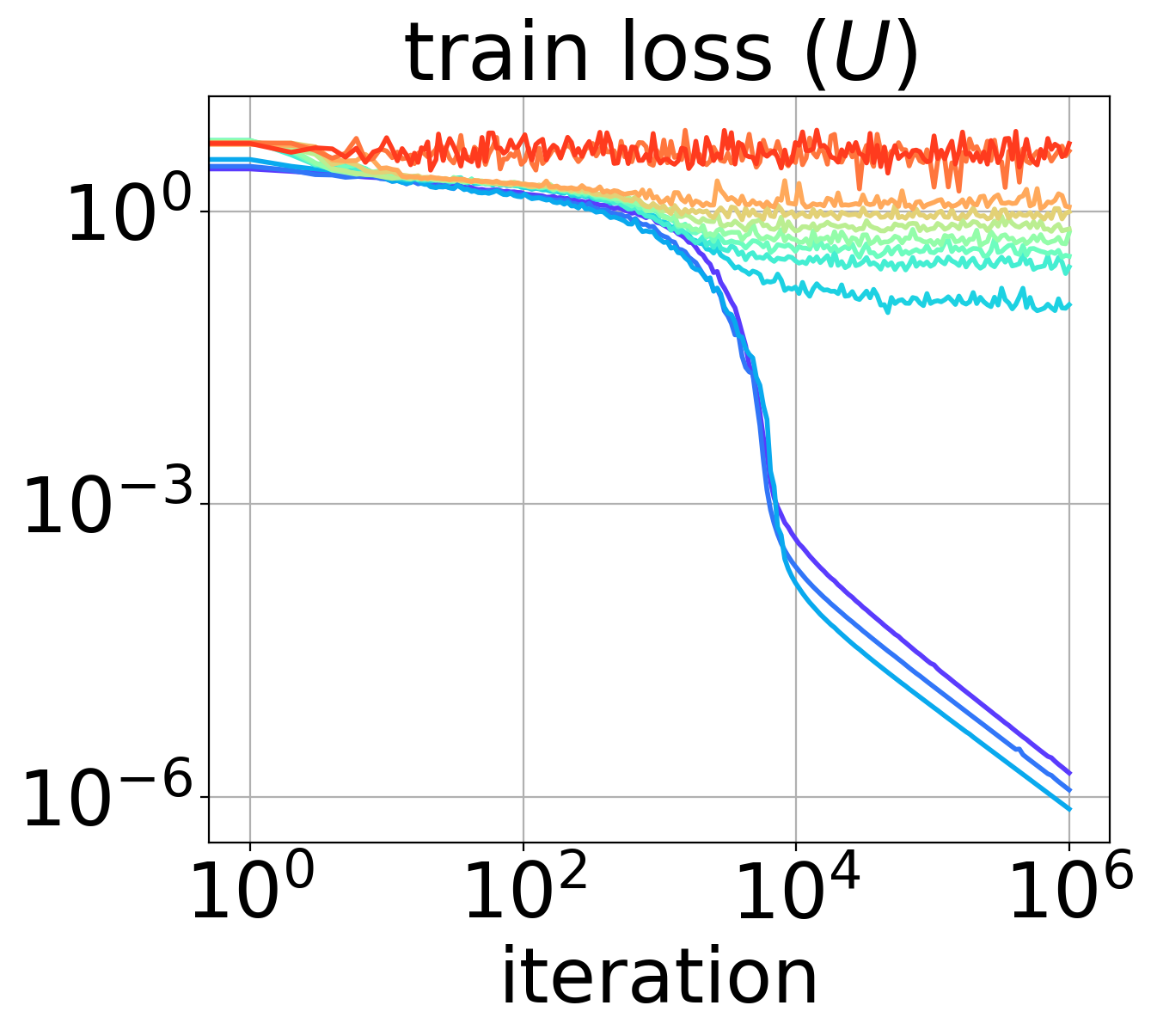} & 
        \includegraphics[width=0.208\textwidth]{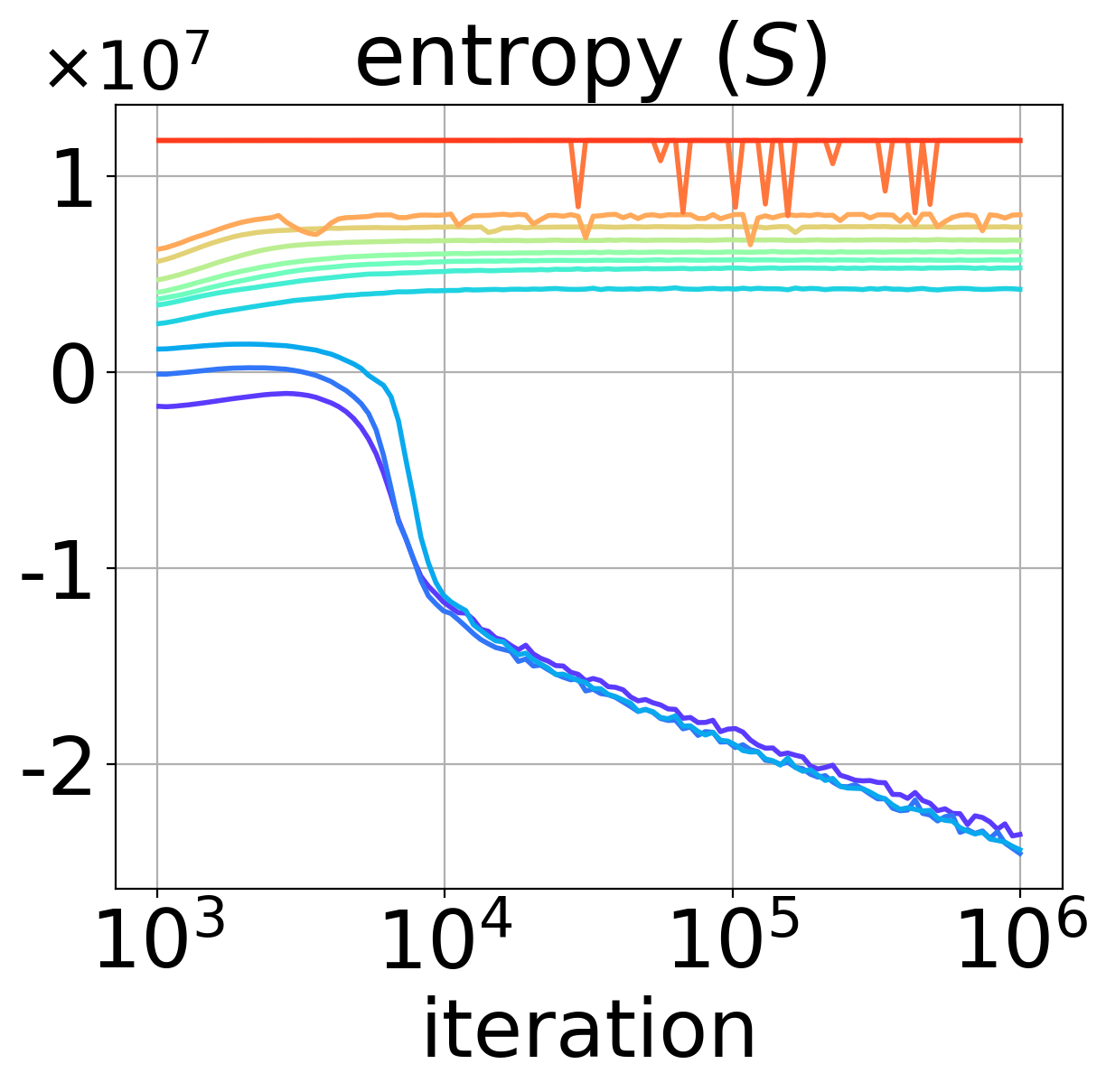} \\ 
        \multicolumn{4}{c}{
        \includegraphics[width=0.76\textwidth]{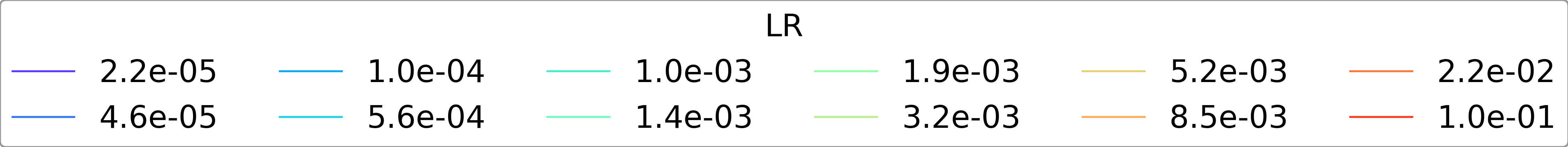}
        } \\
        \multicolumn{2}{c}{UP ResNet-18 on CIFAR-100} & \multicolumn{2}{c}{OP ResNet-18 on CIFAR-100} \\
        \includegraphics[width=0.235\textwidth]{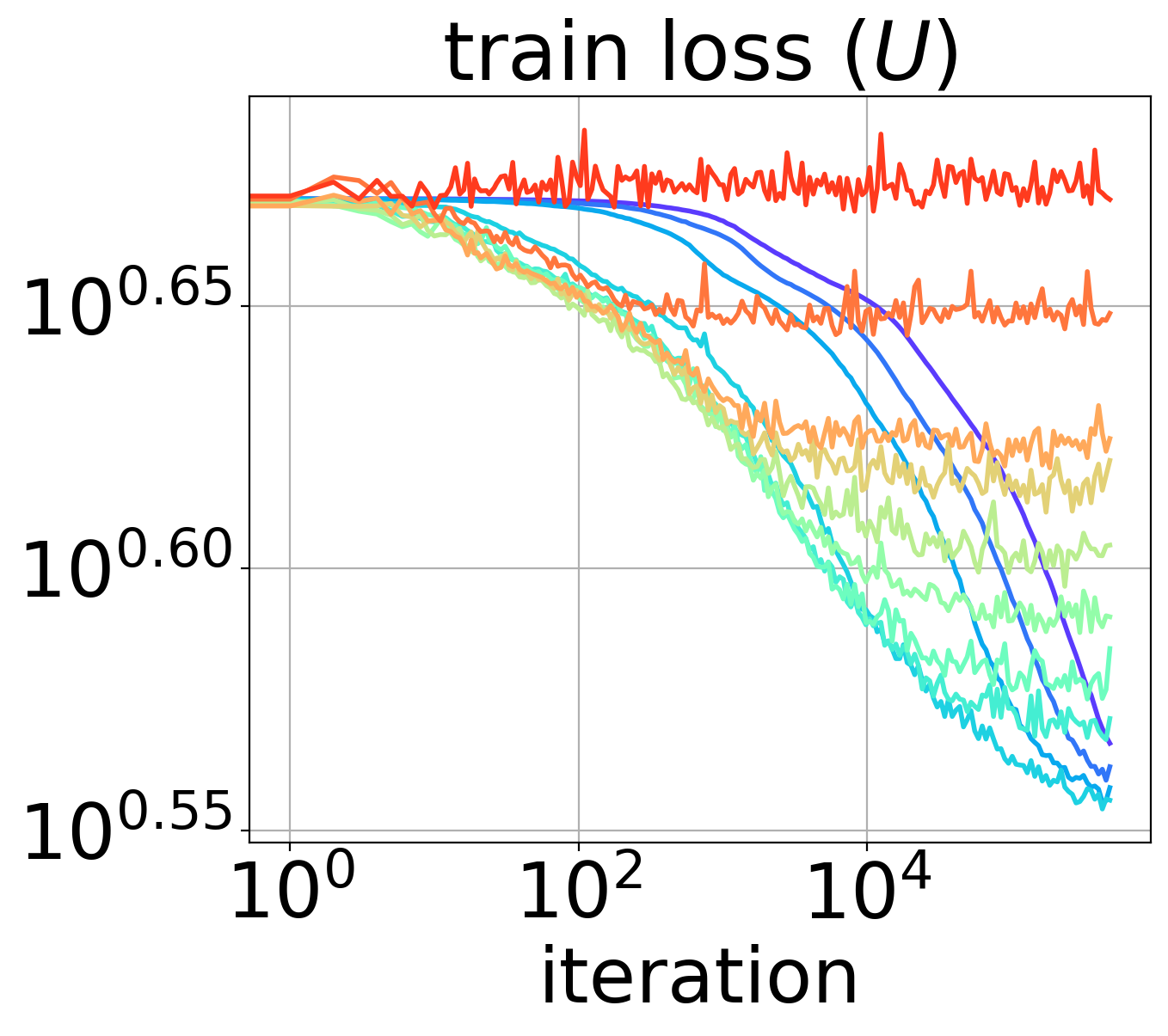} & 
        \includegraphics[width=0.203\textwidth]{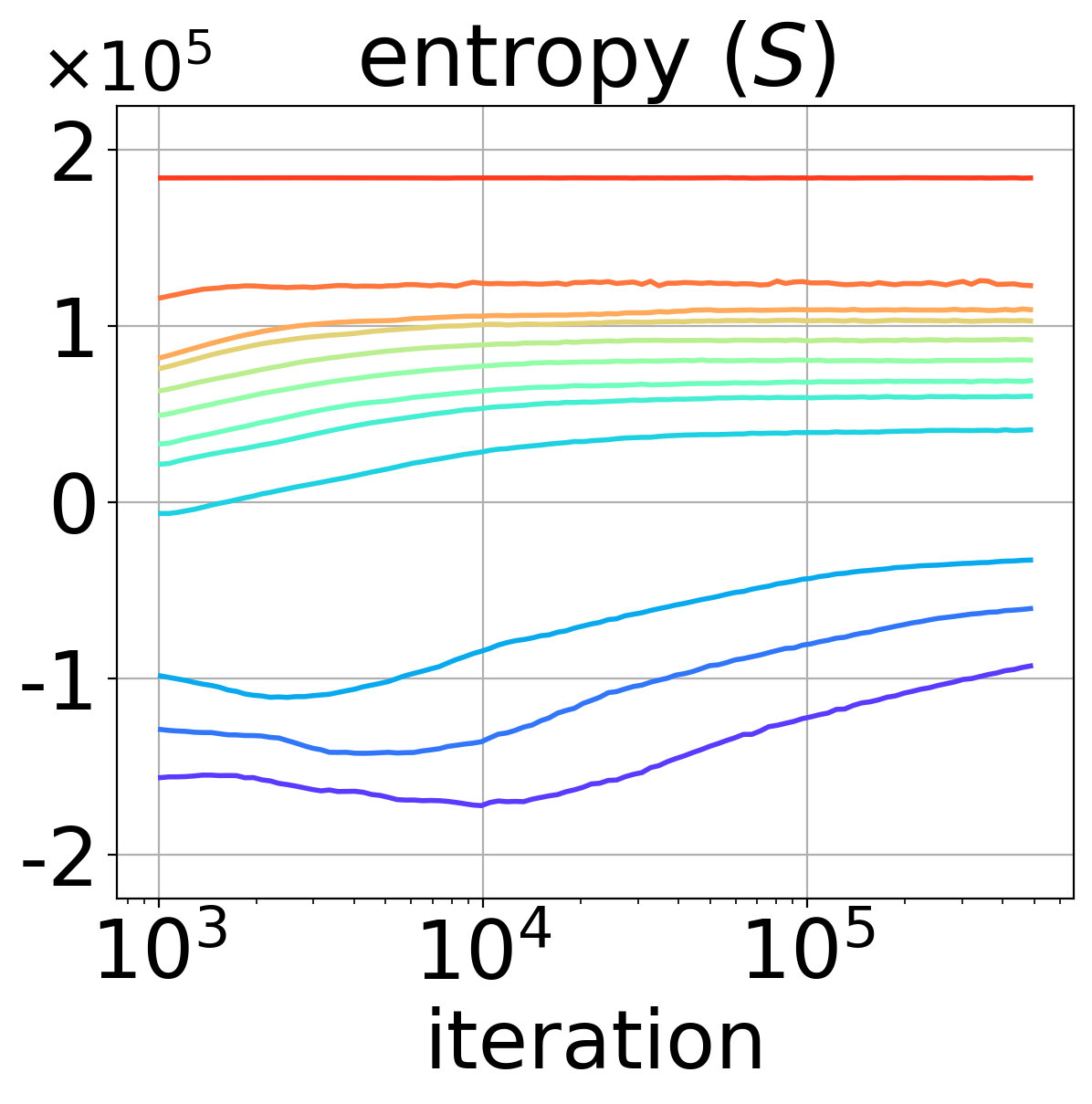} & 
        \includegraphics[width=0.225\textwidth]{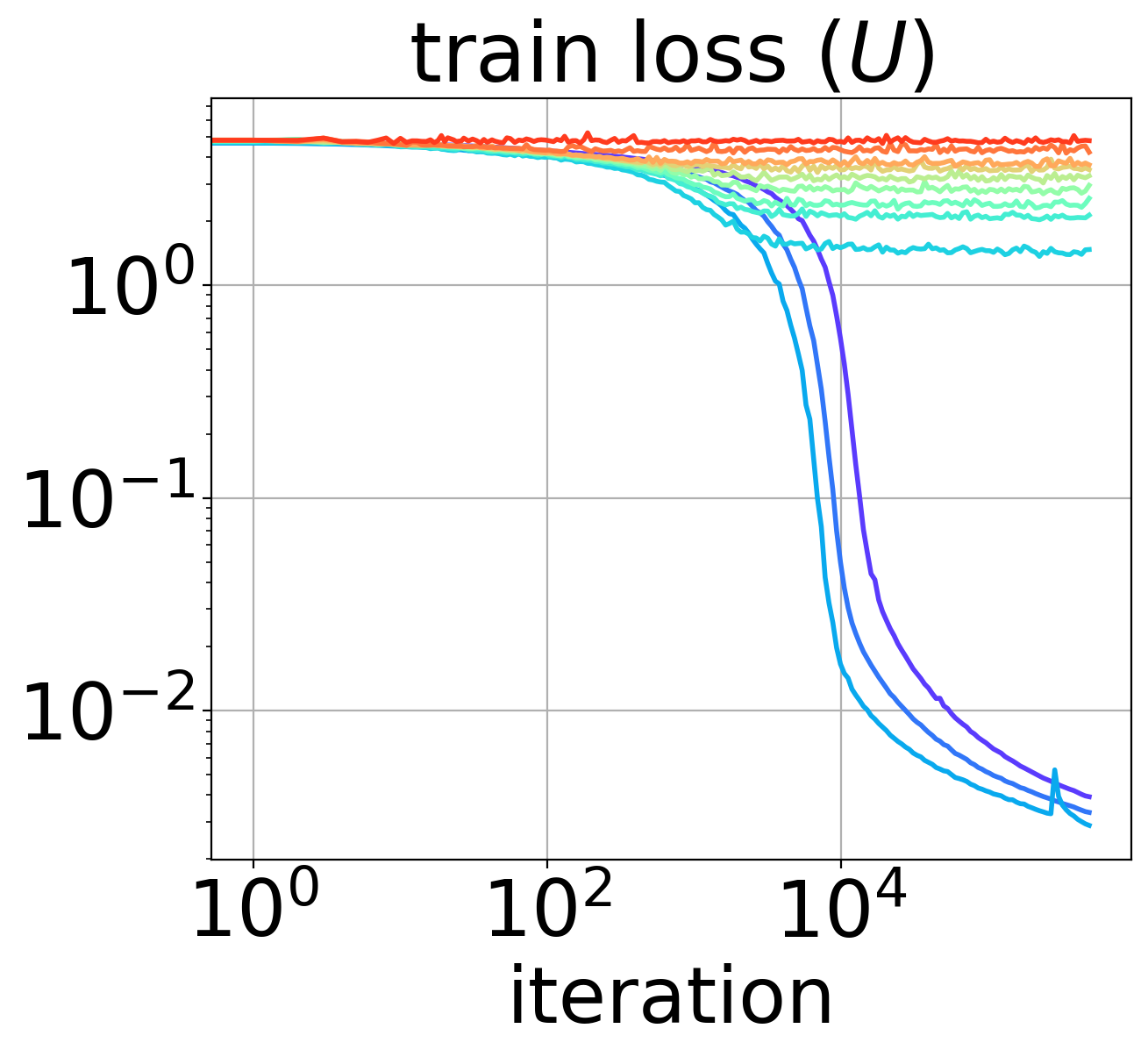} & 
        \includegraphics[width=0.208\textwidth]{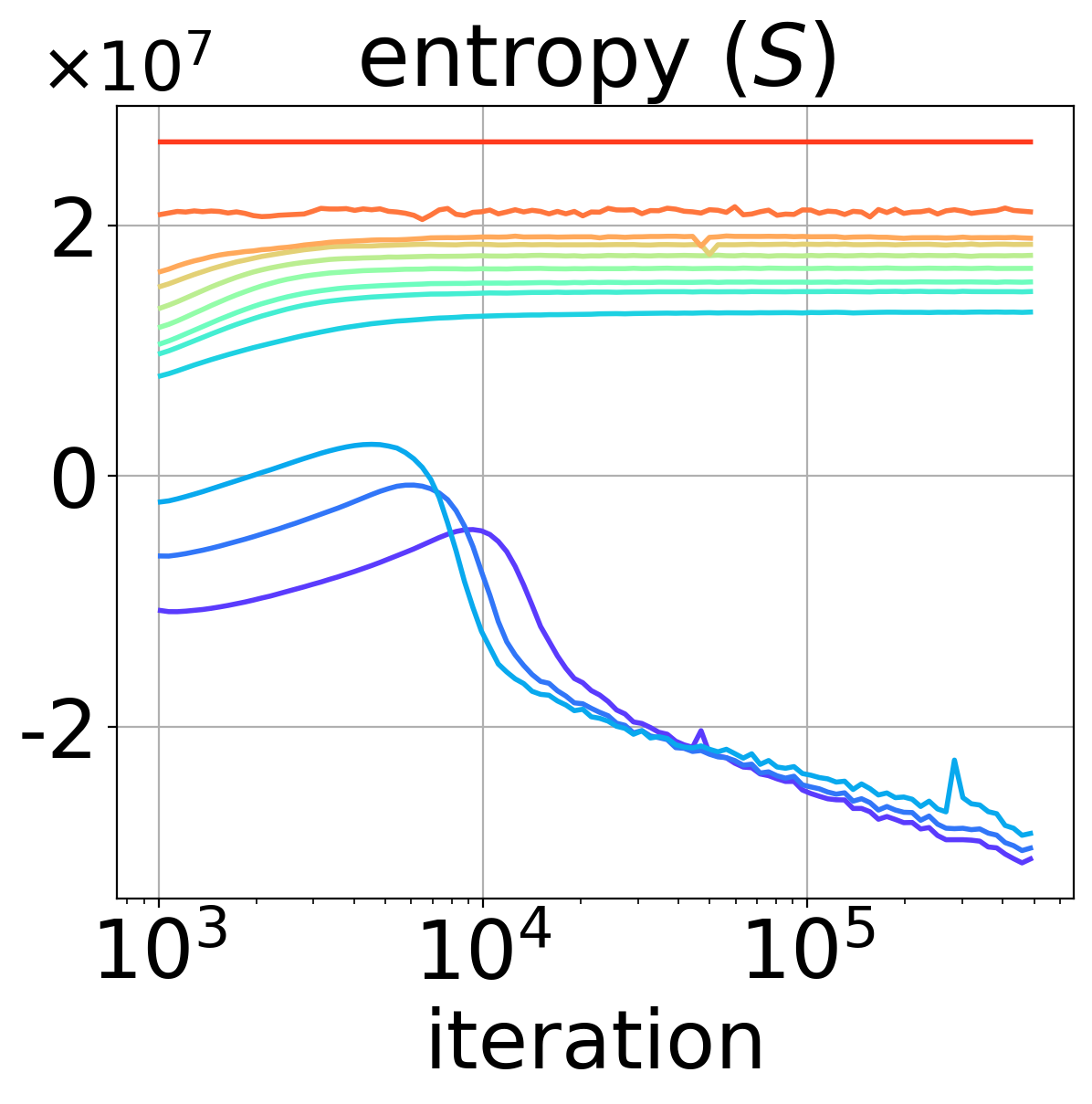} \\ 
        \multicolumn{4}{c}{
        \includegraphics[width=0.76\textwidth]{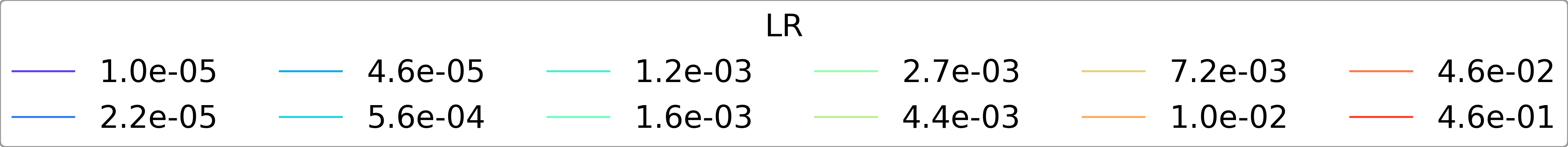}
        } \\
        \multicolumn{2}{c}{UP ConvNet on CIFAR-100} & \multicolumn{2}{c}{OP ConvNet on CIFAR-100} \\
        \includegraphics[width=0.225\textwidth]{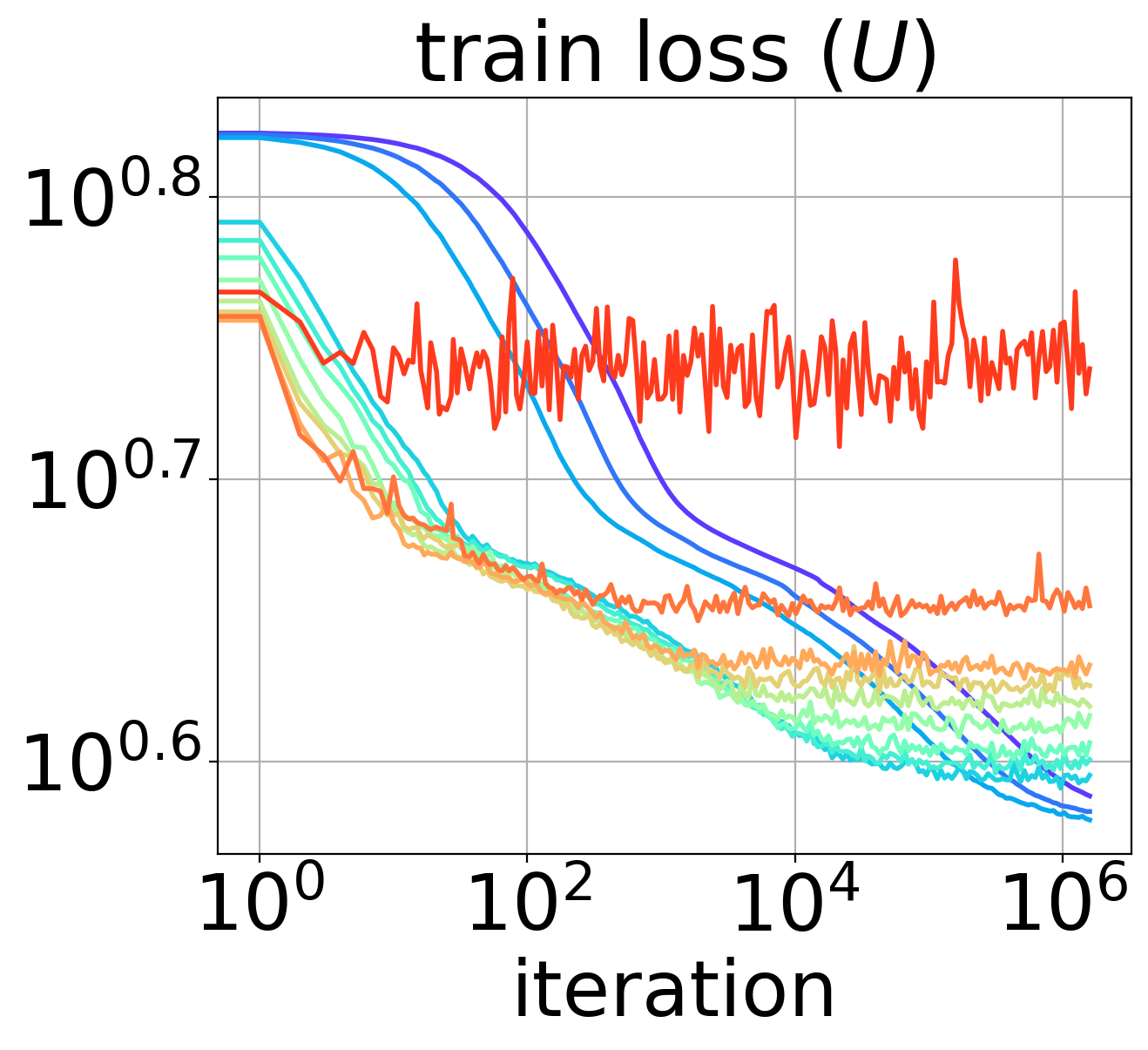} & 
        \includegraphics[width=0.203\textwidth]{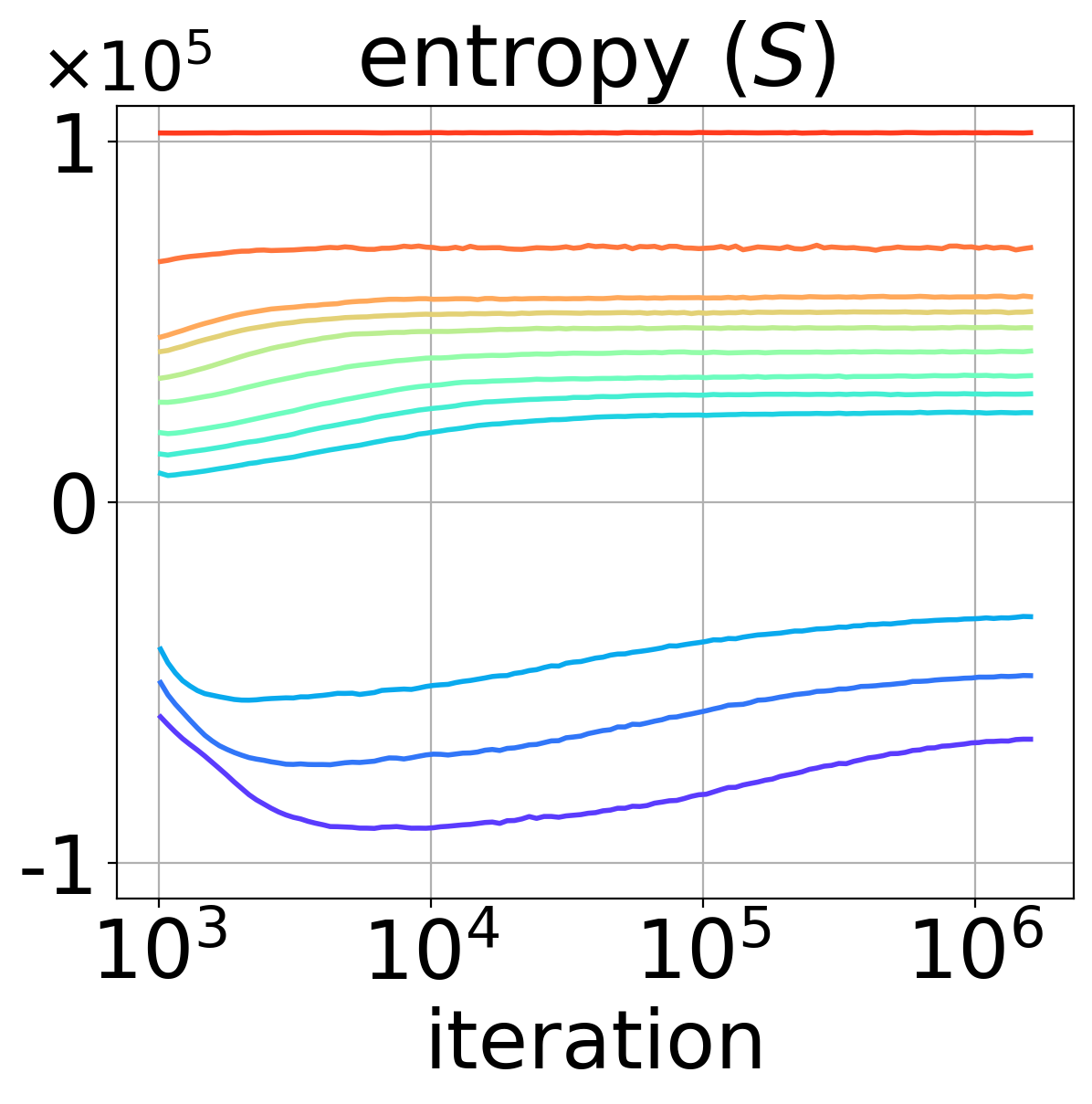} & 
        \includegraphics[width=0.225\textwidth]{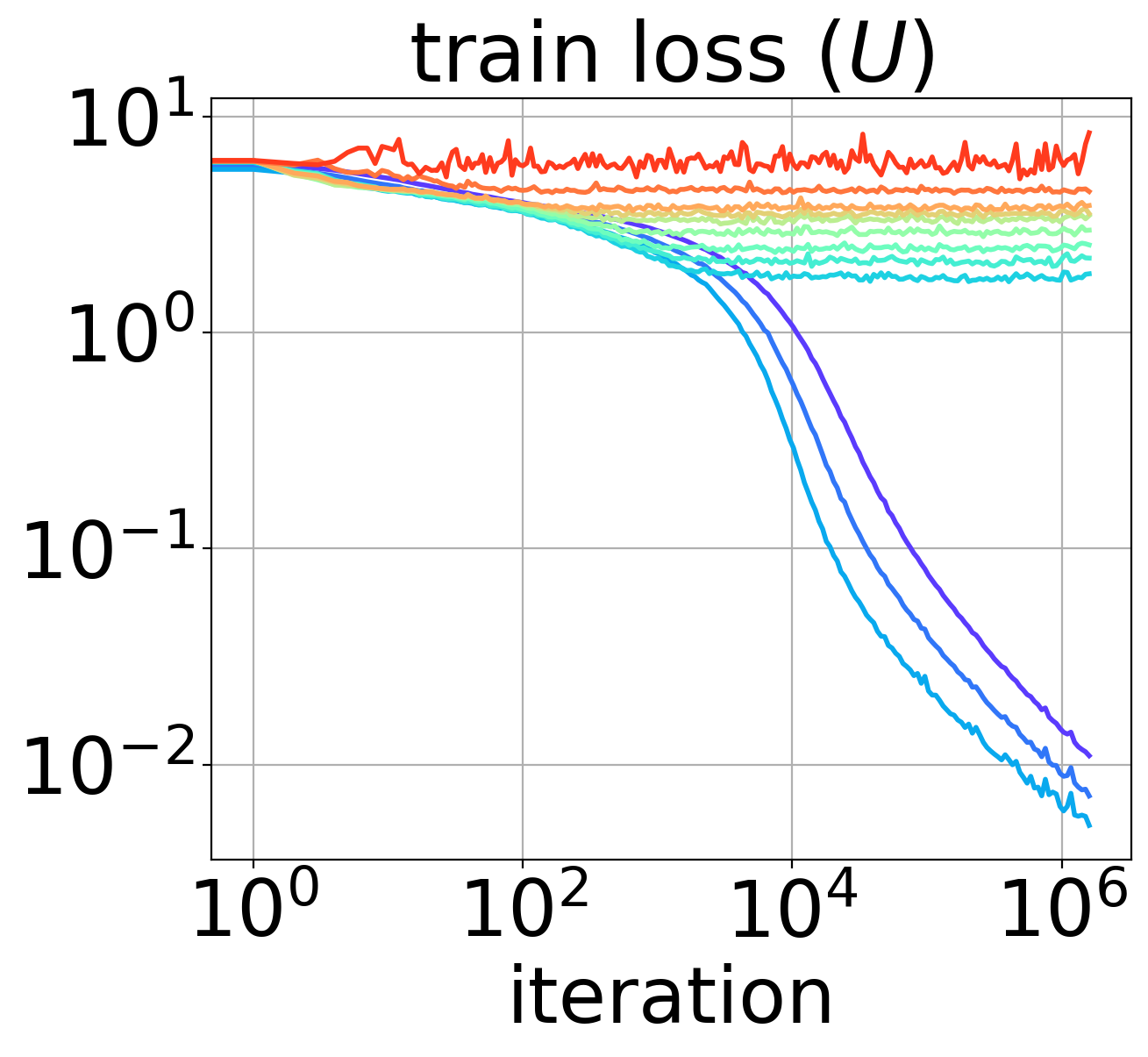} & 
        \includegraphics[width=0.208\textwidth]{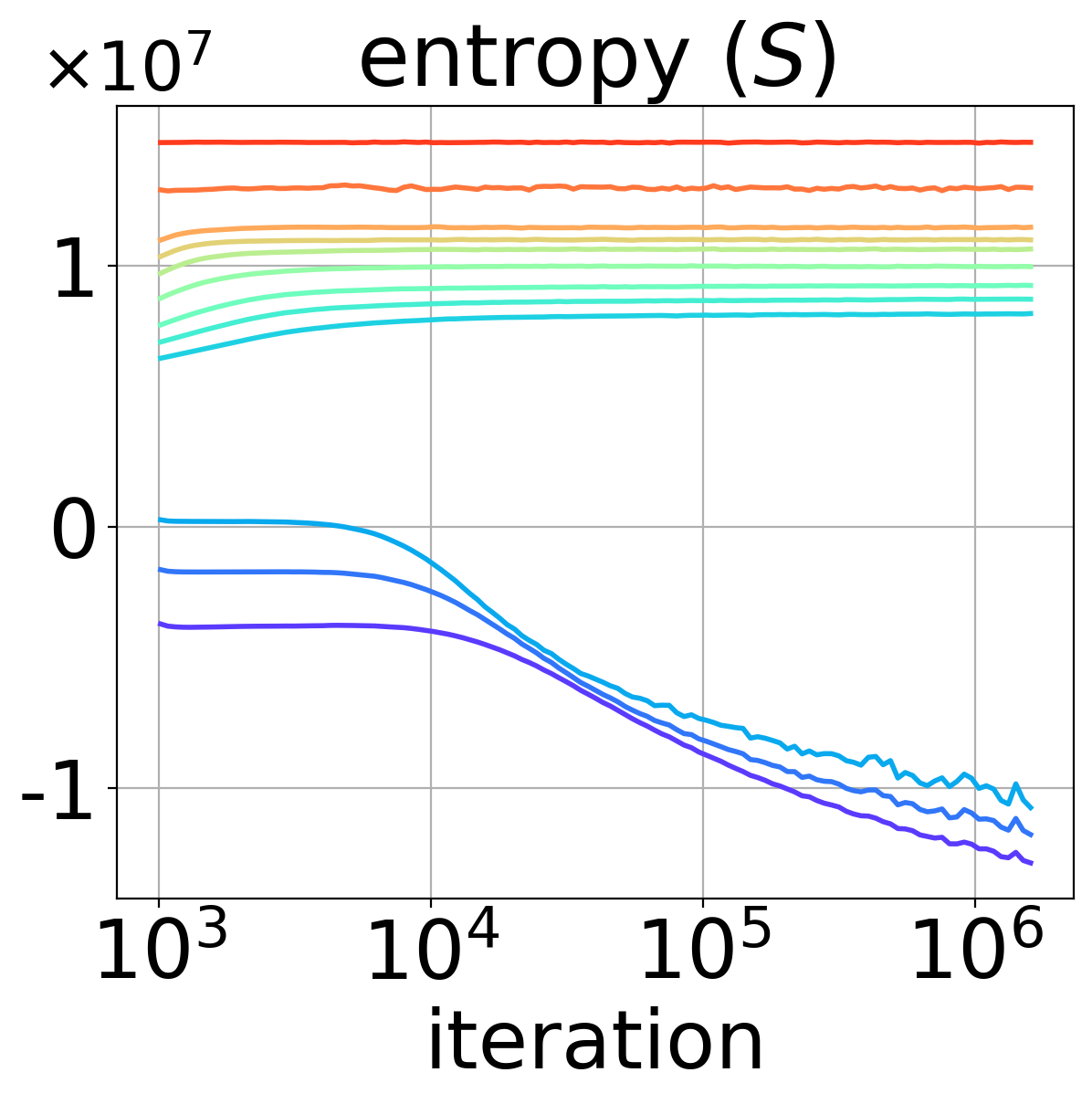} \\ 
        \multicolumn{4}{c}{
        \includegraphics[width=0.76\textwidth]{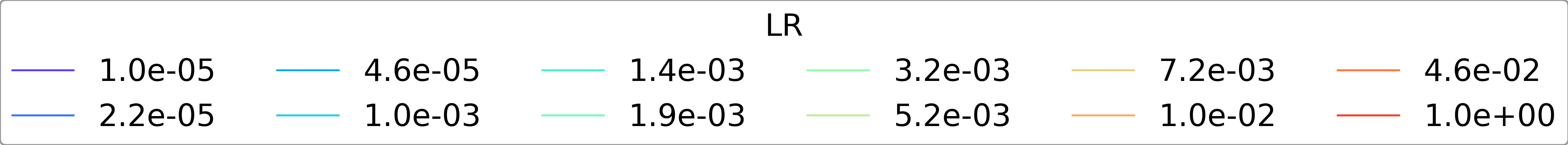}
        } 
    \end{tabular}
    \caption{Stationary loss and entropy for UP (left two columns) and OP (right two columns) settings. This figure complements Figure~\ref{fig:loss_entropy_iters} from the main text.}
    \label{fig:app_loss_entropy_iters}
\end{figure}

\begin{figure}[h!]
    \centering
    \addtolength{\tabcolsep}{-0.4em}

    \begin{tabular}{ccc}
        \multicolumn{3}{c}{UP ResNet-18 on CIFAR-10} \\
        \includegraphics[width=0.33\textwidth]{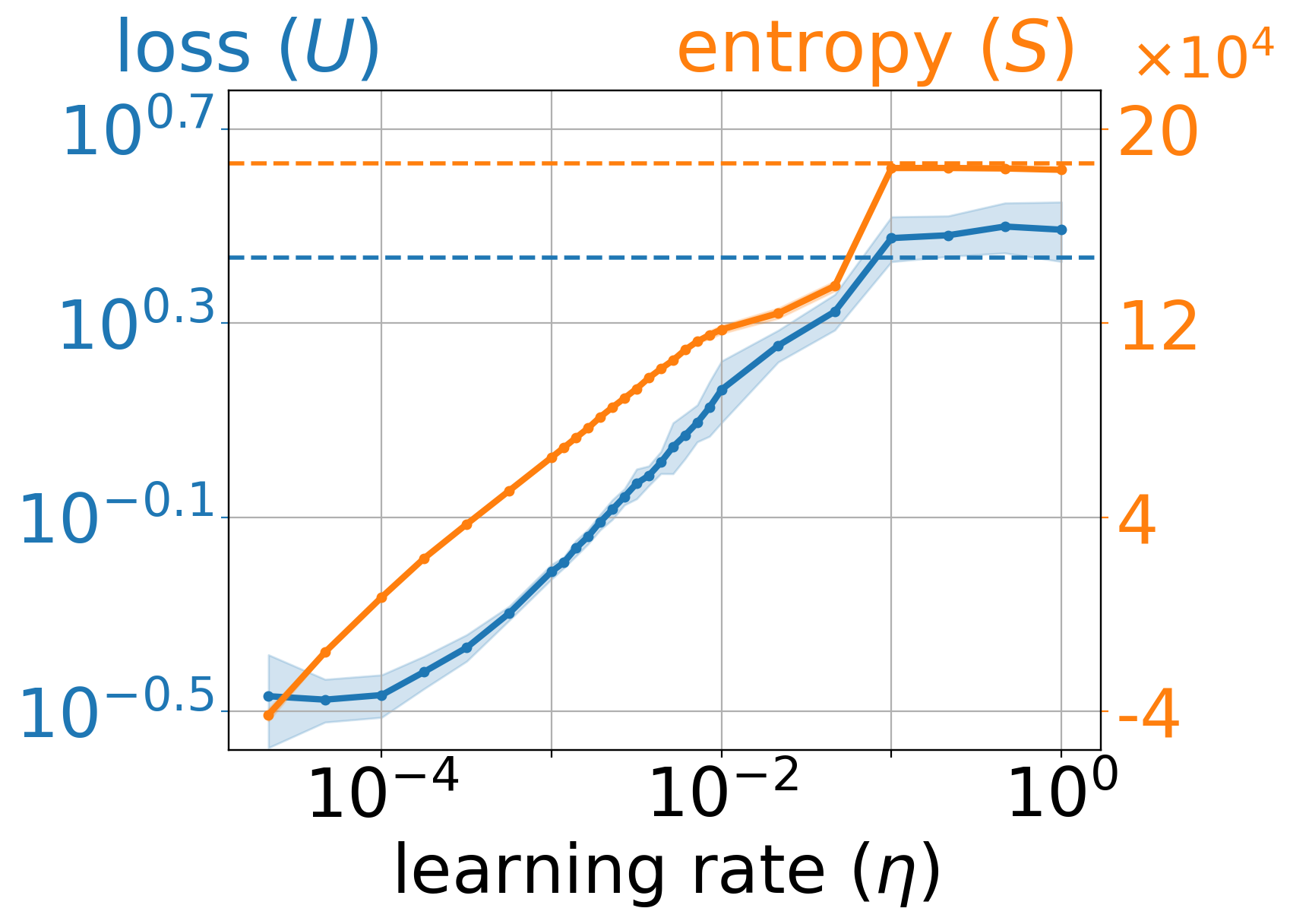} &
        \includegraphics[width=0.31\textwidth]{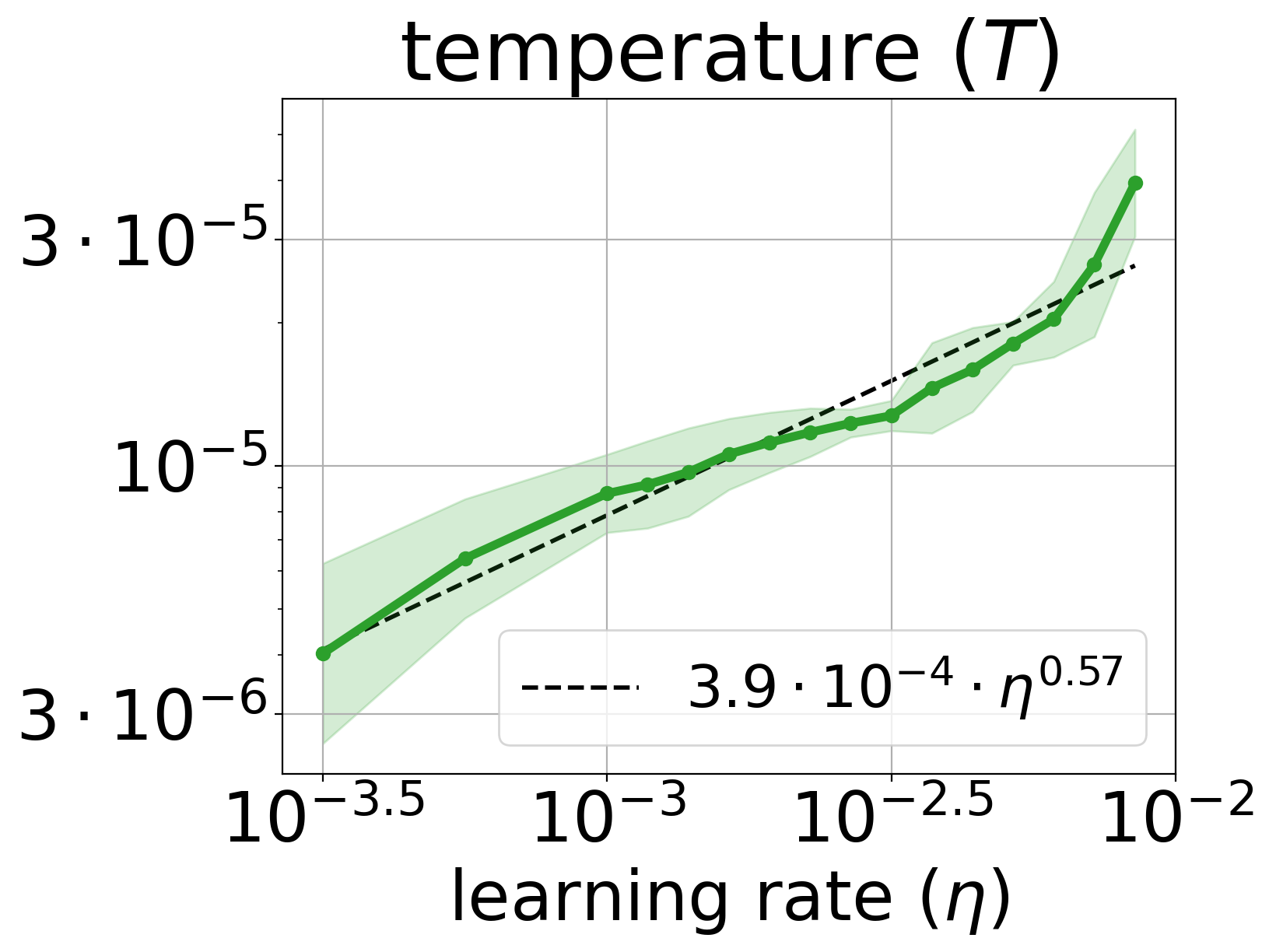} & 
        \includegraphics[width=0.35\textwidth]{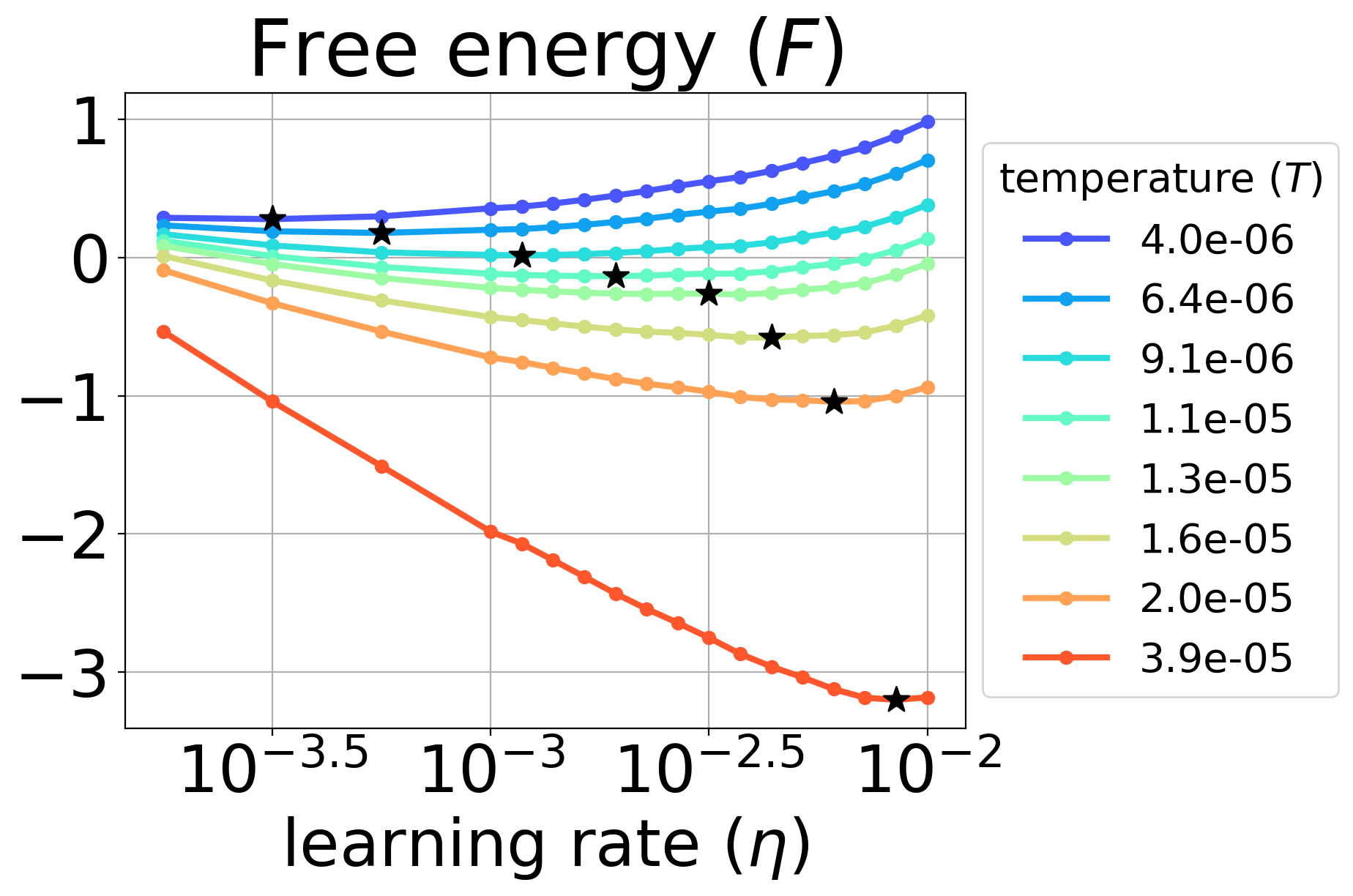} \\
        \multicolumn{3}{c}{OP ResNet-18 on CIFAR-10} \\
        \includegraphics[width=0.33\textwidth]{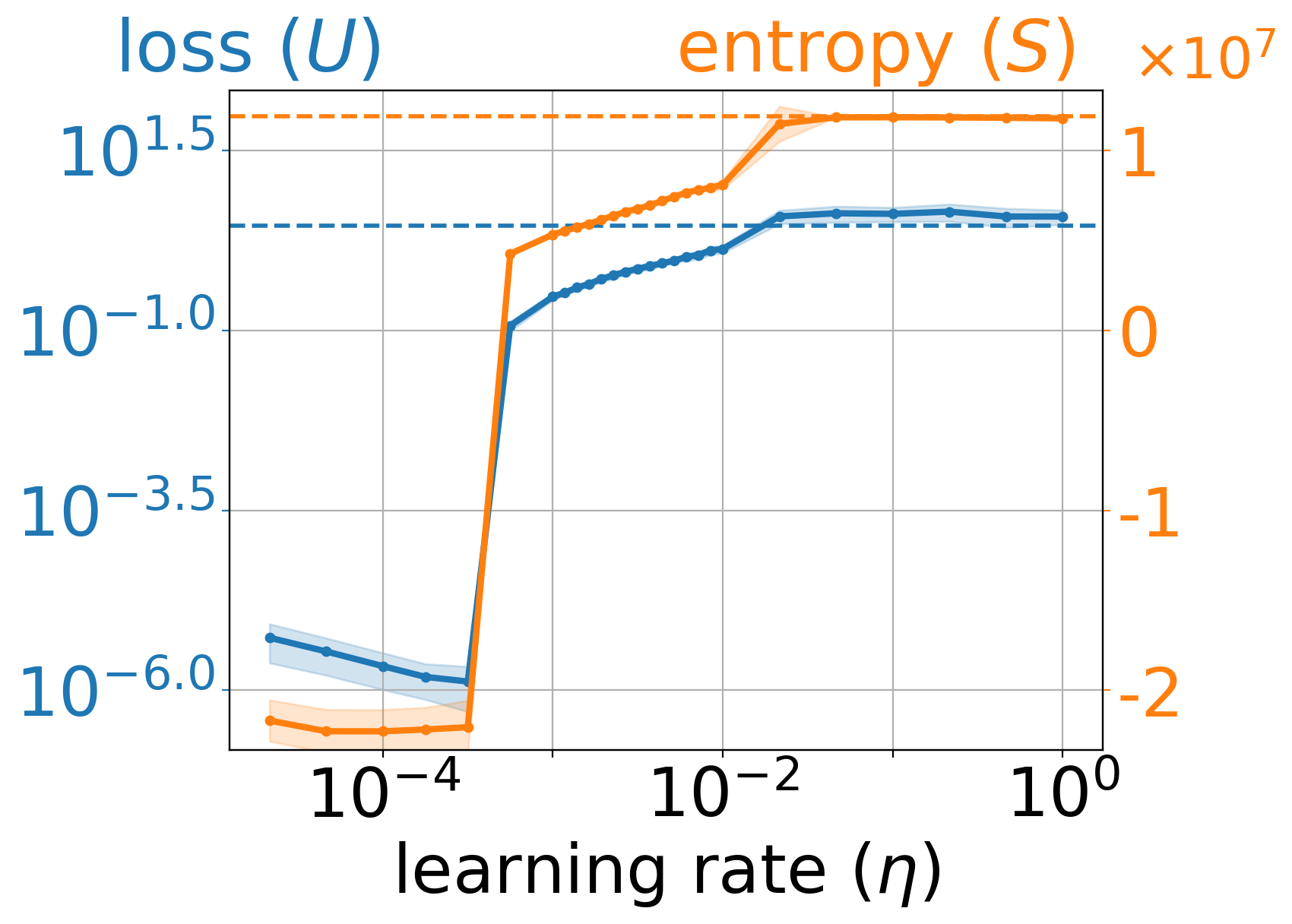} &
        \includegraphics[width=0.307\textwidth]{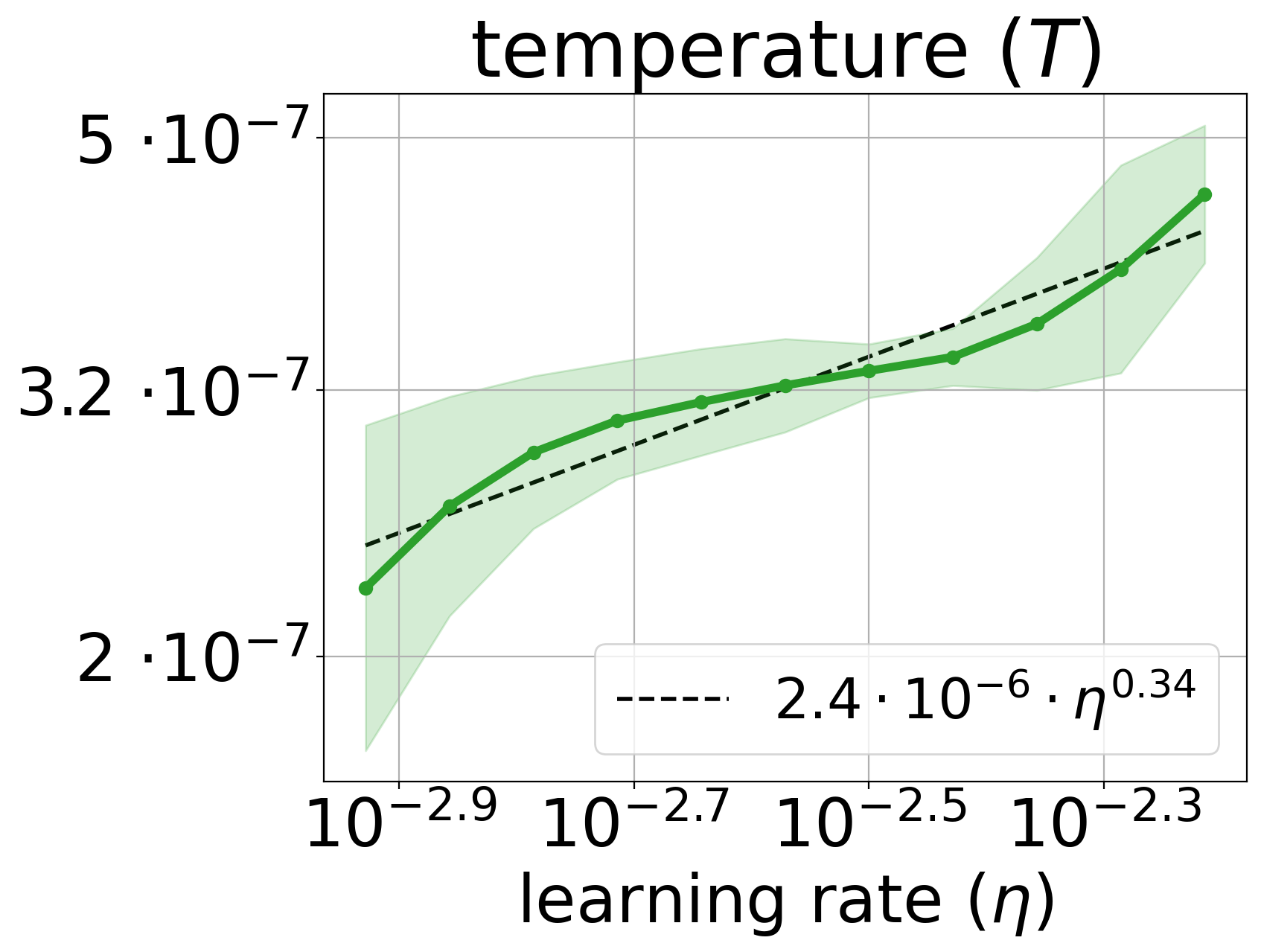} & 
        \includegraphics[width=0.35\textwidth]{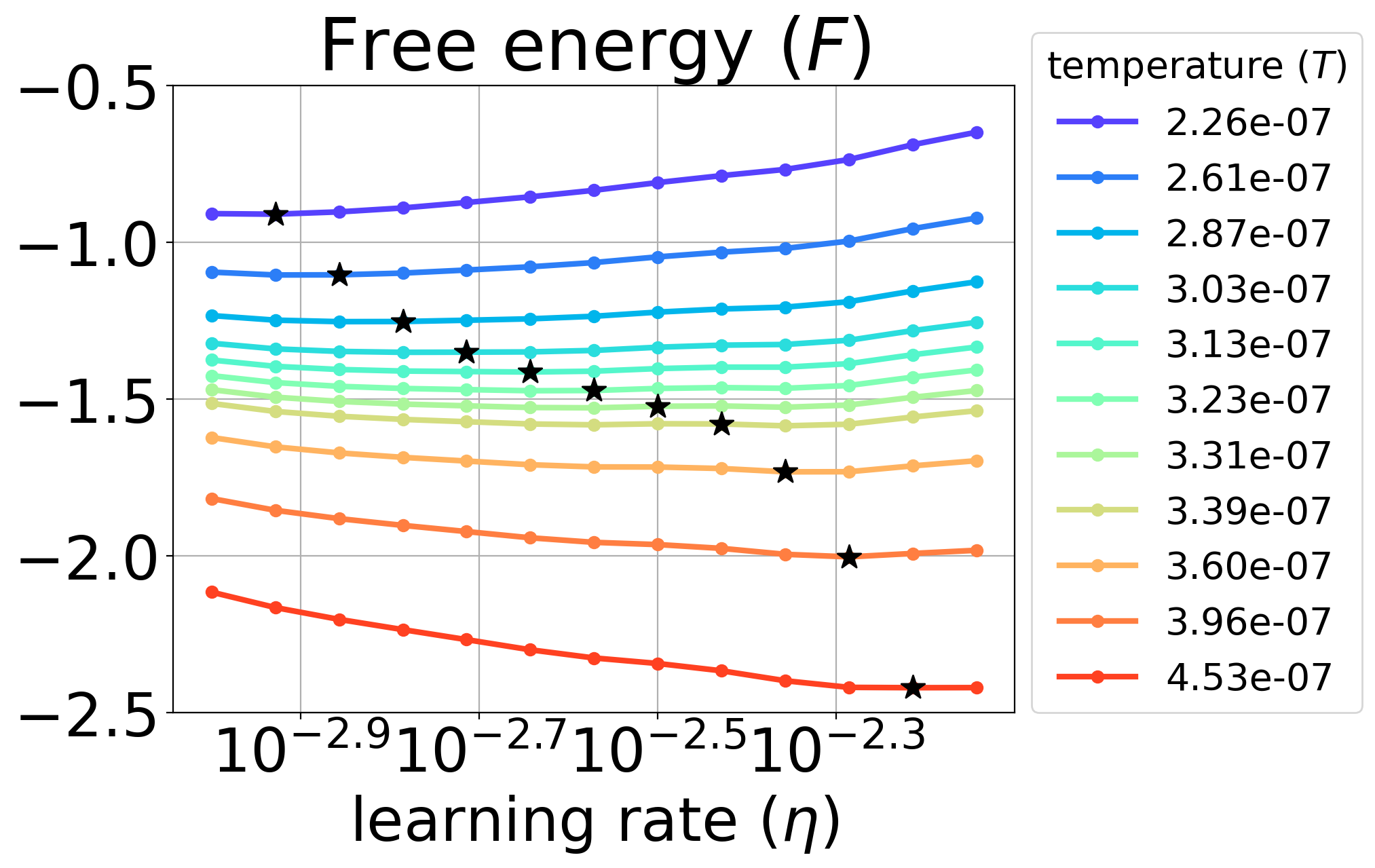} \\
        \multicolumn{3}{c}{UP ResNet-18 on CIFAR-100} \\
        \includegraphics[width=0.33\textwidth]{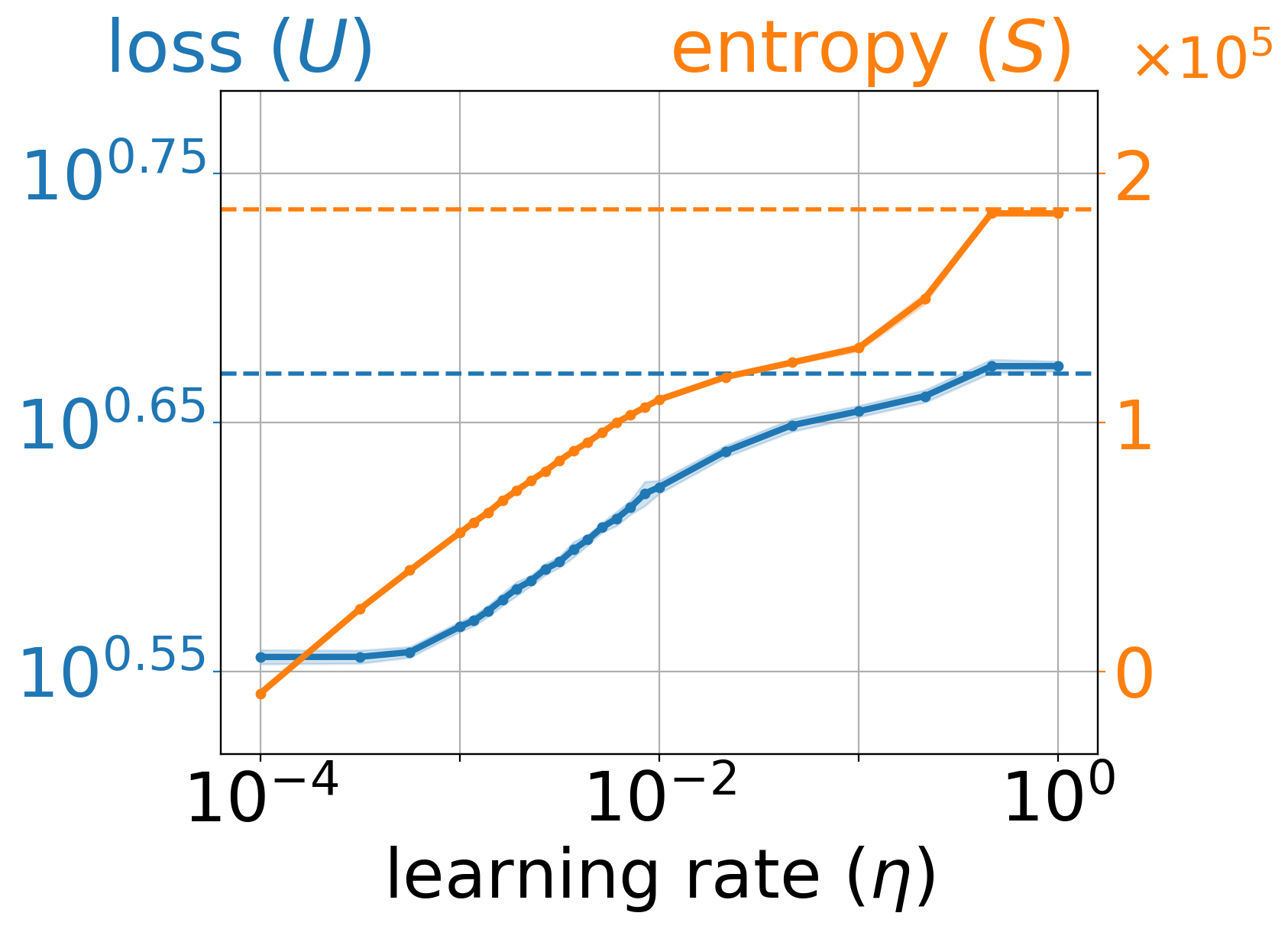} &
        \includegraphics[width=0.31\textwidth]{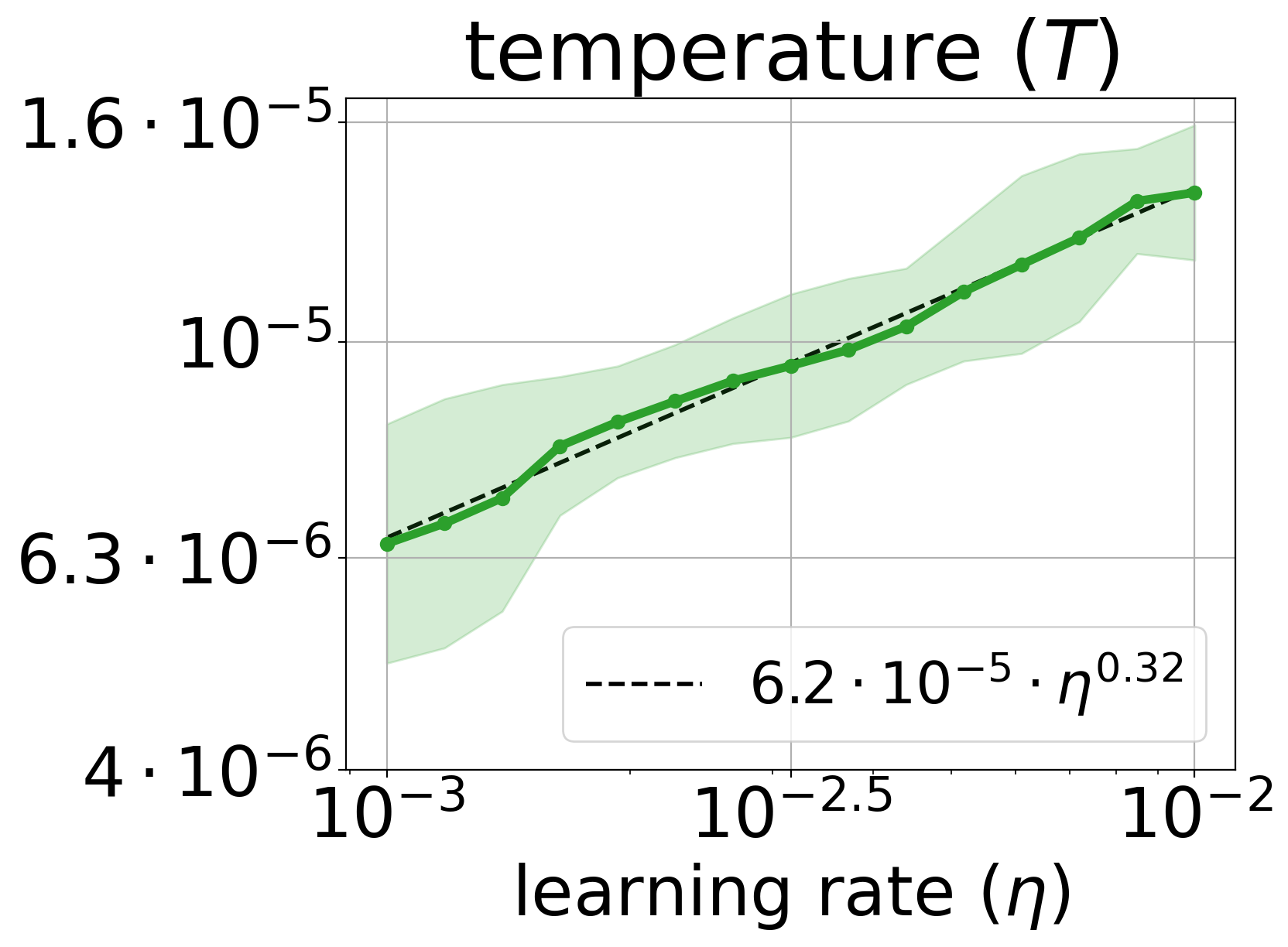} & 
        \includegraphics[width=0.35\textwidth]{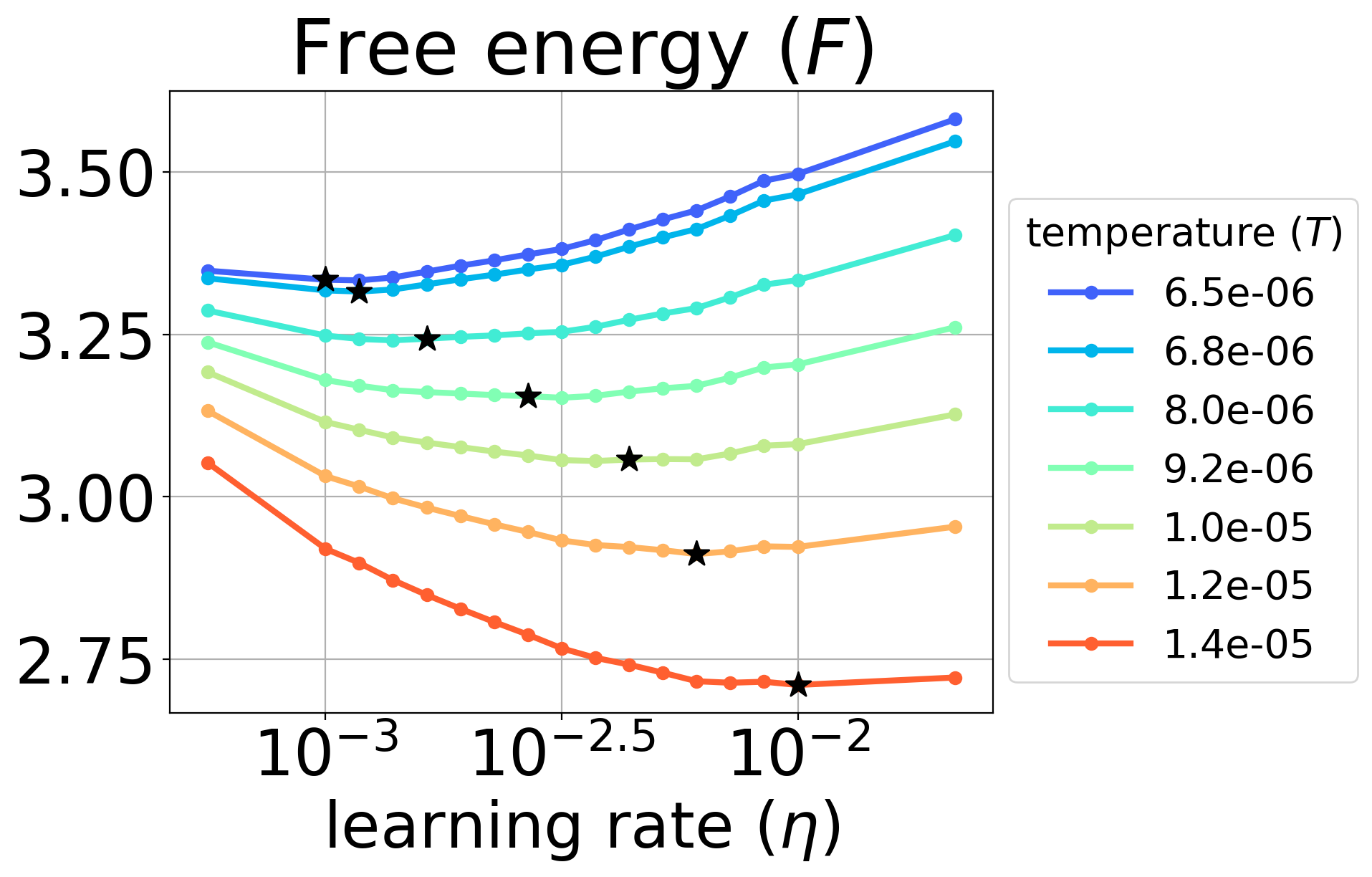} \\
        \multicolumn{3}{c}{OP ResNet-18 on CIFAR-100} \\
        \includegraphics[width=0.32\textwidth]{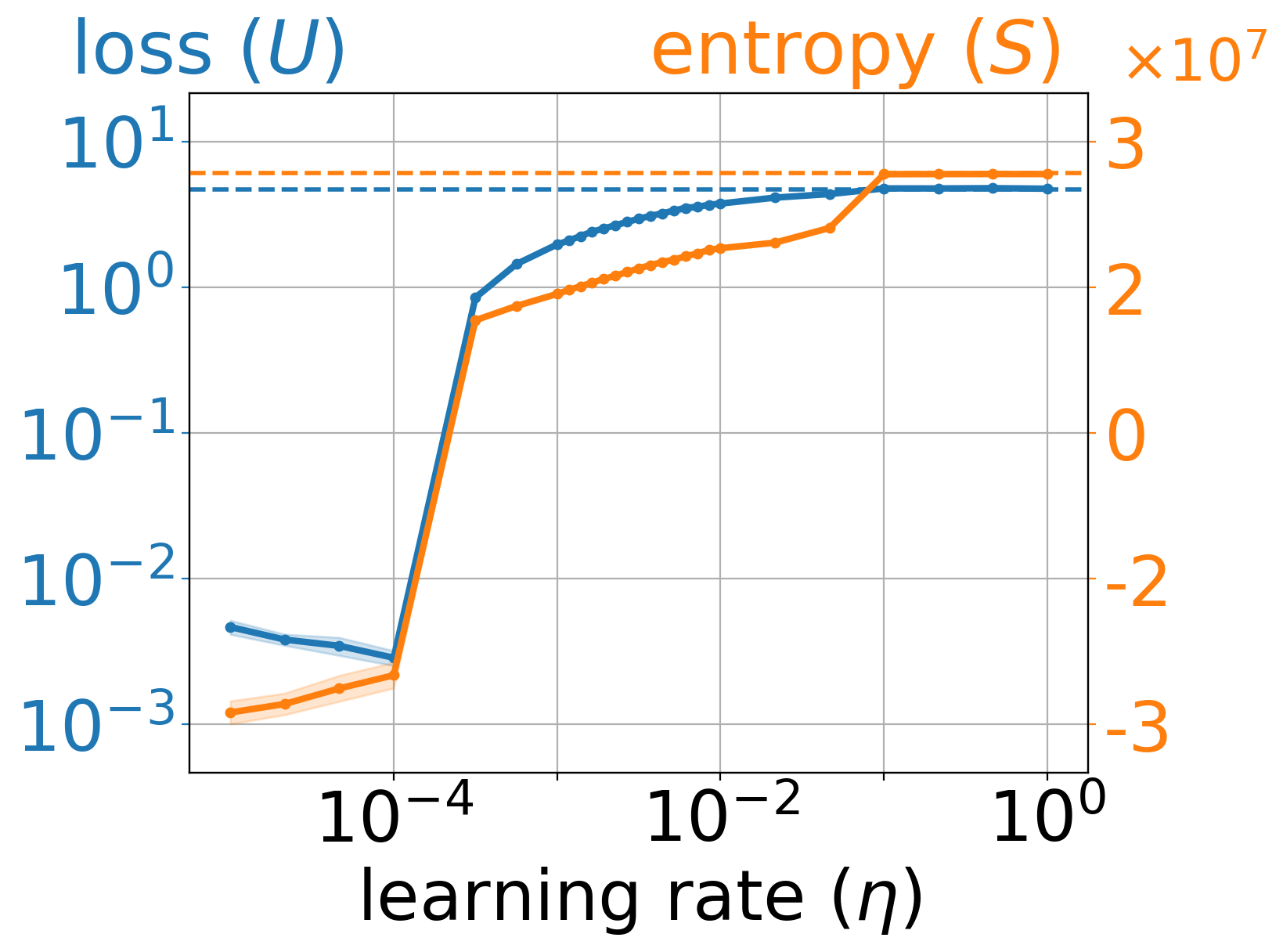} &
        \includegraphics[width=0.32\textwidth]{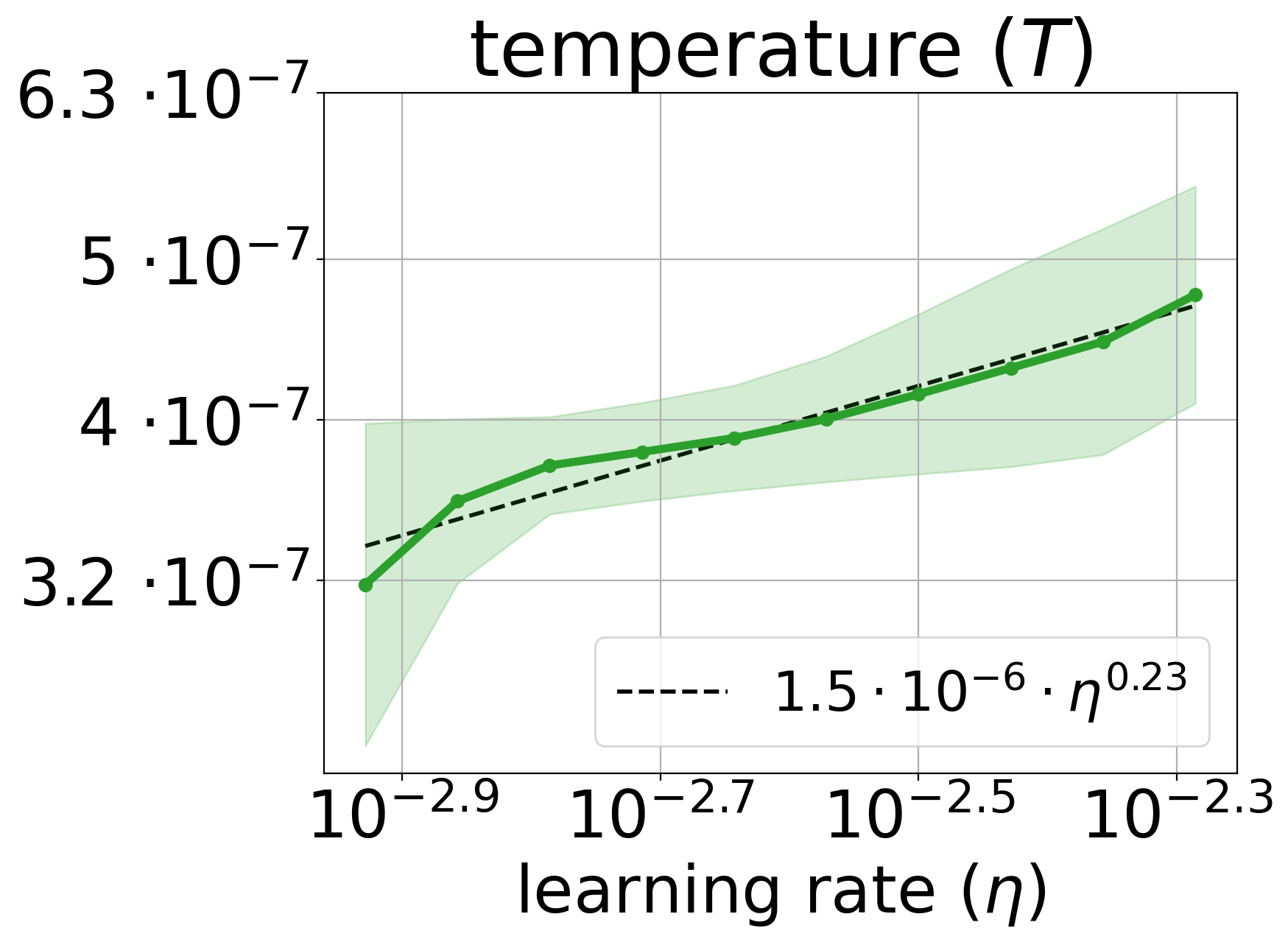} & 
        \includegraphics[width=0.36\textwidth]{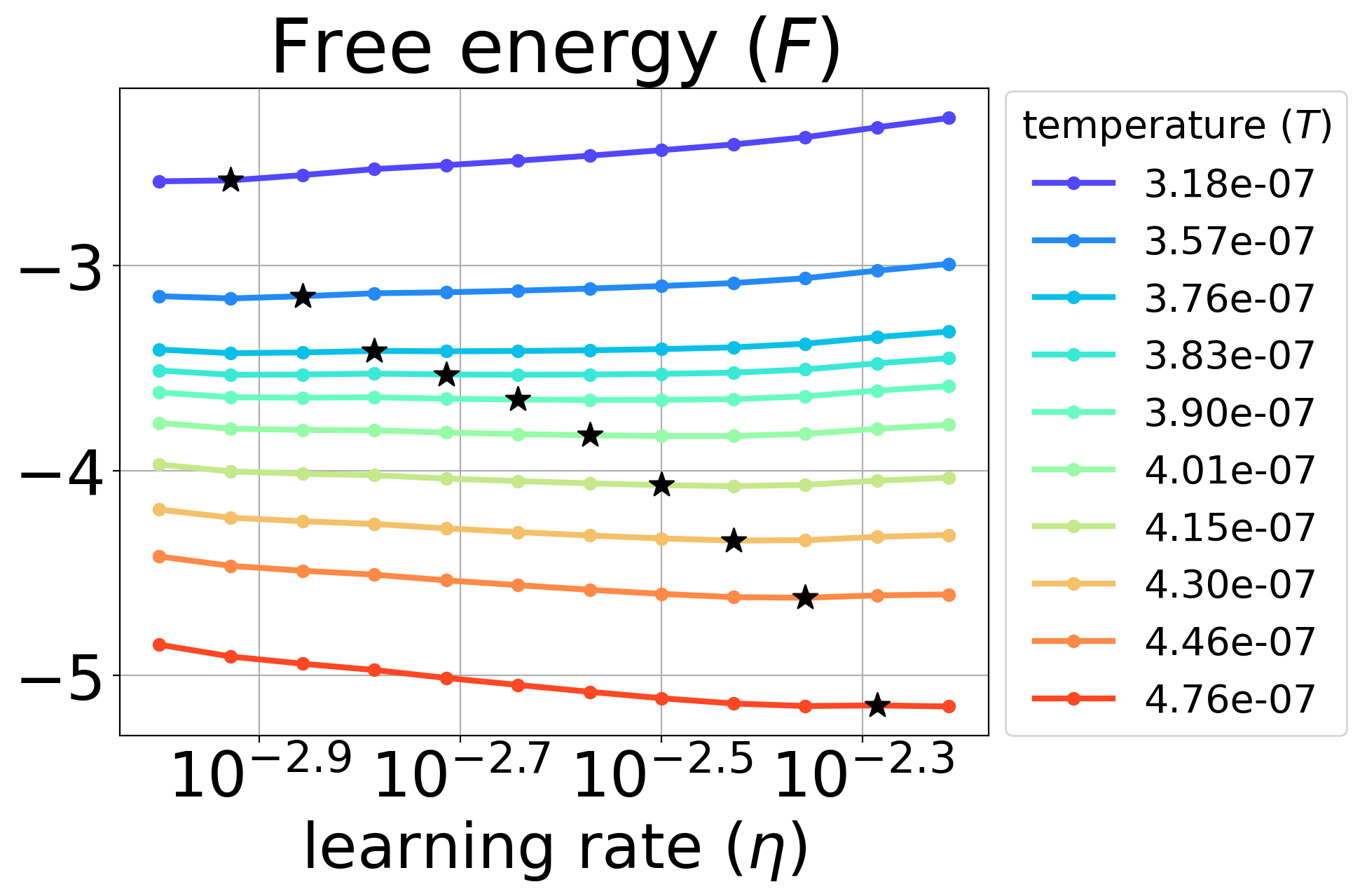} \\
    \end{tabular}
    \caption{Loss and entropy (left), estimated temperature (center) and free energy for different temperatures (right) vs. LR for ResNet-18 on CIFAR-10 and CIFAR-100 datasets. The figure extends Figure~\ref{fig:temperature_up_op} from the main text.}
    \label{fig:app_temperature_up_op_resnet}
\end{figure}

\begin{figure}[h]
    \centering
    \addtolength{\tabcolsep}{-0.4em}

    \begin{tabular}{ccc}
        \multicolumn{3}{c}{UP ConvNet on CIFAR-100} \\
        \includegraphics[width=0.33\textwidth]{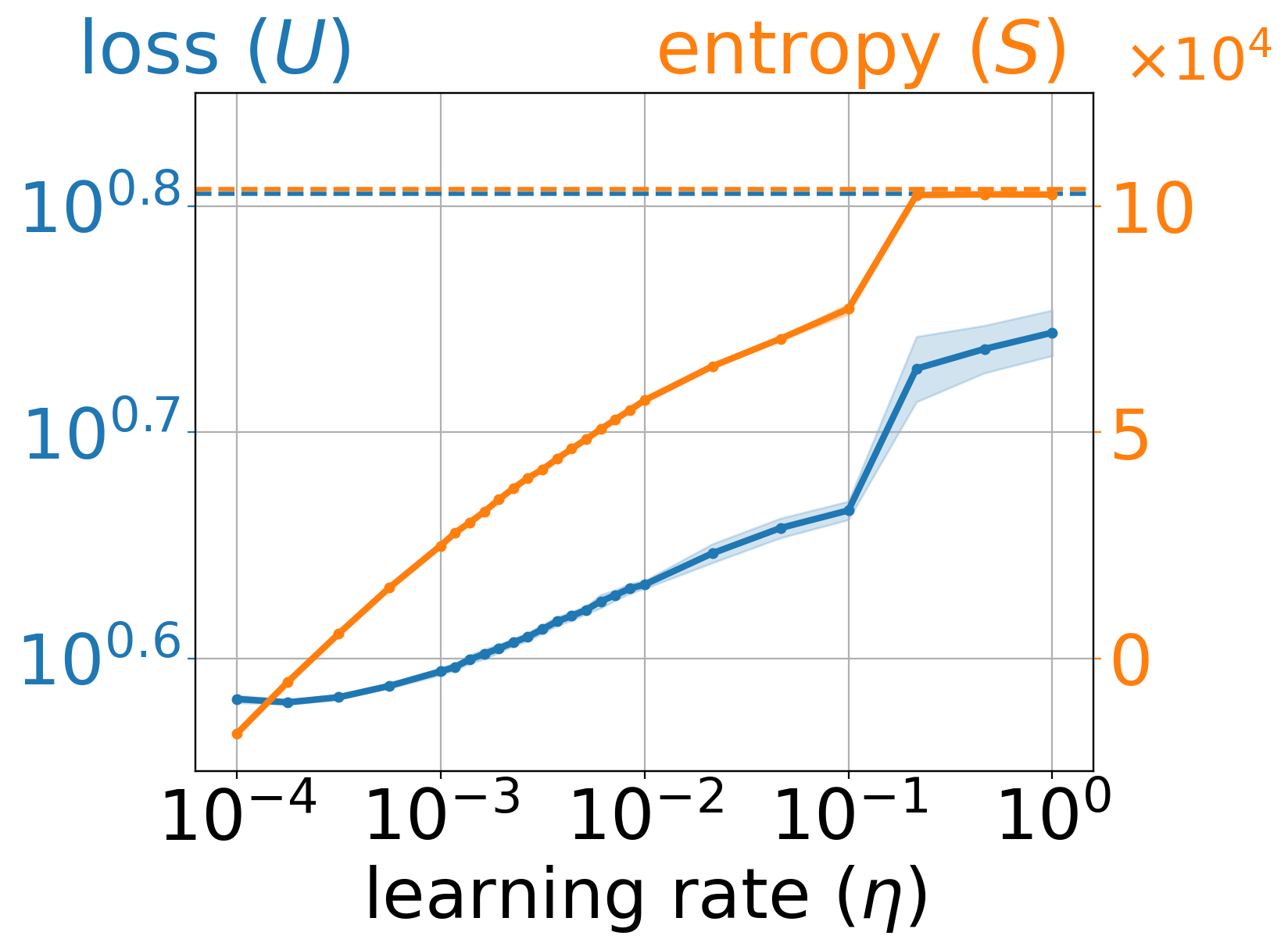} &
        \includegraphics[width=0.31\textwidth]{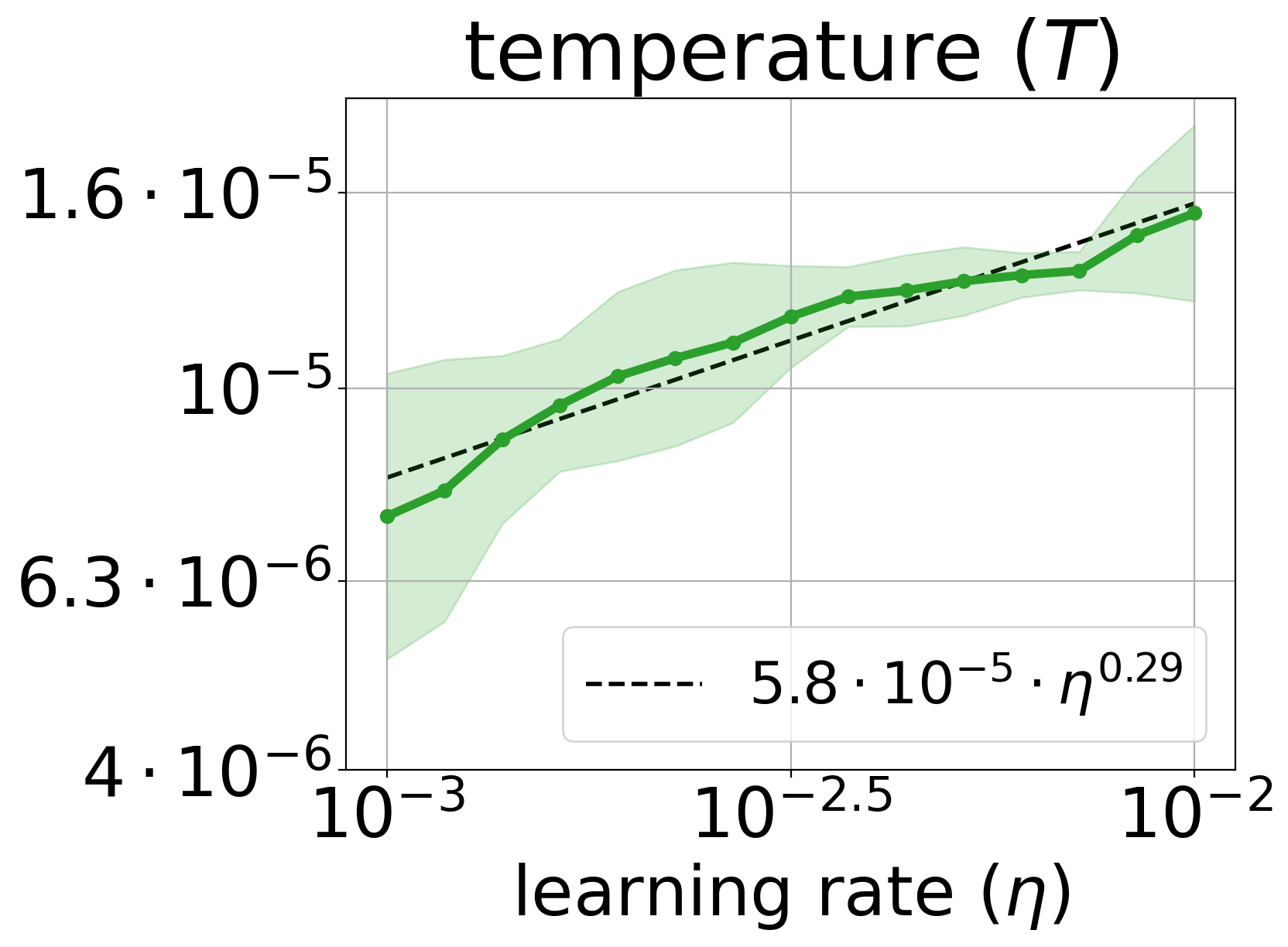} & 
        \includegraphics[width=0.35\textwidth]{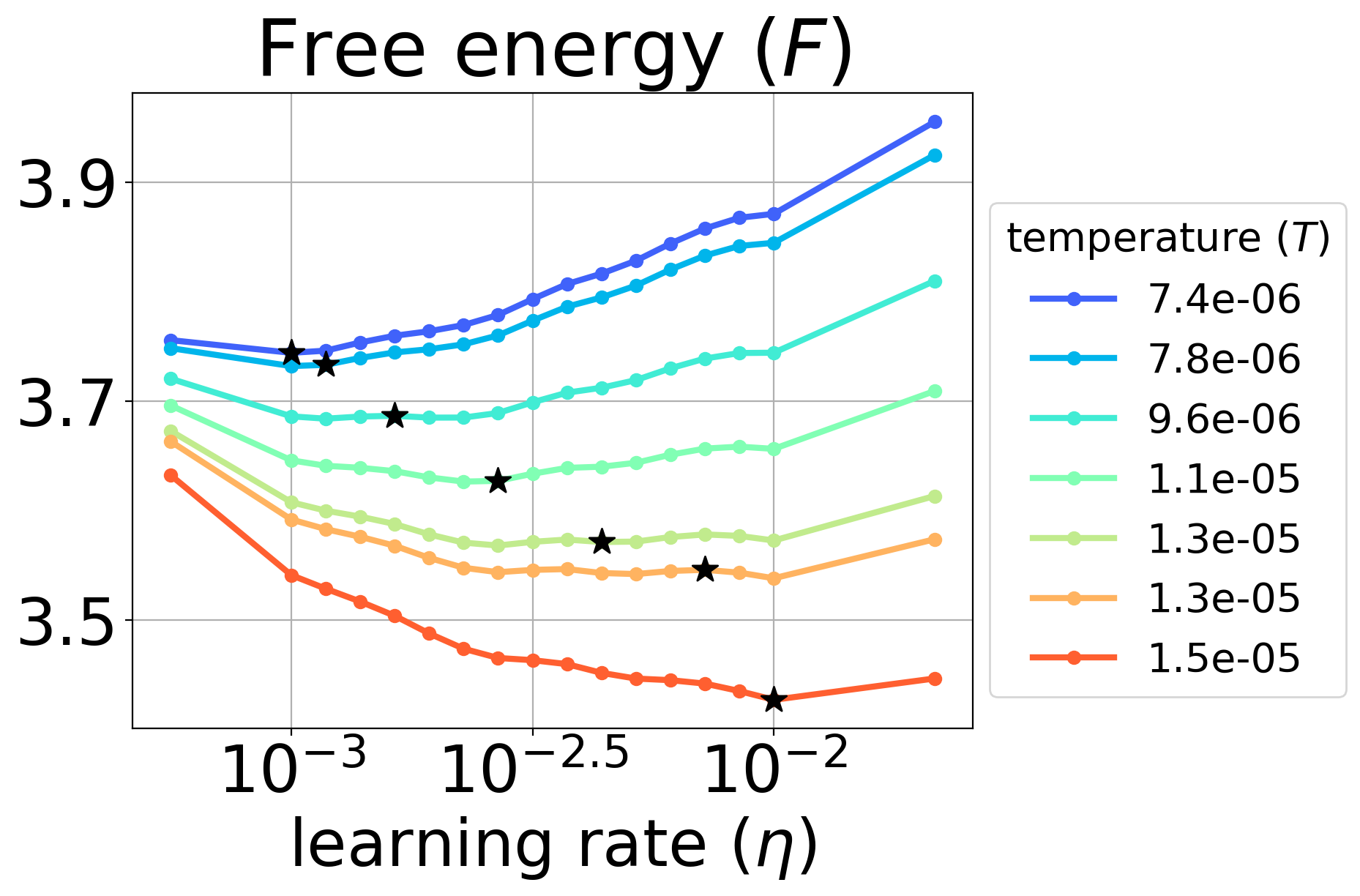} \\
        \multicolumn{3}{c}{OP ConvNet on CIFAR-100} \\
        \includegraphics[width=0.32\textwidth]{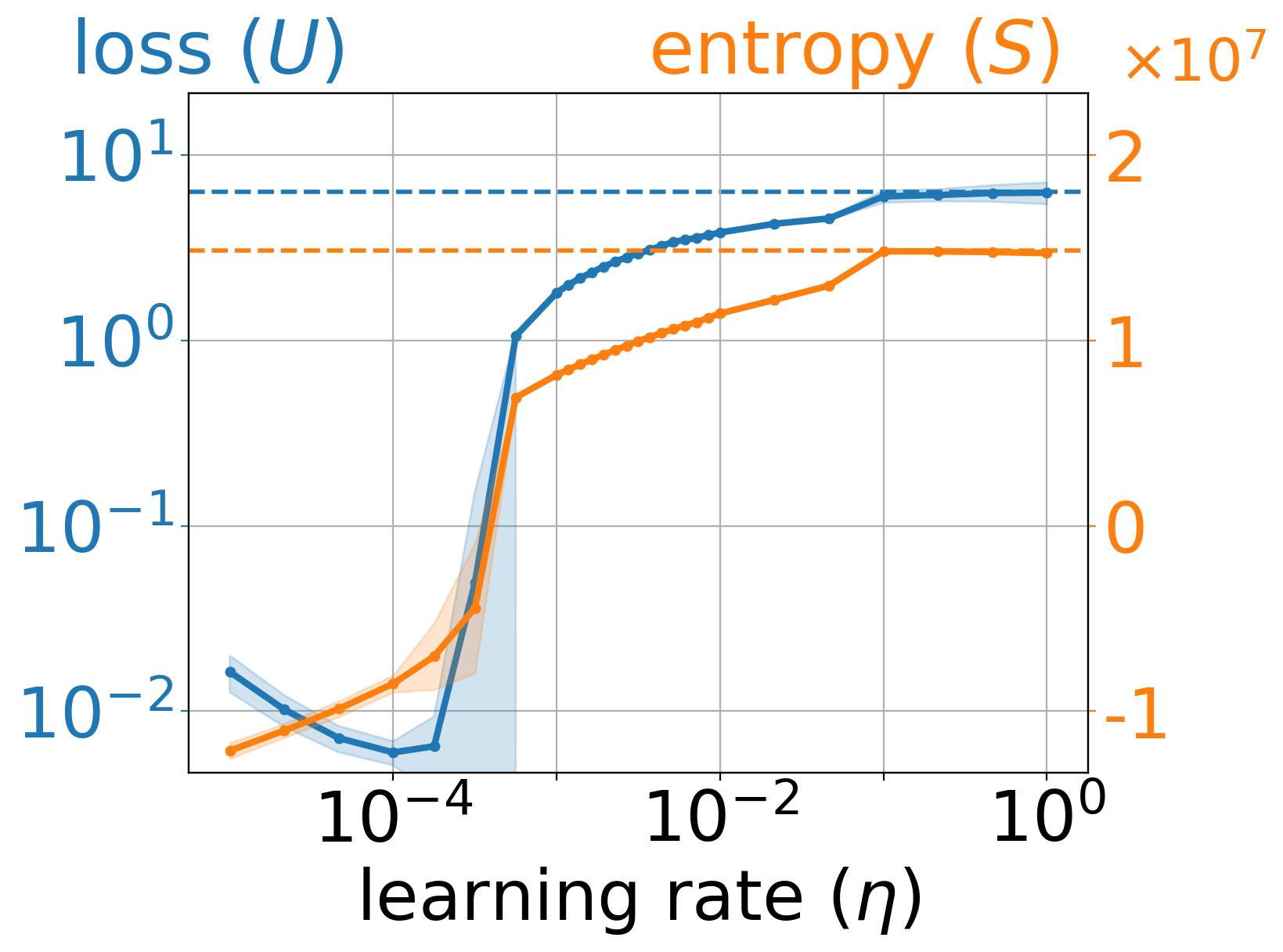} &
        \includegraphics[width=0.32\textwidth]{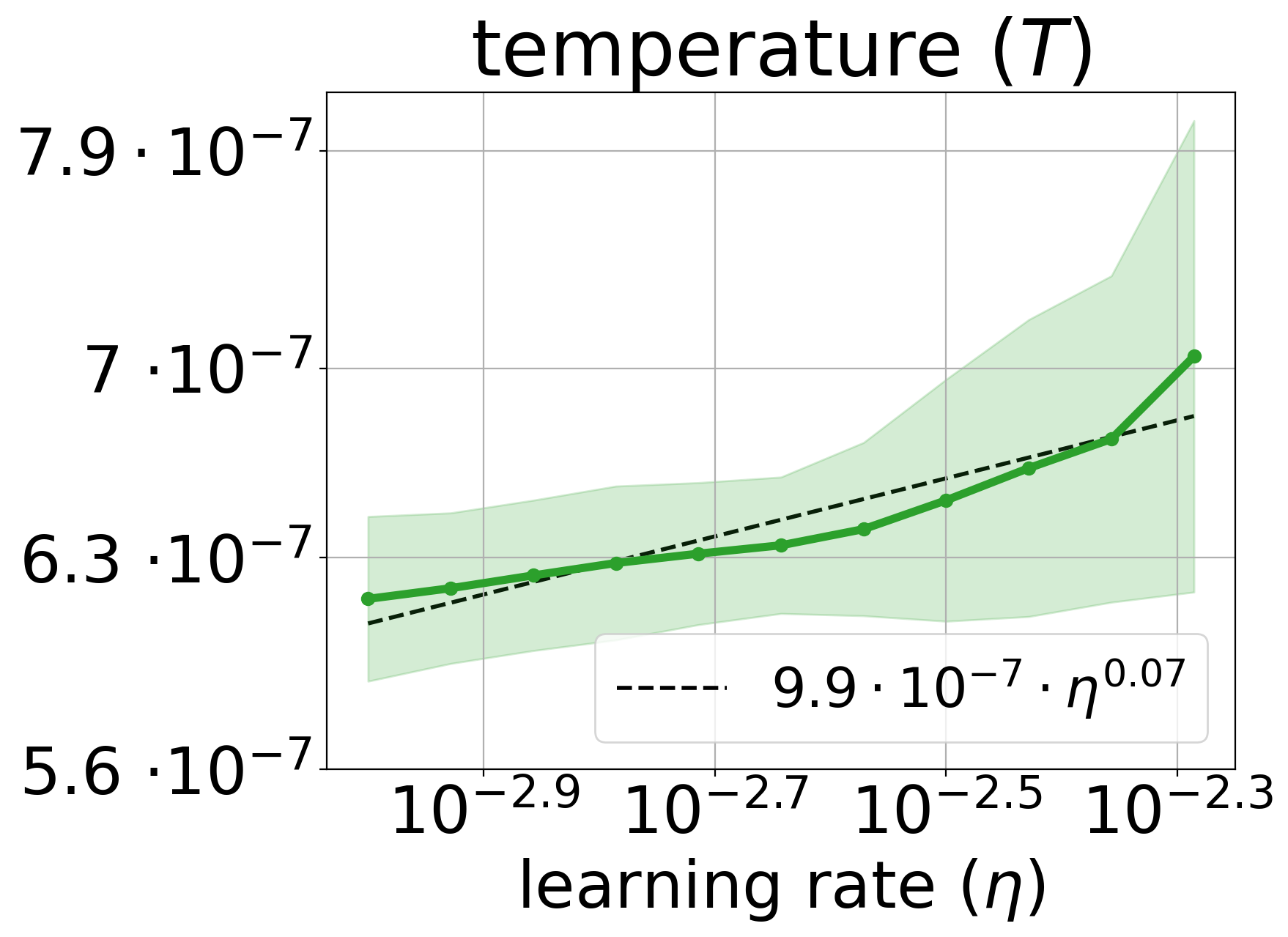} & 
        \includegraphics[width=0.36\textwidth]{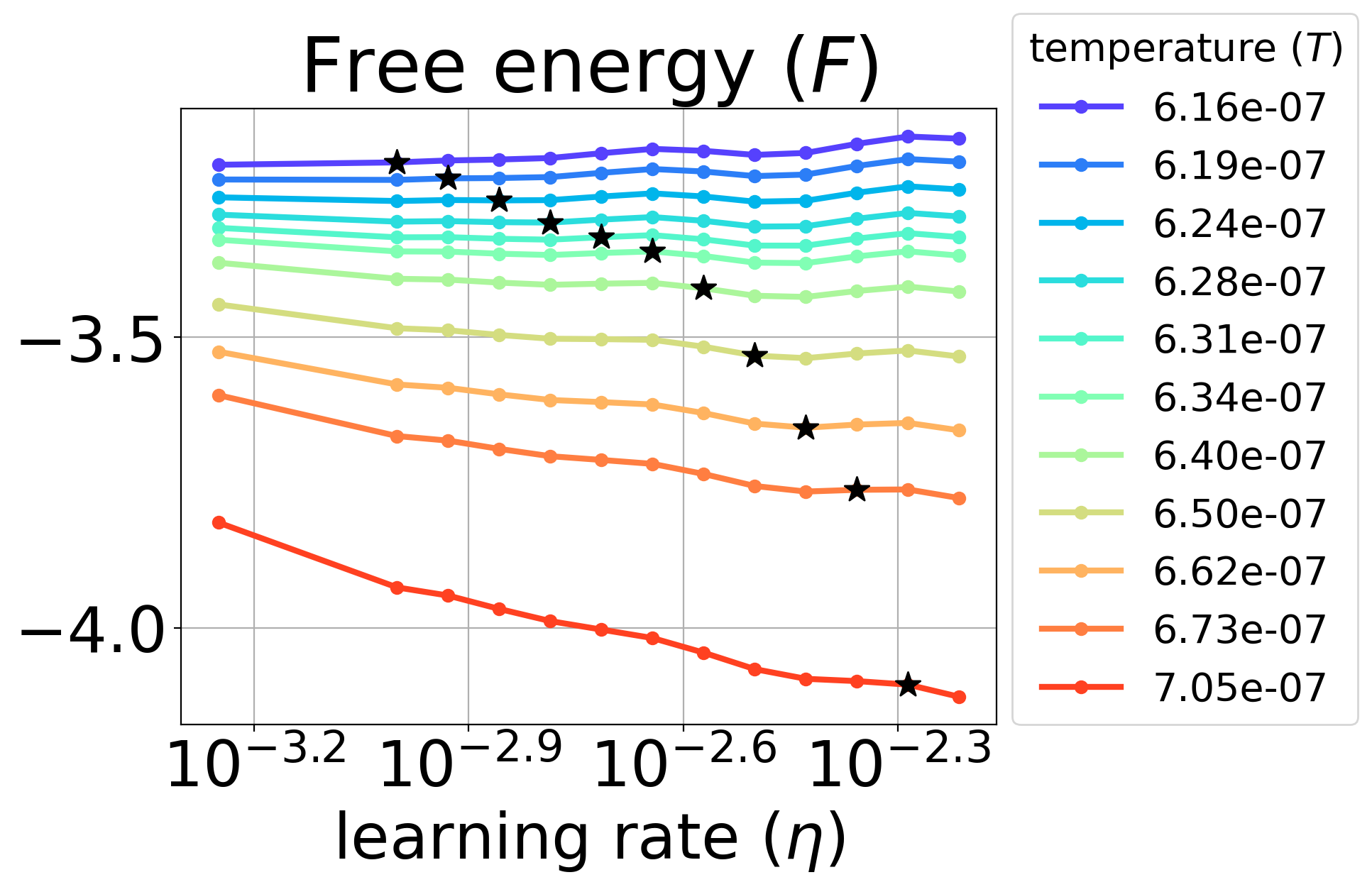} \\
    \end{tabular}
    \caption{Loss and entropy (left), estimated temperature (center) and free energy for different temperatures (right) vs. LR for ConvNet on CIFAR-100. The figure extends Figure~\ref{fig:temperature_up_op} from the main text.}
    \label{fig:app_temperature_up_op_convnet}
\end{figure}

\begin{figure}
    \centering
    \begin{tabular}{ccc}
    ResNet-18 on CIFAR-10 & ResNet-18 on CIFAR-100 & ConvNet on CIFAR-100 \\
    \includegraphics[width=0.31\textwidth]{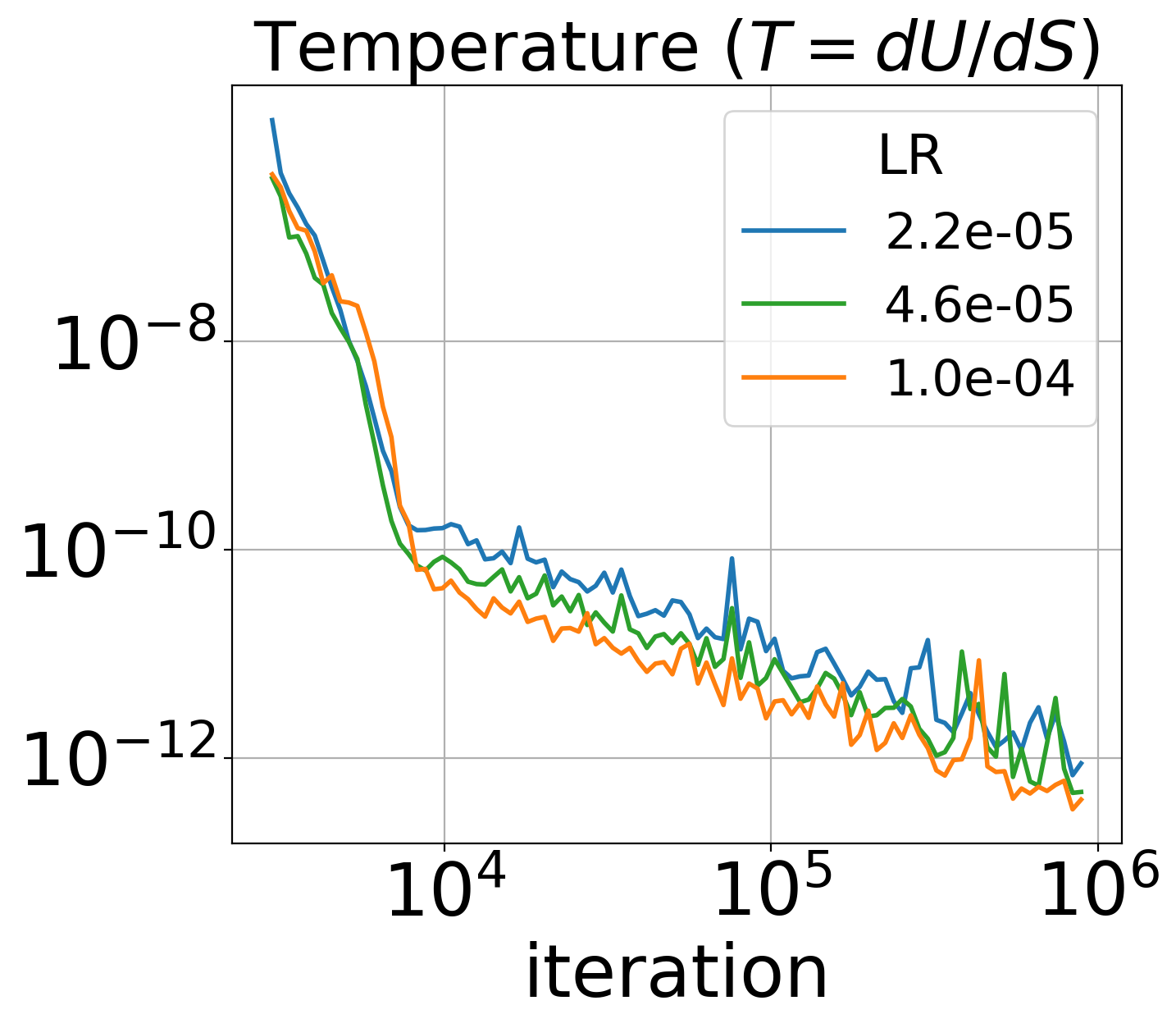} &
    \includegraphics[width=0.3\textwidth]{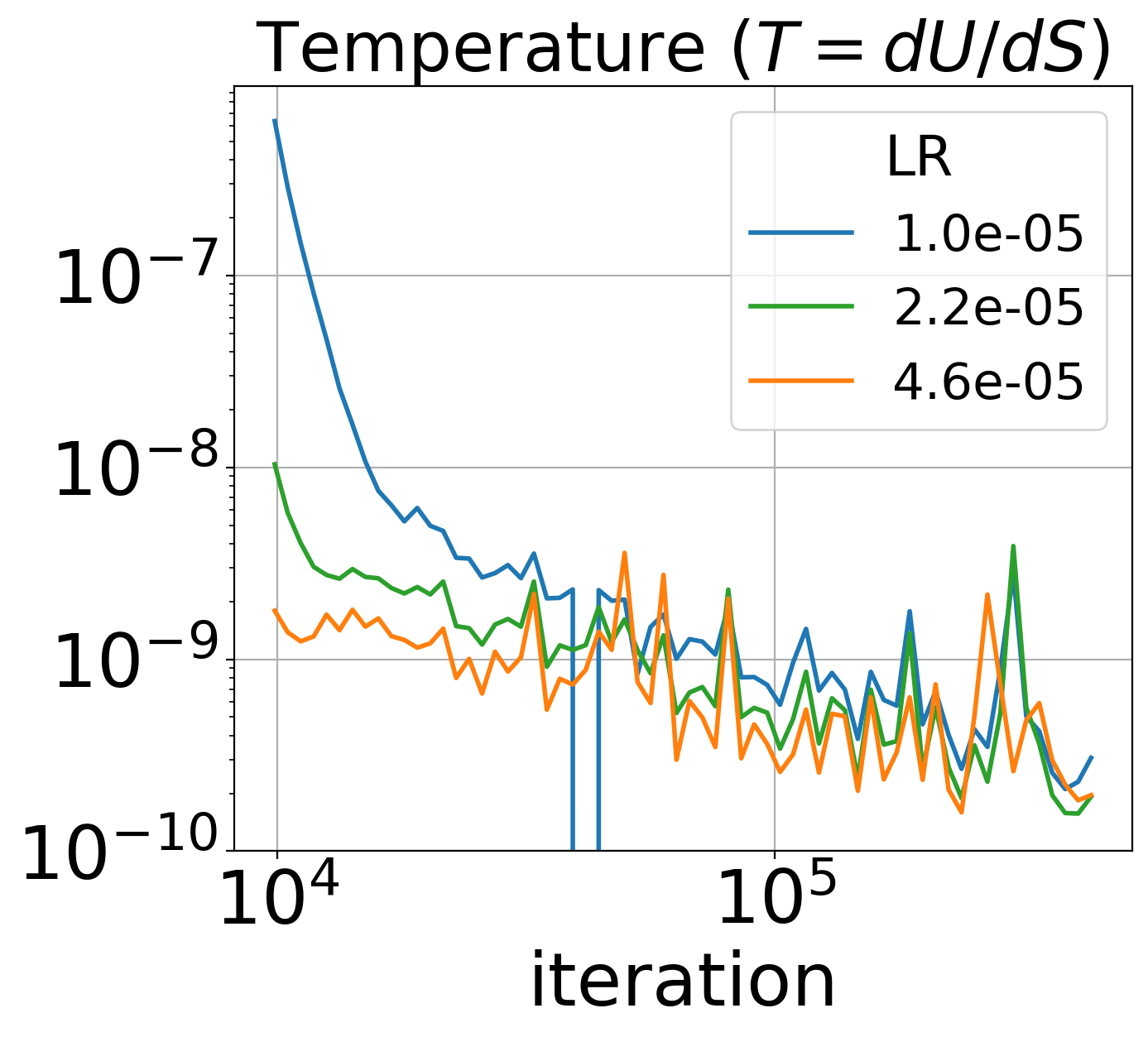} &
    \includegraphics[width=0.28\textwidth]{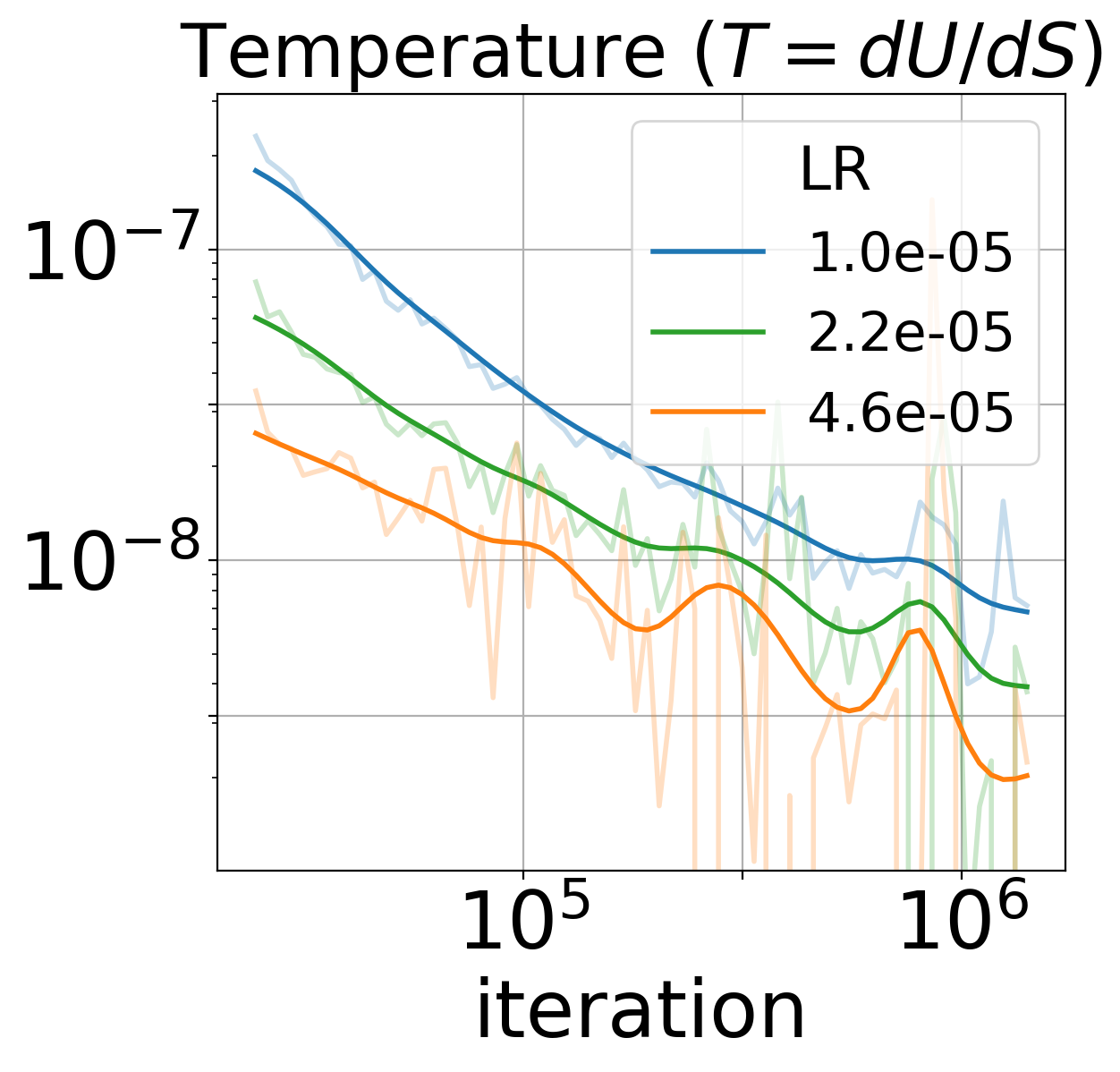}
    \end{tabular}
  \caption{Temperature decay for different architecture-dataset pairs for the OP setting. This figure complements Figure~\ref{fig:temperature_reg1} from the main text. For ConvNet on CIFAR-100, we calculate temperature with smoothed $U$ and $S$ (the smoothing protocol is described in Appendix~\ref{app:exp_setup}), as it otherwise behaves very noisily, frequently reaching negative values (translucent lines).}
  \label{fig:app_temperature_reg1}
\end{figure}

\FloatBarrier
\newpage
\section{Theoretical proofs for SNR}
\label{app:snr_proof}

In this section, we provide two theoretic proofs, which justify the behavior of the SNR near optima in the OP setting.
We use quadratic approximation of the loss $L(w)$ (with $w\in \mathbb{R}^D$) near the optimum point $w^*$:
\[
L(w) = L(w^*) + \frac{1}{2} (w - w^*)^T H (w-w^*),
\]
where $H=\nabla^2 L(w^*)$ represents the hessian matrix and the linear term is absent, as at the optimum we have $\nabla L(w^*) = 0$.
Our definition of overparameterization requires all stochastic gradients to be equal to zero together with the full gradient, so the loss $L_i(w)$ on object $i$ is:
\[
L_i(w) = L_i(w^*) + \frac{1}{2} (w - w^*)^T H_i (w-w^*),
\]
where $H_i=\nabla^2 L_i(w^*)$ is the hessian matrix of loss on this object. 

\newtheorem{theorem}{Theorem}
\newtheorem{lemma}{Lemma}
\begin{lemma}
\label{lemma:1}
Consider a unit vector $r\in \mathbb{R}^D$, $\|r\|=1$. Then, for the weights lying in the direction of $r$ from the optimum $w^*$, i.~e., $w=w^* + \delta r$ for some small $\delta > 0$ we have:
\begin{enumerate}[label={(\arabic*)}]
    \item $\cfrac{\|\nabla L_i(w)\|}{\|\nabla L(w)\|} = \cfrac{\|H_ir\|}{\|Hr\|}$
    \item $\text{SNR}(w) = \cfrac{\|Hr\|}{\sqrt{\mathbb{E} \|H_i r\|^2 - \|Hr\|^2}}$
\end{enumerate}
\end{lemma}

\begin{proof}
The proof is a straightforward application of the quadratic approximation. First, we write out the gradients:
\[
\nabla L(w) = H(w-w^*) = \delta \cdot H r \\
\]
\[
\nabla L_i(w) = H(w-w^*) = \delta \cdot H_i r, \\
\]
from which (1) immediately follows.
By definition of the SNR,
\begin{align*}
    SNR(w) &= \frac{\|\nabla L(w)\|}{\sqrt{\mathbb{E} \|\nabla L_i(w) - \nabla L(w)\|^2}} = \frac{\|\delta \cdot Hr\|}{\sqrt{\mathbb{E} \|\delta \cdot (H_ir - Hr)\|^2}} = \frac{\|Hr\|}{\sqrt{\mathbb{E} \|H_ir - Hr\|^2}} = \\
    &= \frac{\|Hr\|}{\sqrt{\mathbb{E} \|H_ir\|^2 - \|Hr\|^2}}
\end{align*}
\end{proof}

\begin{wrapfigure}{r}{0.3\textwidth}
    \vspace{-0.72cm}
    \begin{center}
        \includegraphics[width=0.27\textwidth]{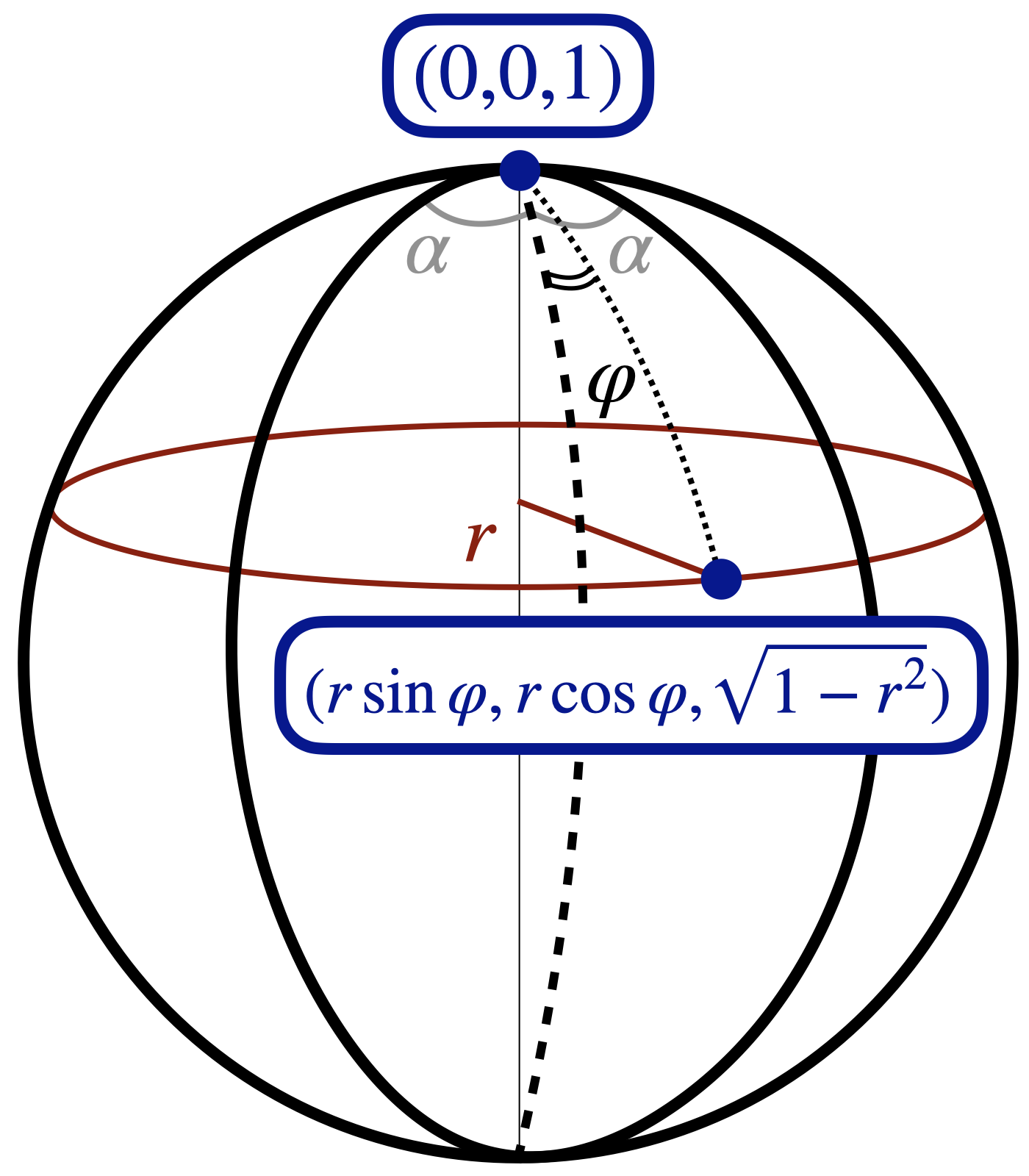}
    \end{center}
    \caption{A scheme of the 3D sphere OP setup, used in proof of Theorem~\ref{th:1}.}
    \label{fig:app_sphere_theory}
\end{wrapfigure}
This lemma implies that full and stochastic gradients decay at the same rate, as they remain
proportional, meaning that the exponent discussed in Section~\ref{sec:snr_nn} should be exactly 1.
Moreover, the asymptotic SNR depends on the direction of approach $r$ toward the optimum and generally varies across directions.
These limiting SNR values are strictly positive, except in degenerate cases where $\|Hr\| = 0$, i.e., when $r$ is a zero-eigenvalue eigenvector of the Hessian.
In our quadratic approximation, this would imply $L(w^* + \delta r) = L(w^*)$, meaning SGD would not progress along $r$, as stochastic gradients vanish.
Thus, meaningful descent occurs only in directions with $\|Hr\| > 0$, ensuring a strictly positive SNR.

Now, we calculate the SNR for our OP 3D sphere toy example (see Section~\ref{sec:toy} and Appendix~\ref{app:sphere_result} for the description of the setup).
Contrary to Lemma~\ref{lemma:1}, here we do not rely on the quadratic approximation near the optimum and instead derive exact values of SNR over the sphere.

\begin{theorem}
\label{th:1}
Consider a 3D unit sphere $x^2 + y^2 + z^2 = 1$ and define two loss functions $f_1$ and $f_2$:
\[
f_1(x, y, z) = \frac{1}{2} \cdot \frac{(x\cos \alpha + y\sin \alpha)^2}{x^2+y^2+z^2}\quad\quad\quad
f_2(x, y, z) = \frac{1}{2} \cdot \frac{(x\cos \alpha - y\sin \alpha)^2}{x^2+y^2+z^2},
\]
where $0 < \alpha < \pi/4$ is a fixed angle. For the full loss $f=(f_1+f_2)/2$, we have:
\begin{enumerate}[label={(\arabic*)}]
    \item $\text{SNR}^2(x,y,z) = \cfrac{x^2\cos^4 \alpha + y^2 \sin^4 \alpha - (x^2 \cos^2 \alpha + y^2 \sin^2 \alpha)^2}{\sin^2 \alpha \cos^2 \alpha (x^2 + y^2 - 4x^2y^2)}$
    \item if we have $x=r\sin\varphi$ and $y=r\cos\varphi$ with $0 < r \le 1$ and $-\alpha \le \varphi \le \alpha$ (see Figure~\ref{fig:app_sphere_theory} for an illustration):
    \begin{enumerate}
        \item for a fixed $r$, the minimum of SNR is achieved at $\varphi = 0$ (i.~e., at the central meridian)
        \item for a fixed $\varphi$ and $r\to0_+$, SNR approaches a constant value from below
    \end{enumerate}
\end{enumerate}
\end{theorem}

\begin{proof}
(1). The proof is mostly technical and requires straightforward derivation of the gradients and SNR.
First, denote $v=(x, y, z)^T$, $A_1 = x\cos\alpha + y\sin\alpha$, and $A_2 = x\cos\alpha - y\sin\alpha$. The functions $f_1$ and $f_2$ can be shortened to:
\[
f_1 = \frac{A_1^2}{2\|v\|^2} \quad\quad\quad f_2 = \frac{A_2^2}{2\|v\|^2} 
\]
Now, we derive their gradients:
\begin{align*}
    g_1 &:= \nabla f_1 = \frac{A_1 \|v\|^2 \cdot \nabla A_1 - A_1^2 \cdot v}{\|v\|^4} = \{\|v\| = 1\} = A_1 \cdot \nabla A_1 - A_1^2 \cdot v \\
    g_2 &:= \nabla f_2 = \frac{A_2 \|v\|^2 \cdot \nabla A_2 - A_2^2 \cdot v}{\|v\|^4} = \{\|v\| = 1\} = A_2 \cdot \nabla A_2 - A_2^2 \cdot v,
\end{align*}
where $\nabla A_1 = (\cos \alpha, \sin \alpha, 0)^T$ and $\nabla A_2 = (\cos \alpha, -\sin \alpha, 0)^T$. We will also require several auxiliary quantities:
\begin{gather*}
    A_1 + A_2 = 2x\cos \alpha \quad\quad\quad A_1 - A_2 = 2y\sin \alpha \\
    A_1^2 - A_2^2 = (A_1 - A_2)(A_1 + A_2) = 4xy \sin\alpha\cos\alpha \\
    A_1^2 + A_2^2 = \frac{1}{2} \left((A_1 + A_2)^2 + (A_1 - A_2)^2\right) = 2x^2\cos^2\alpha + 2y^2\sin^2\alpha =: 2B
\end{gather*}
Next, we the express the SNR via $g_1$ and $g_2$:
\begin{align*}
    \text{SNR} &= \frac{\|(g_1 + g_2)/2\|}{\sqrt{0.5 \cdot \|g_1 - (g_1 + g_2)/2\|^2 + 0.5 \cdot \|g_2 - (g_1 + g_2)/2\|^2}} = \\
    &= \frac{\|(g_1 + g_2)/2\|}{\sqrt{0.5 \cdot \|(g_1 - g_2)/2\|^2 + 0.5 \cdot \|(g_2 - g_1)/2\|^2}} = \frac{\|(g_1 + g_2)/2\|}{\sqrt{\|(g_1 - g_2)/2\|^2}} = \frac{\|g_1 + g_2\|}{\|g_1 - g_2\|}
\end{align*}
The sum of gradients is:
\begin{align*}
    g_1 + g_2 &= A_1 \nabla A_1 + A_2 \nabla A_2 - (A_1^2 + A_2^2) v = \\ 
    &=
    \begin{pmatrix}
    x\cos^2\alpha + y\sin\alpha\cos\alpha \\
    x\cos\sin\alpha + y\sin^2 \alpha \\
    0
    \end{pmatrix} + 
    \begin{pmatrix}
    x\cos^2\alpha - y\sin\alpha\cos\alpha \\
    -x\cos\sin\alpha + y\sin^2 \alpha \\
    0
    \end{pmatrix} - 2Bv = \\
    &= \underbrace{2 \begin{pmatrix}
    x\cos^2\alpha \\
    y\sin^2 \alpha \\
    0
    \end{pmatrix}}_{=:2u} - 2Bv = 2(u - Bv)
\end{align*}
The difference of gradients is:
\begin{align*}
    g_1 - g_2 &= A_1 \nabla A_1 - A_2 \nabla A_2 - (A_1^2 - A_2^2) v = \\ 
    &=
    \begin{pmatrix}
    x\cos^2\alpha + y\sin\alpha\cos\alpha \\
    x\sin\alpha\cos\alpha + y\sin^2 \alpha \\
    0
    \end{pmatrix} - 
    \begin{pmatrix}
    x\cos^2\alpha - y\sin\alpha\cos\alpha \\
    -x\sin\alpha\cos\alpha + y\sin^2 \alpha \\
    0
    \end{pmatrix} - 4xy\sin\alpha\cos\alpha \cdot v = \\
    &= \underbrace{2 \begin{pmatrix}
    y\sin\alpha\cos\alpha \\
    x\sin\alpha\cos\alpha \\
    0
    \end{pmatrix}}_{=: 2\sin\alpha\cos\alpha \cdot w} - 4xy\sin\alpha\cos\alpha \cdot v = 2 \sin\alpha\cos\alpha (w - 2xy \cdot v)
\end{align*}
Now, we derive the squared norms:
\begin{align*}
    \|g_1 + g_2\|^2 &= 4\Big(\|u\|^2 + B^2 \|v\|^2 - 2B\langle u, v\rangle\Big) = 4\Big(x^2\cos^4\alpha + y^2\sin^4 \alpha + B^2 - 2B \cdot B \Big) = \\
    &= 4 \Big(x^2\cos^4\alpha + y^2\sin^4 \alpha - B^2 \Big) \\\\
    \|g_1-g_2\|^2 &=  4 \sin^2 \alpha \cos^2 \alpha \Big(\|w\|^2 + 4x^2y^2 \|v\|^2 - 4xy \langle w, v\rangle\Big) = \\
    &= 4 \sin^2 \alpha \cos^2 \alpha \Big(x^2 + y^2 + 4x^2y^2 - 4xy \cdot 2xy\Big) = \\
    &= 4 \sin^2 \alpha \cos^2 \alpha \Big(x^2 + y^2 - 4x^2y^2\Big)
\end{align*}
Dividing the first by the second and substituting $B$, we get:
\[
\text{SNR}^2(x,y,z) = \frac{\|g_1 + g_2\|^2}{\|g_1 - g_2\|^2} = \frac{x^2\cos^4\alpha + y^2\sin^4 \alpha - (x^2\cos^2\alpha + y^2\sin^2\alpha)^2}{\sin^2\alpha\cos^2\alpha (x^2 + y^2 - 4x^2y^2)}
\]
The gradients of scale-invariant function are orthogonal to the radial direction, so the ratio between their norms can be expressed via $x$ and $y$ and thus SNR does not depend on $z$.

(2). The next step is to plug in the polar parameterization $x=r\sin\varphi, y=r\cos\varphi$ to analyze the behavior of SNR w.r.t $r$ and $\varphi$:
\begin{align*}
    SNR^2(r, \varphi) &= \frac{r^2\sin^2\varphi \cos^4 \alpha +r^2\cos^2\varphi\sin^4\alpha - (r^2 \sin^2\varphi\cos^2\alpha + r^2\cos^2\varphi\sin^2\alpha)^2}{\sin^2\alpha\cos^2\alpha (r^2-4r^4\sin^2\varphi\cos^2\varphi)} \\
    &= \frac{\sin^2\varphi \cos^4 \alpha + \cos^2\varphi\sin^4\alpha - r^2 (\sin^2\varphi\cos^2\alpha + \cos^2\varphi\sin^2\alpha)^2}{\sin^2\alpha\cos^2\alpha (1-4r^2\sin^2\varphi\cos^2\varphi)} 
\end{align*}

(2a). We want to find the minimum of SNR over $\varphi$ for a fixed $r$.
As we have $0 \le \varphi \le \alpha < \pi/4$, the maximum of the denominator is achieved at $\varphi = 0$ and is equal to $\sin^2\alpha \cos^2\alpha$.
Let us show that $\varphi=0$ also minimizes the nominator. 
Define $S:=\sin^2\varphi$ so that $0\le S \le \sin^2\alpha$ and $1-S=\cos^2\varphi$:
\begin{align*}
    &S\cos^4\alpha + (1-S)\sin^4\alpha - r^2(S\cos^2\alpha + (1-S)\sin^2\alpha)^2 = \\
    &= \sin^4\alpha + S(\cos^4\alpha - \sin^4\alpha) - r^2(\sin^2\alpha + S\big(\cos^2\alpha - \sin^2\alpha)\big)^2 = \\
    &= \big\{\cos^4\alpha - \sin^4\alpha = \cos^2\alpha - \sin^2\alpha =: \Delta \big\} =
    \sin^4\alpha + S\Delta - r^2(\sin^2\alpha + S\Delta\big)^2 = \\
    &= (1-r^2)\sin^4\alpha + S(1-2r^2\sin^2\alpha)\Delta - S^2r^2\Delta^2
\end{align*}
This is a parabola with the negative leading coefficient.
Thus, its minimum is achieved at one of the two boundaries.
We show that the maximum point $S_m > \sin^2\alpha$, so the minimum is reached at $S=0$ (and thus $\varphi=0$):
\begin{align*}
    S_m = \frac{(1-2r^2\sin^2\alpha)\Delta}{2r^2\Delta^2} = \frac{1}{\Delta} \left(\frac{1}{2r^2} - \sin^2\alpha\right) > \sin^2\alpha 
\end{align*}
We take the $\inf_{0\le r \le 1} \{S_m\} = (1/2-\sin^2\alpha)/\Delta$ and substitute $\Delta = \cos^2\alpha-\sin^2\alpha = 1-2\sin^2\alpha$:
\begin{gather*}
    \frac{1}{2} - \sin^2\alpha > \Delta \cdot \sin^2\alpha = \sin^2\alpha - 2\sin^4\alpha \\
    2\sin^4\alpha - 2 \sin^2\alpha + \frac{1}{2} = 2\left(\sin^2\alpha - \frac{1}{2}\right)^2 > 0
\end{gather*}
Therefore, both the nominator achieves its minimum and the denominator achieves its maximum at $\varphi=0$, giving us the minimum of the squared SNR, which is equal to:
\[
\text{SNR}^2(r, \varphi=0) = (1-r^2)\tan^2\alpha
\]
Note that this value is strictly positive for $r < 1$.

(2b). Let us represent the squared SNR as a function of $r^2$. Define:
\begin{align*}
    M &= \sin^2\varphi \cos^4 \alpha + \cos^2\varphi\sin^4\alpha \\
    P &= (\sin^2\varphi\cos^2\alpha + \cos^2\varphi\sin^2\alpha)^2 \\
    Q &= 4 \sin^2\varphi\cos^2\varphi 
\end{align*}
Omitting positive constants in the denominator, the squared SNR is:
\[
SNR^2(r^2) \propto \frac{M-Pr^2}{1 - Qr^2}
\]
We need to show that this function decreases monotonically for $0 < r^2 < 1$, as this would mean that for $r\to 0_+$ there is ascending convergence to the limit value.
Let us calculate the derivative over $r^2$:
\[
\frac{d(\text{SNR}^2)}{d(r^2)} = \frac{-P(1 - Qr^2)+Q(M-Pr^2)}{(1-Qr^2)^2} = \frac{MQ-P}{(1-Qr^2)^2}
\]
Thus, it is sufficient to show that $MQ-P \le 0$ for $0 \le \varphi \le \alpha < \pi/4 $.
Denote $S=\sin^2\varphi, R=\sin^2\alpha$ with $0 \le S \le R < 1/2$. The quantities $M, P, Q$ could be expressed as:
\begin{align*}
    M &= S(1-R)^2 + (1-S)R^2 \\
    P &= \big(S(1-R) + (1-S)R\big)^2 \\
    Q &= 4S(1-S)
\end{align*}
We use the \texttt{sympy}\footnote{https://docs.sympy.org/latest/index.html} Python library designed for symbolic calculations to expand the brackets in the expression $MP-Q$ and then factor the result.
The notebook \texttt{code/toy/theorem1.ipynb} for reproducing this step is included into supplementary materials.
The resulting factorization is:
\begin{align*}
    MP-Q &= (S-R)(8RS^2-8RS+R-4S^2+3S)
\end{align*}
As $S\le R$, the first bracket is non-positive. It is only left to check that the second bracket is non-negative. For a fixed $S$, the second bracket is a linear function w.r.t $R$, so we can check that the values at the boundaries $R=0$ and $R=1/2$ are non-negative:
\begin{align*}
    R = 0 &\Rightarrow -4S^2+3S = 4S\underbrace{(1/2 - S)}_{\ge 0} + \underbrace{S}_{\ge 0} \ge 0 \\
    R = 1/2 &\Rightarrow 4S^2 - 4S + 1/2 - 4S^2 + 3S = 1/2 - S \ge 0,
\end{align*}
which proves the claim (2b).
\end{proof}

\section{Additional results for spherical toy example}
\label{app:sphere_result}

\begin{figure}
    \centering
    \newcommand{\specialcell}[2][c]{%
    \begin{tabular}[#1]{@{}c@{}}#2\end{tabular}}
    \renewcommand{\arraystretch}{1.2}

    \begin{tabular}{cccc}
        loss & SNR & \specialcell{full \\ grad. norm} & \specialcell{mean stoch.\\grad norm} \\
        \hline\multicolumn{4}{c}{Underparameterized (UP)}\\ \hline
        \includegraphics[width=0.22\textwidth]{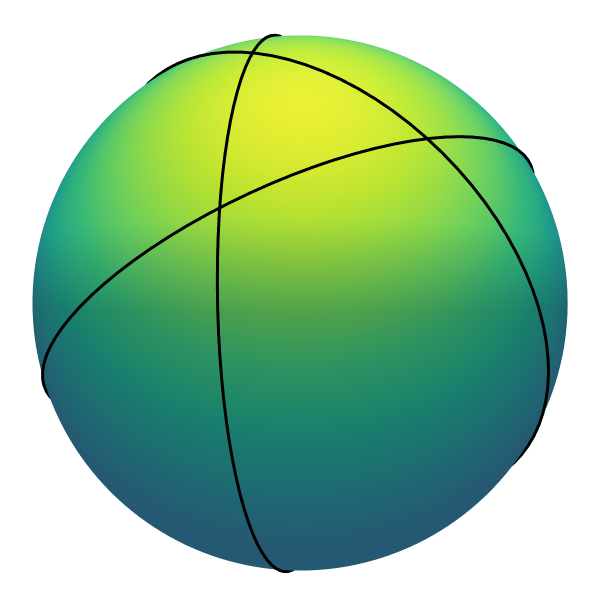} &
        \includegraphics[width=0.22\textwidth]{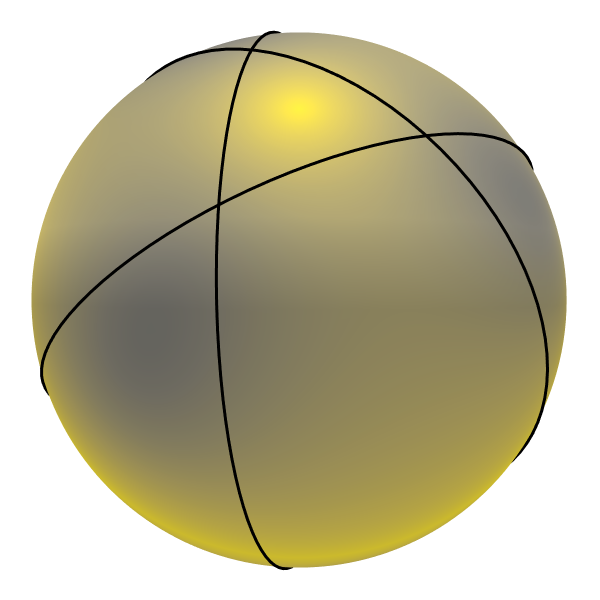} &
        \includegraphics[width=0.22\textwidth]{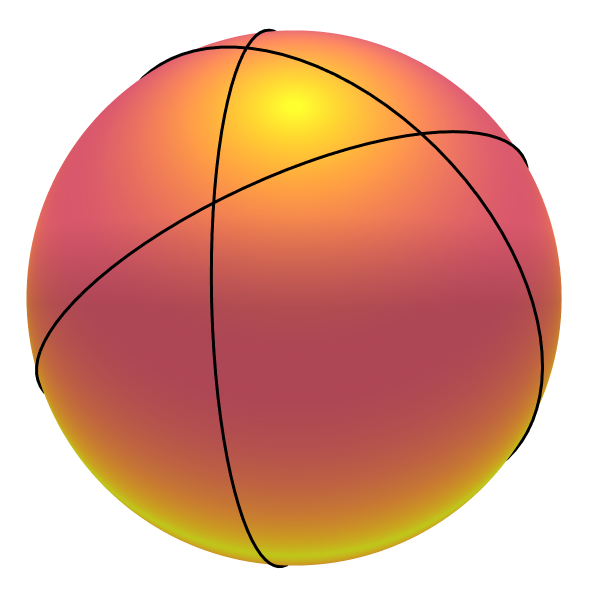} &
        \includegraphics[width=0.22\textwidth]{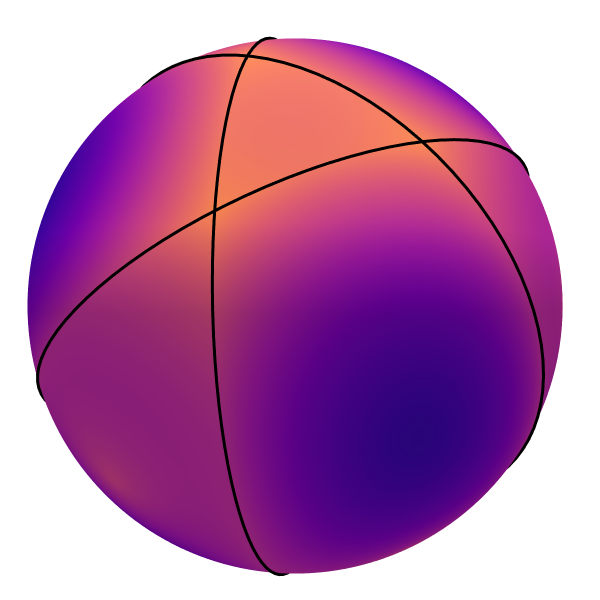} \\
        \hline\multicolumn{4}{c}{Overparameterized (OP)}\\ \hline
        \includegraphics[width=0.22\textwidth]{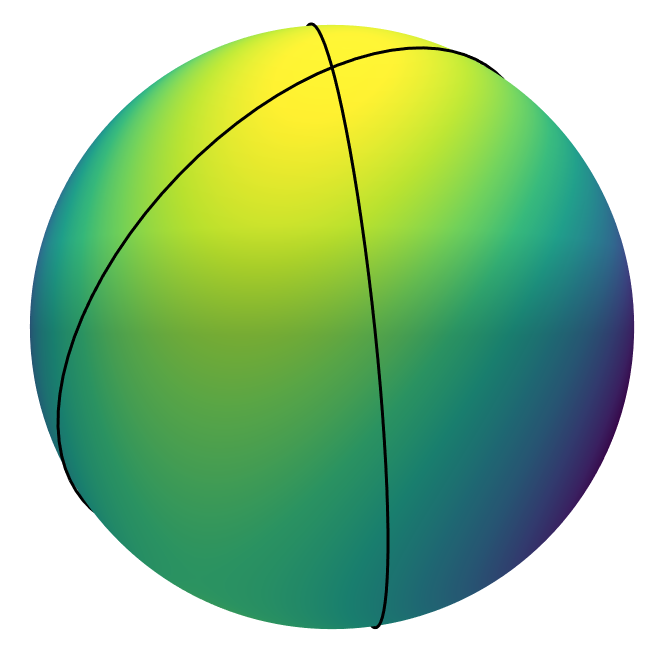} &
        \includegraphics[width=0.22\textwidth]{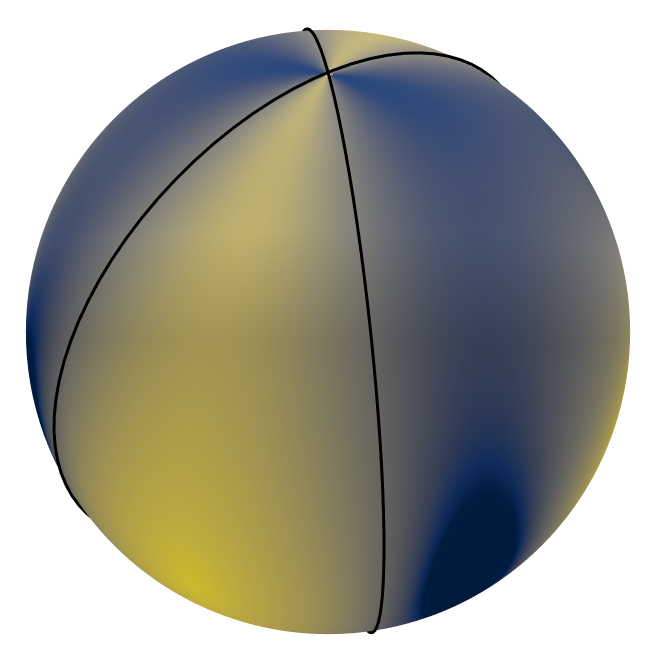} &
        \includegraphics[width=0.22\textwidth]{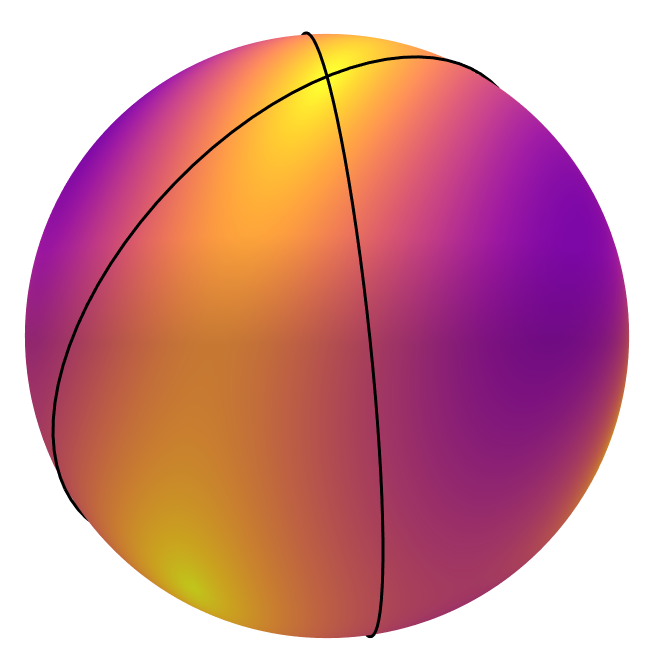} &
        \includegraphics[width=0.22\textwidth]{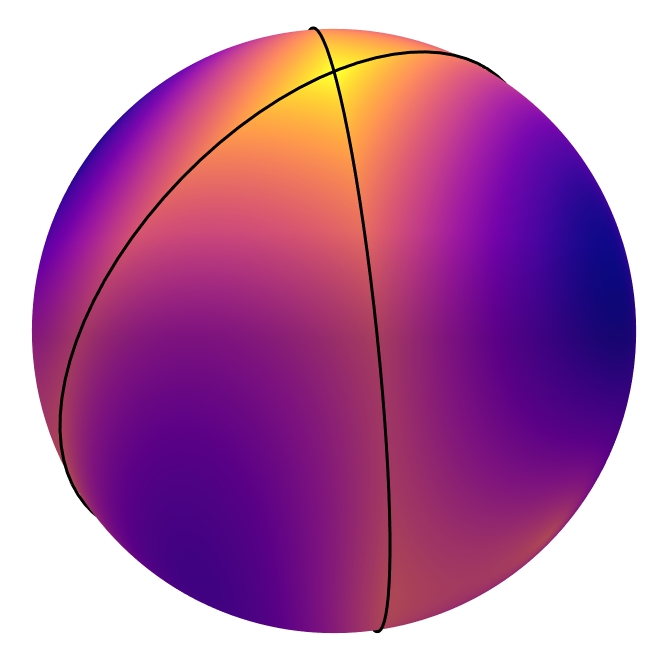} \\
        \hline
        \includegraphics[width=0.22\textwidth]{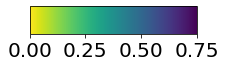} & 
        \includegraphics[width=0.22\textwidth]{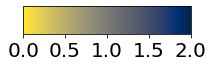} & 
        \multicolumn{2}{c}{\includegraphics[width=0.38\textwidth]{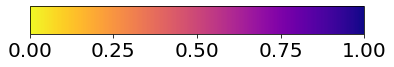}} \\ 
    \end{tabular}
    \caption{Loss, SNR, full gradient norm and mean stochastic gradient norm of UP and OP setups for 3D toy example.}
    \label{fig:app_sphere_metrics}
\end{figure}

\begin{figure}
    \centering
    \newcommand{\specialcell}[2][c]{%
    \begin{tabular}[#1]{@{}c@{}}#2\end{tabular}}
    \renewcommand{\arraystretch}{1.2}

    \begin{tabular}{cccc}
        \toprule\multicolumn{4}{c}{Underparameterized (UP)}\\ \midrule
        \includegraphics[width=0.22\textwidth]{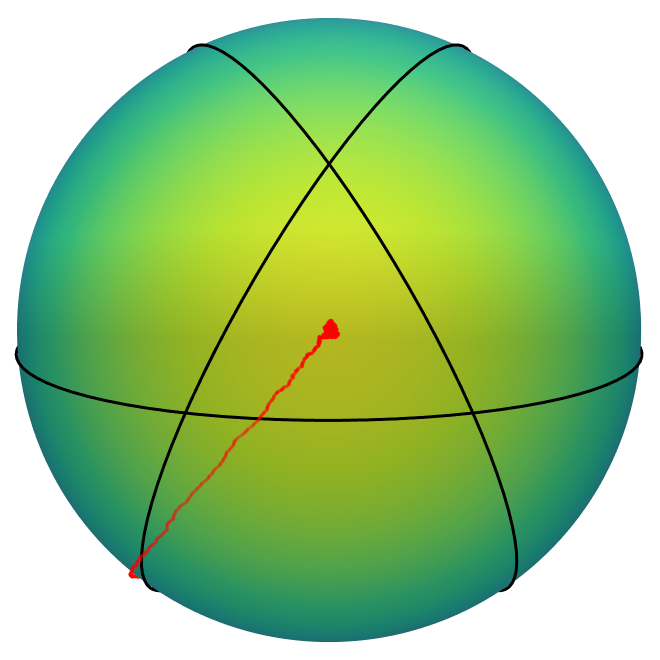} &
        \includegraphics[width=0.22\textwidth]{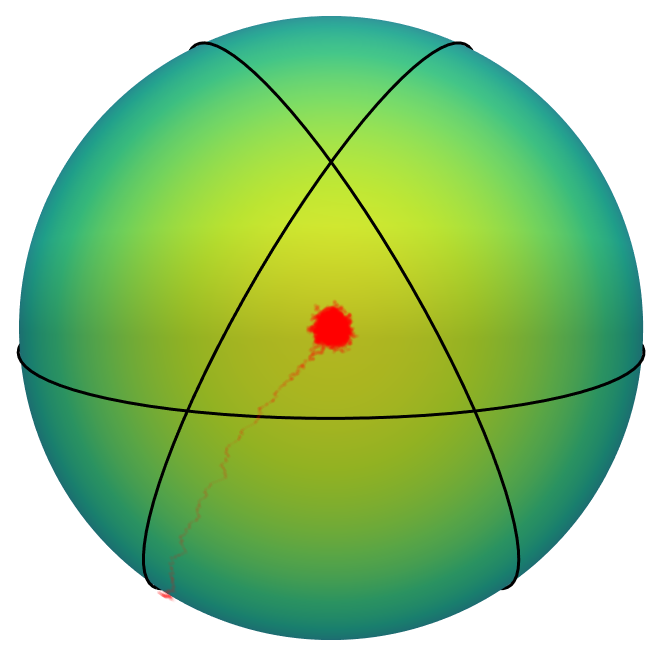} &
        \includegraphics[width=0.22\textwidth]{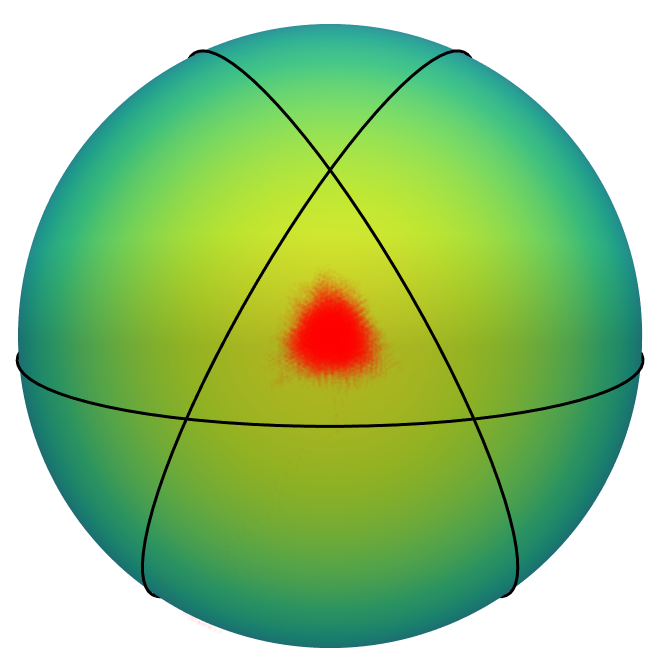} &
        \includegraphics[width=0.22\textwidth]{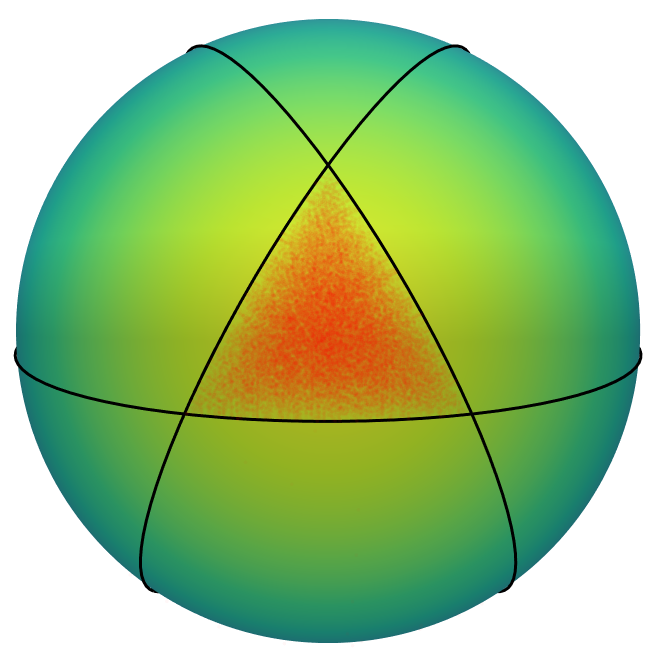} \\
        $\text{LR} = 10^{-3}$ & $\text{LR} = 6\cdot10^{-3}$ & $\text{LR} = 3.5\cdot10^{-2}$ & $\text{LR} = 2.1\cdot10^{-1}$ \\ 
        \midrule \multicolumn{4}{c}{Overparameterized (OP)} \\ \midrule \\
        \includegraphics[width=0.22\textwidth]{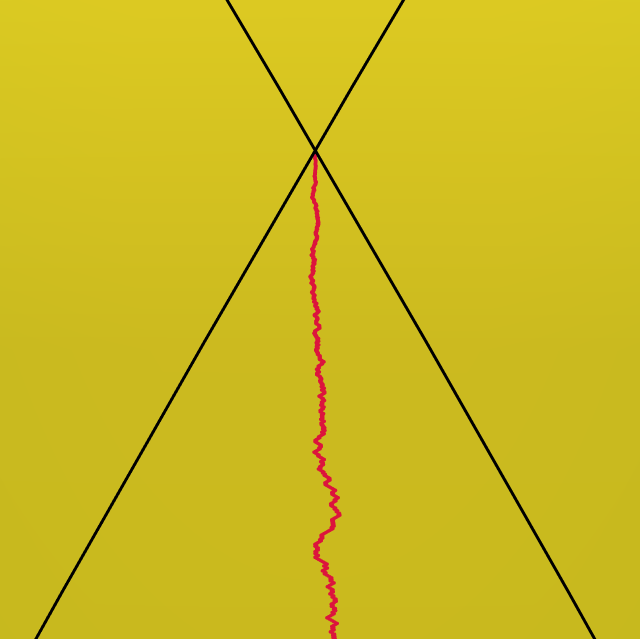} &
        \includegraphics[width=0.22\textwidth]{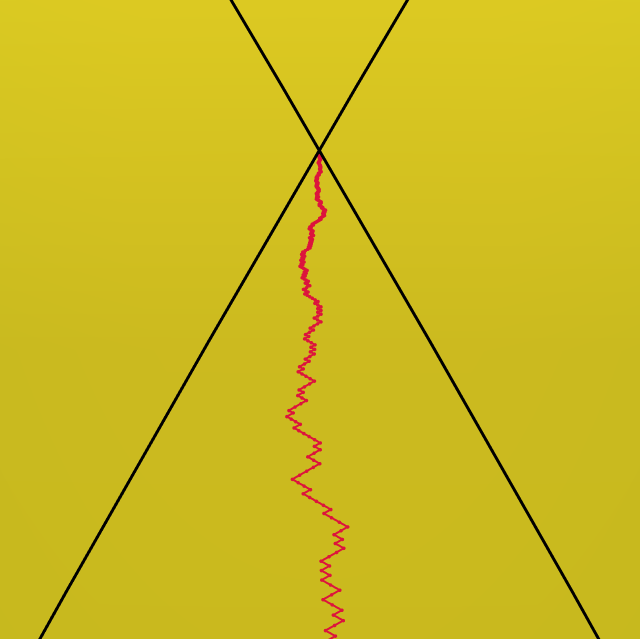} &
        \includegraphics[width=0.22\textwidth]{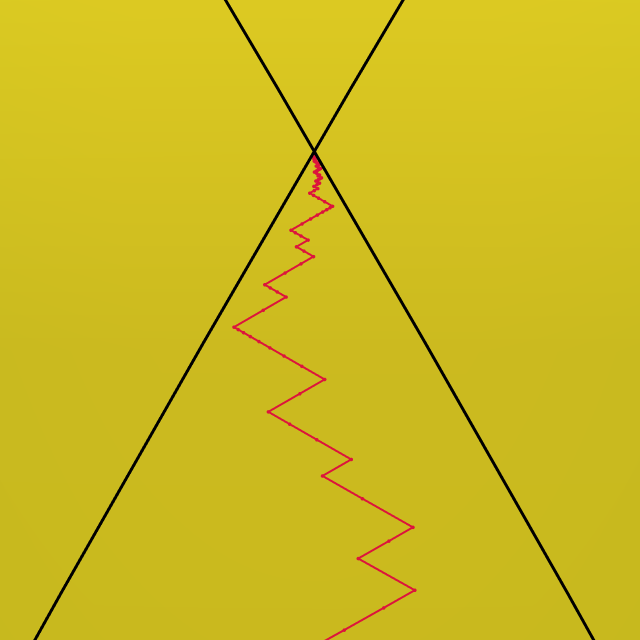} &
        \includegraphics[width=0.22\textwidth]{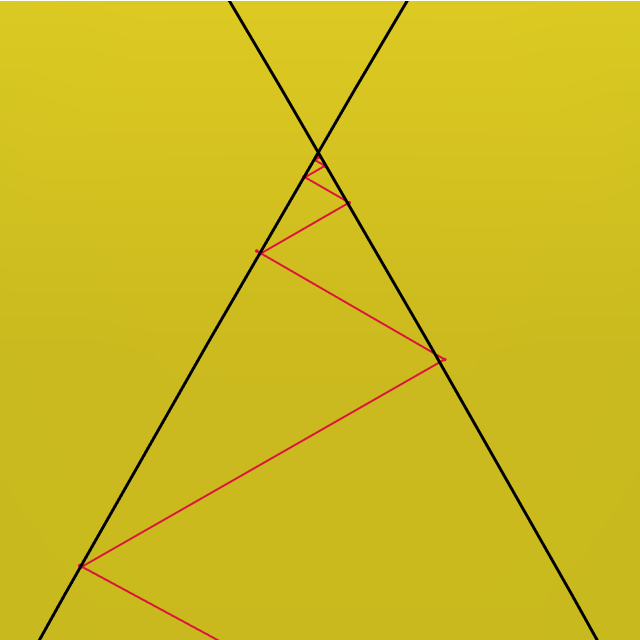} \\
        $\text{LR} = 4.8\cdot10^{-3}$ & $\text{LR} = 2.3\cdot10^{-2}$ & $\text{LR} = 1.1\cdot10^{-1}$ & $\text{LR} = 5.1\cdot10^{-1}$ \\ 
        \midrule
        \multicolumn{2}{c}{\includegraphics[width=0.38\textwidth]{figures/sphere_legend/sphere_colorbar.png}} & 
        \multicolumn{2}{c}{\includegraphics[width=0.38\textwidth]{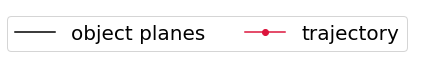}} \\ 
        \bottomrule
    \end{tabular}
    \caption{Stationary distribution for UP and convergence to optimum for OP for different LRs. For the OP setup in the bottom row, we zoom in closer to the minimum to compare the convergence dynamics.}
    \label{fig:app_sphere_stationary}
\end{figure}

 \paragraph{Experimental setup} 
For the OP configuration, we use two normal vectors: $(\sqrt{3}/2, 1/2, 0)$ and $(\sqrt{3}/2, -1/2, 0)$.
This setup yields a minimum at the point $(0, 0, 1)$ on the unit sphere.
For the UP configuration, we also place the minimum at $(0, 0, 1)$, and use three great circles that (1) are equidistant from the minimum, and (2) have equal pairwise angles between their normal vectors.
These conditions result in the normal vectors $(1, 0, 0.2)$, $(-1/2, \sqrt{3}/2, 0.2)$, and $(-1/2, -\sqrt{3}/2, 0.2)$.
Their $z$-coordinate controls the distance to the minimum.
These vectors are then normalized to unit length.
We train both setups using projected SGD for $50$K iterations, or until the loss in the OP setup falls below $10^{-16}$.
At each iteration, one object is sampled independently for stochastic gradient estimation.

\paragraph{Results} 
In Figure~\ref{fig:app_sphere_metrics}, we show the loss, SNR, and gradient norms over the sphere for both configurations.
In the UP setup, the SNR is exactly zero at the minimum and varies continuously in its neighborhood.
In contrast, the OP setup exhibits undefined SNR at the minimum, and the limiting value depends on the direction (meridian) from which the minimum is approached, as proven in Theorem~\ref{th:1}.
Although the full gradient is zero at the minimum in both setups, the behavior of stochastic gradients differs substantially.
In the UP setup, the mean norm of stochastic gradients remains positive and almost constant, with a slight decrease close to the great circles.
In the OP setup, the mean norm is exactly zero at the minimum and is continuous in its vicinity.

Figure~\ref{fig:app_sphere_stationary} illustrates the behavior of SGD under different LRs.
This hyperparameter significantly affects the dynamics in both setups.
In the UP configuration, a small ($\text{LR} = 10^{-3}$) causes the trajectory to fluctuate very close to the minimum.
Increasing the LR leads to wider oscillations and a broader stationary distribution.
In the OP configuration, all considered LRs converge directly to the minimum, but smaller LRs tend to follow paths close to the central meridian.
Together with Theorem~\ref{th:1}, this explains why higher LRs result in larger asymptotic SNR values—they cause the iterates to oscillate farther from the central meridian.

\section{Additional results on SNR}
\label{app:snr_results}

\begin{wrapfigure}{r}{0.365\textwidth}
    \vspace{-0.6cm}
    \begin{center}
    \begin{tabular}{c}
    ResNet-18 on CIFAR-10 \\
    \includegraphics[width=0.33\textwidth]{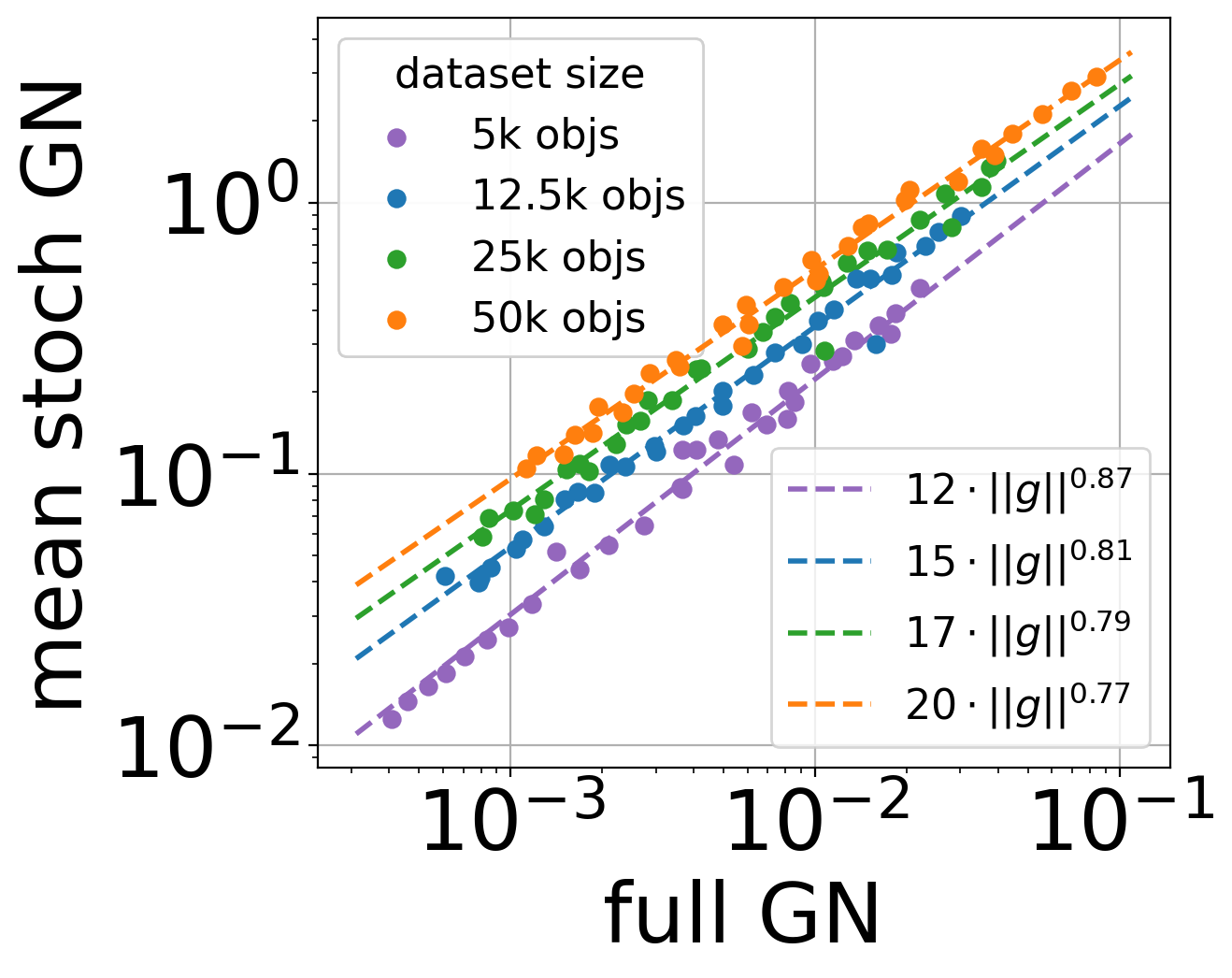}
    \end{tabular}
    \end{center}
    \caption{Phase diagram of mean norm of stochastic grad. vs. full grad. norm for OP ResNet on CIFAR-10 subsets of different size. For all four datasets we use LR equal to $2.1 \cdot 10^{-5}$. Dashed lines show power law approximation. This figure complements Figure~\ref{fig:grad_phase} from the main text.}
    \label{fig:app_grad_phase}
\end{wrapfigure}
In this section, we present supplementary results on the behavior of stochastic gradients in neural networks. 
Figures~\ref{fig:app_snr_resnet} and \ref{fig:app_snr_convnet} show the evolution of the full gradient norm, mean stochastic gradient norm, and SNR during training, complementing Figure~\ref{fig:snr_convnet} from the main text.

In the UP setting, all three metrics stabilize, consistent with the ConvNet CIFAR-10 experiment. 
However, prior to stabilization, both gradient norms grow rapidly, suggesting that their stationary distributions are concentrated in regions of high sharpness.
Notably, in CIFAR-100 experiments (for both ResNet-18 and ConvNet), SNR also increases before leveling off.
Overall, training with smaller LRs results in higher stationary gradient norms but lower stationary SNR values.

In contrast, the OP setting exhibits a steady decline in both full and stochastic gradient norms throughout training.
In ResNet-18 on CIFAR-10 and ConvNet on CIFAR-100, this decline is mirrored by a decreasing SNR.
However, a different pattern emerges in the ResNet-18 CIFAR-100 experiment, where we use a wider model ($k=48$).
Here, SNR stabilizes even as both gradient norms continue to decay, supporting our claim that higher degrees of overparameterization eventually lead to stable SNR levels.
These stationary SNR values increase with the LR, consistent with the 3D sphere toy example in Figure~\ref{fig:sphere_snr}.

We also extend the results of Figure~\ref{fig:grad_phase} by varying the CIFAR-10 subset size for ResNet-18.
The corresponding gradient phase diagram is shown in Figure~\ref{fig:app_grad_phase}.
As in the main text, the power-law exponent increases with smaller subsets, indicating that greater overparameterization improves coherence between the decay of full and stochastic gradients.

\begin{figure}
\centering
    \addtolength{\tabcolsep}{-0.4em}
    \begin{tabular}{ccc}
        \multicolumn{3}{c}{UP ResNet-18 on CIFAR-10} \\
        \includegraphics[width=0.31\textwidth]{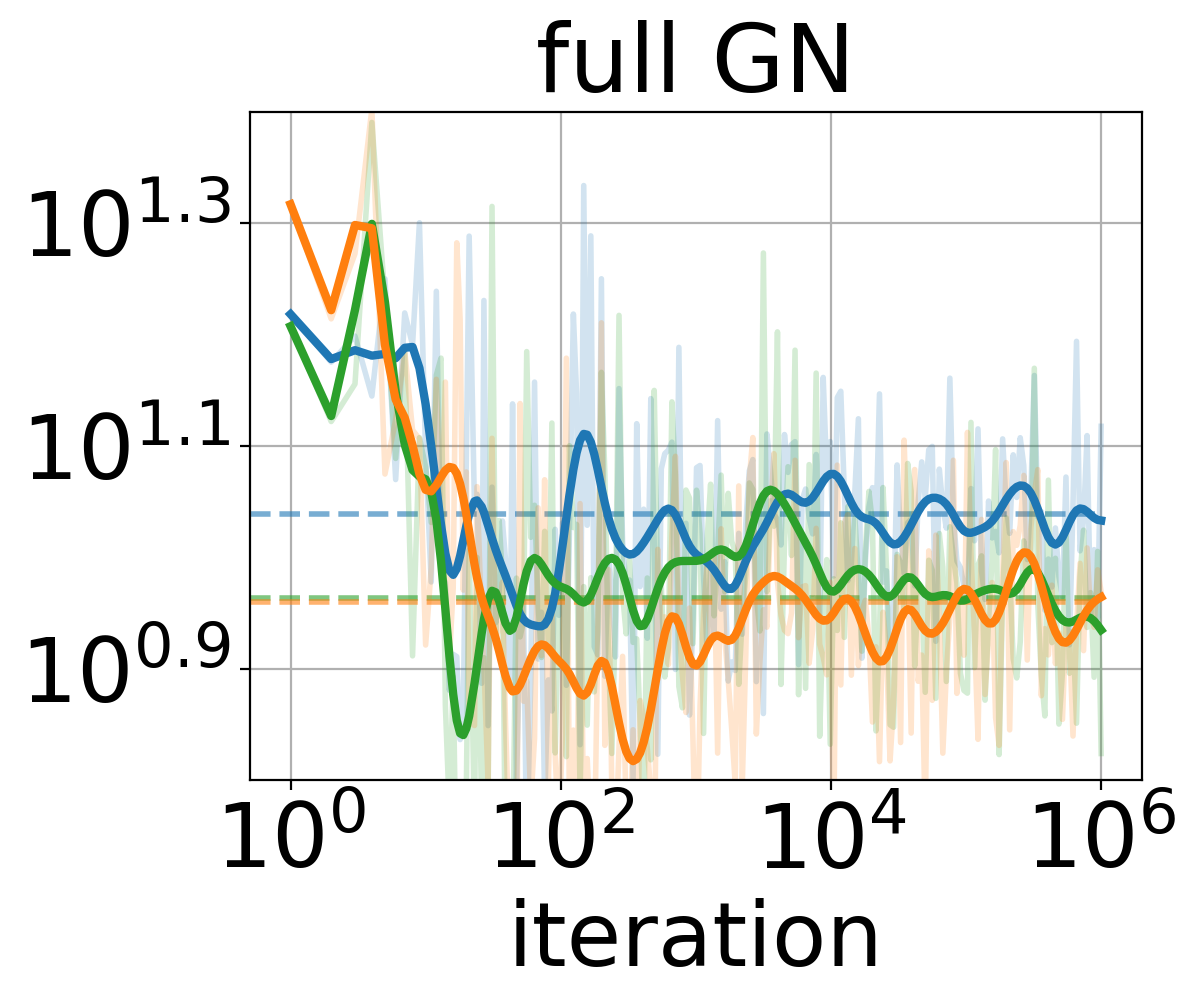} &
        \includegraphics[width=0.31\textwidth]{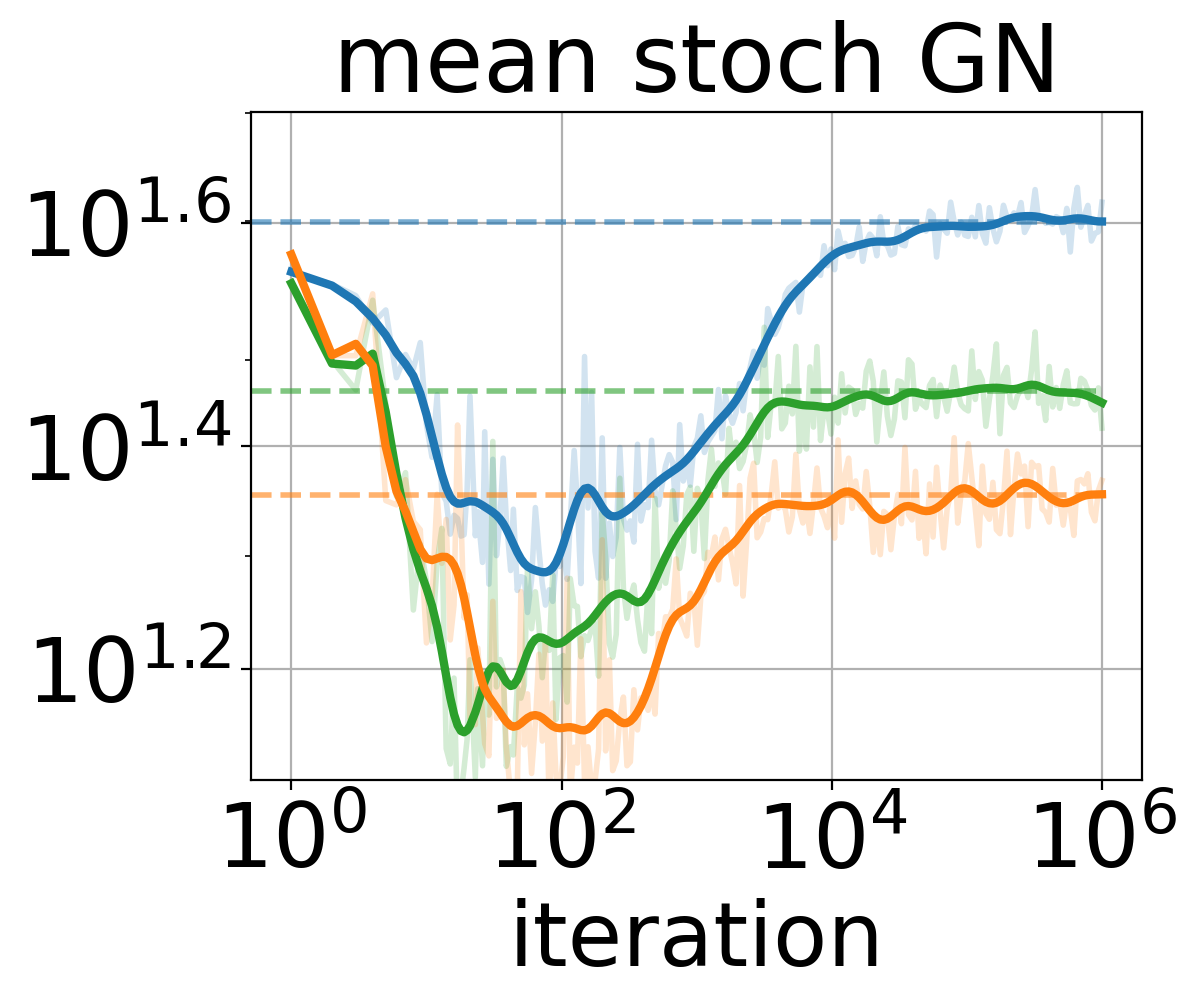} & 
        \includegraphics[width=0.325\textwidth]{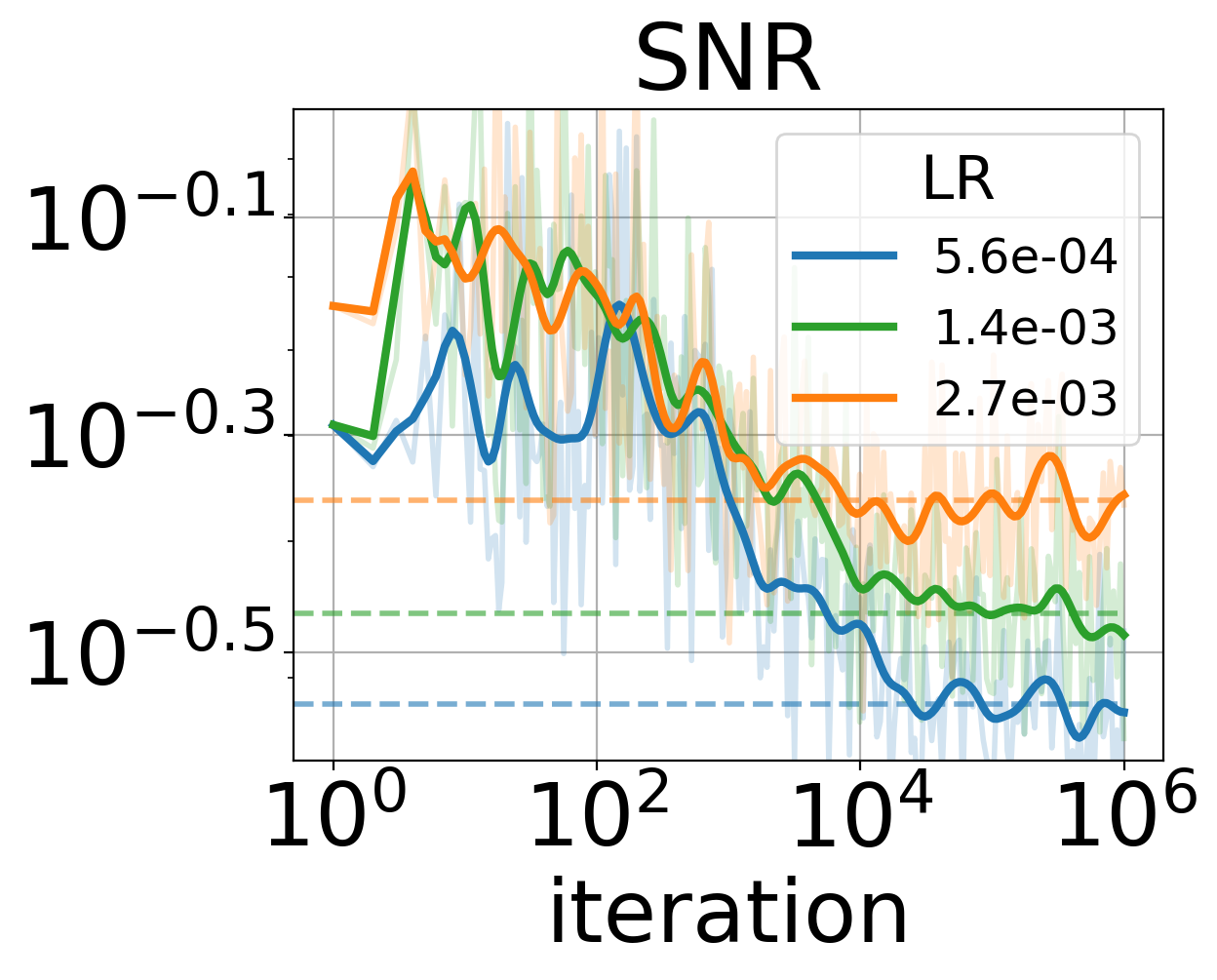} \\
        \multicolumn{3}{c}{OP ResNet-18 on CIFAR-10} \\
        \includegraphics[width=0.31\textwidth]{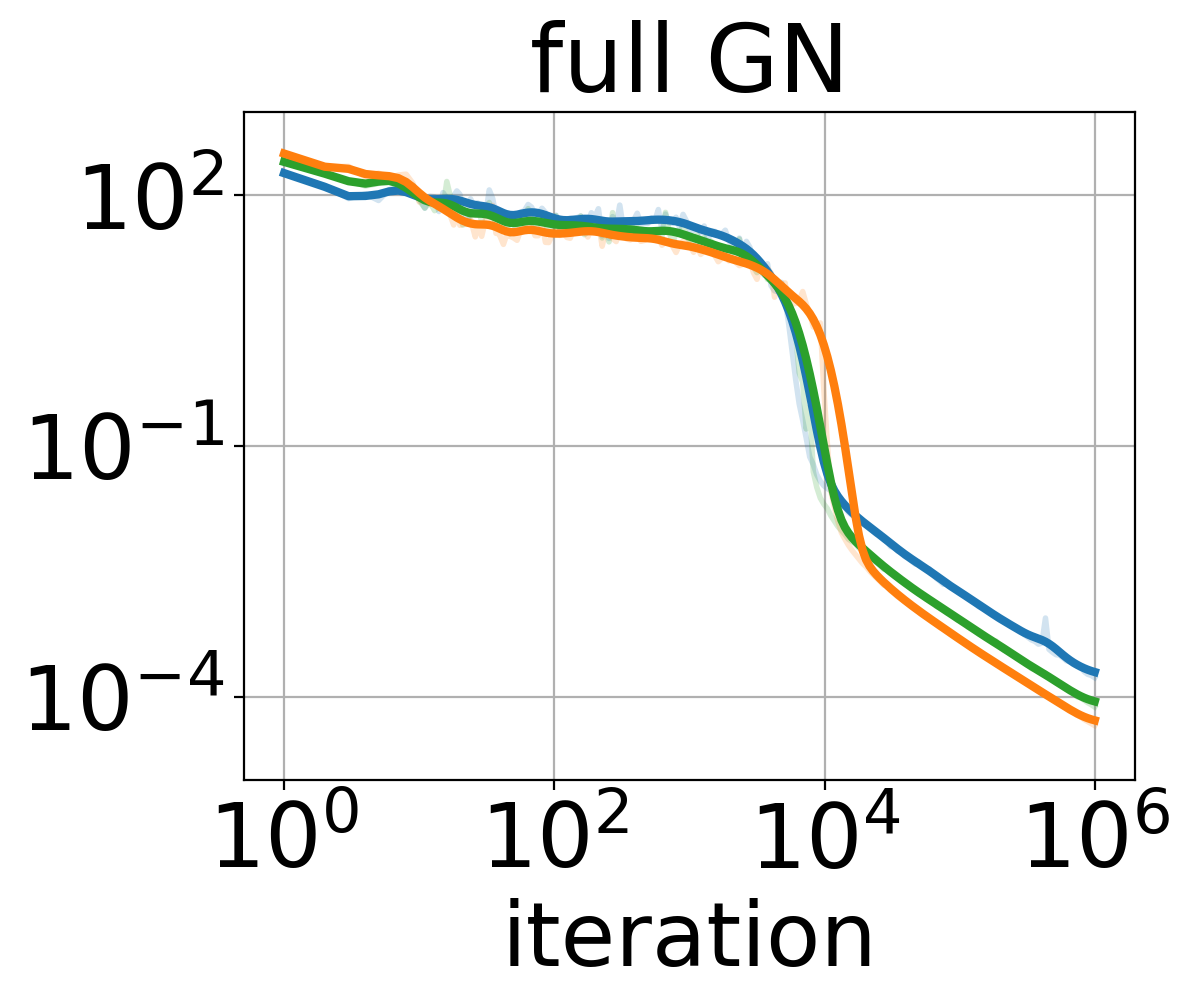} &
        \includegraphics[width=0.31\textwidth]{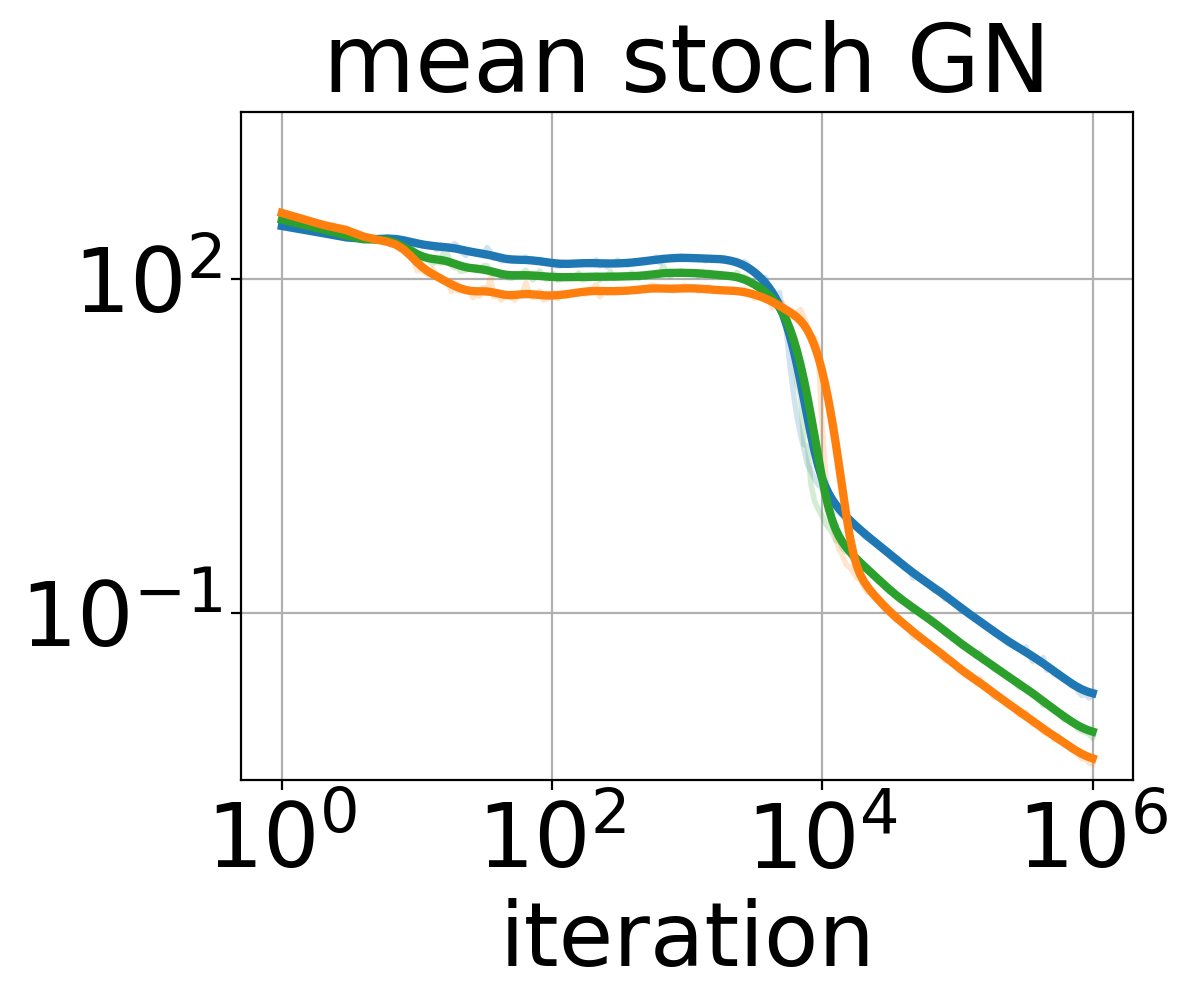} & 
        \includegraphics[width=0.315\textwidth]{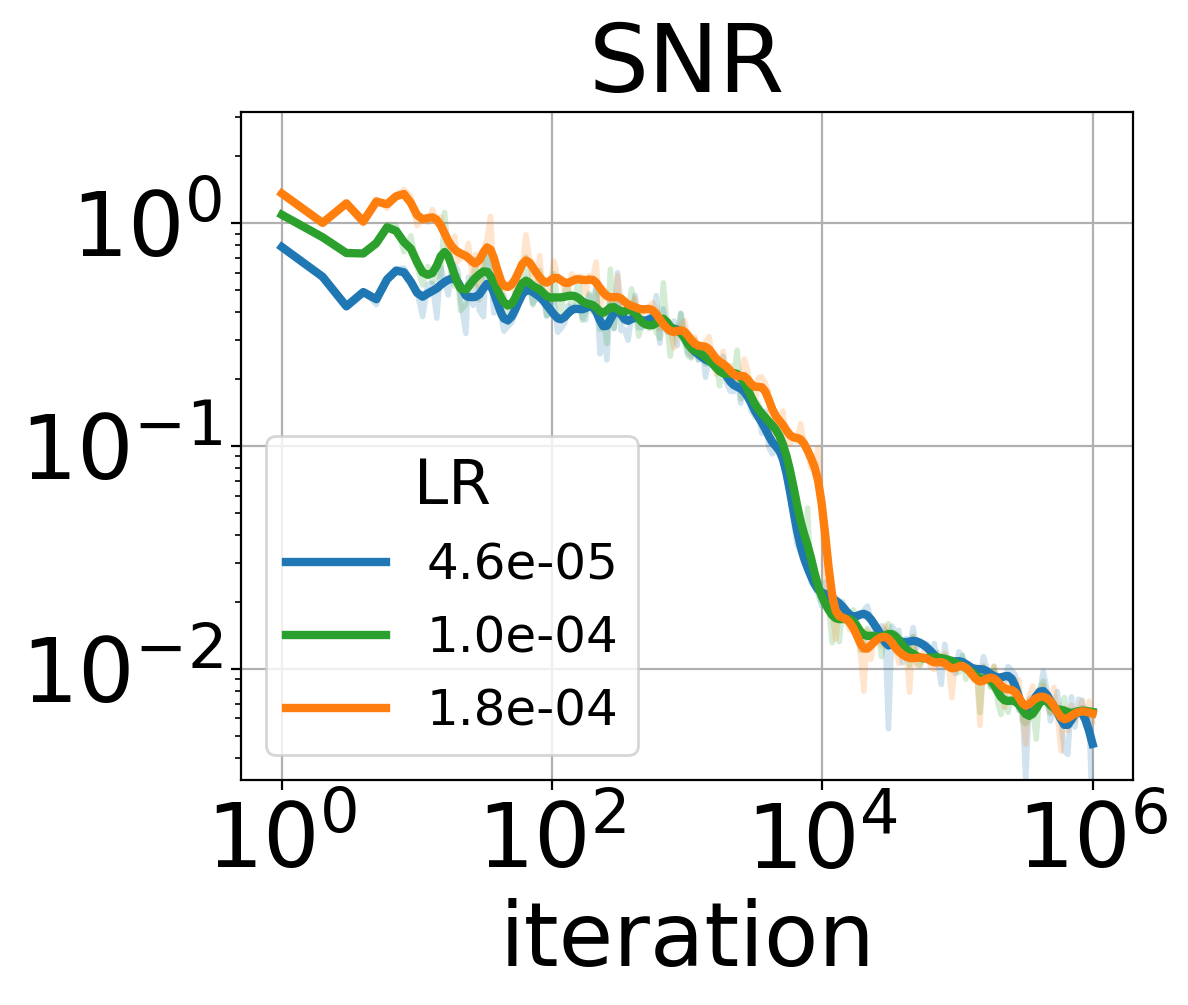} \\
        \multicolumn{3}{c}{UP ResNet-18 on CIFAR-100} \\
        \includegraphics[width=0.31\textwidth]{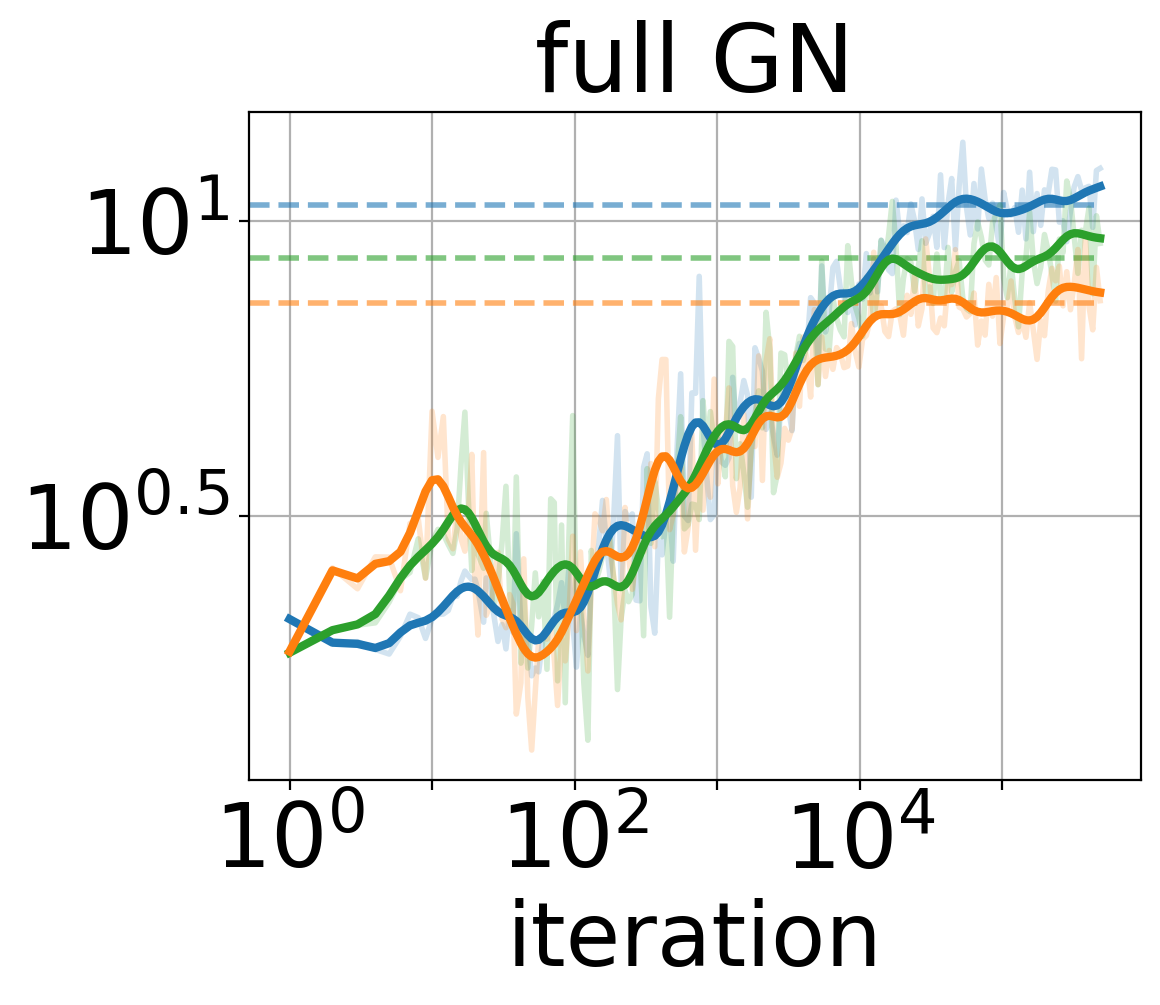} &
        \includegraphics[width=0.31\textwidth]{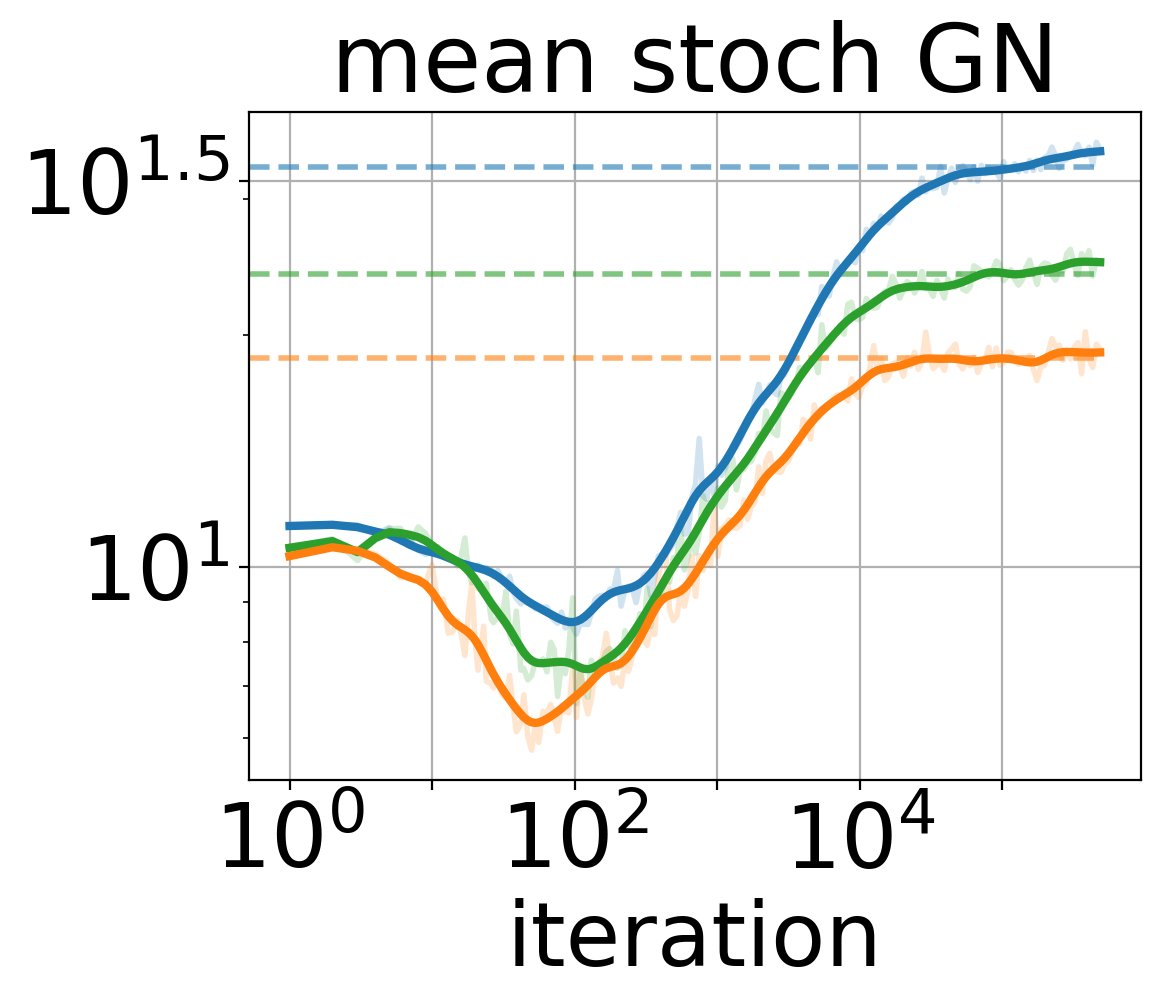} & 
        \includegraphics[width=0.325\textwidth]{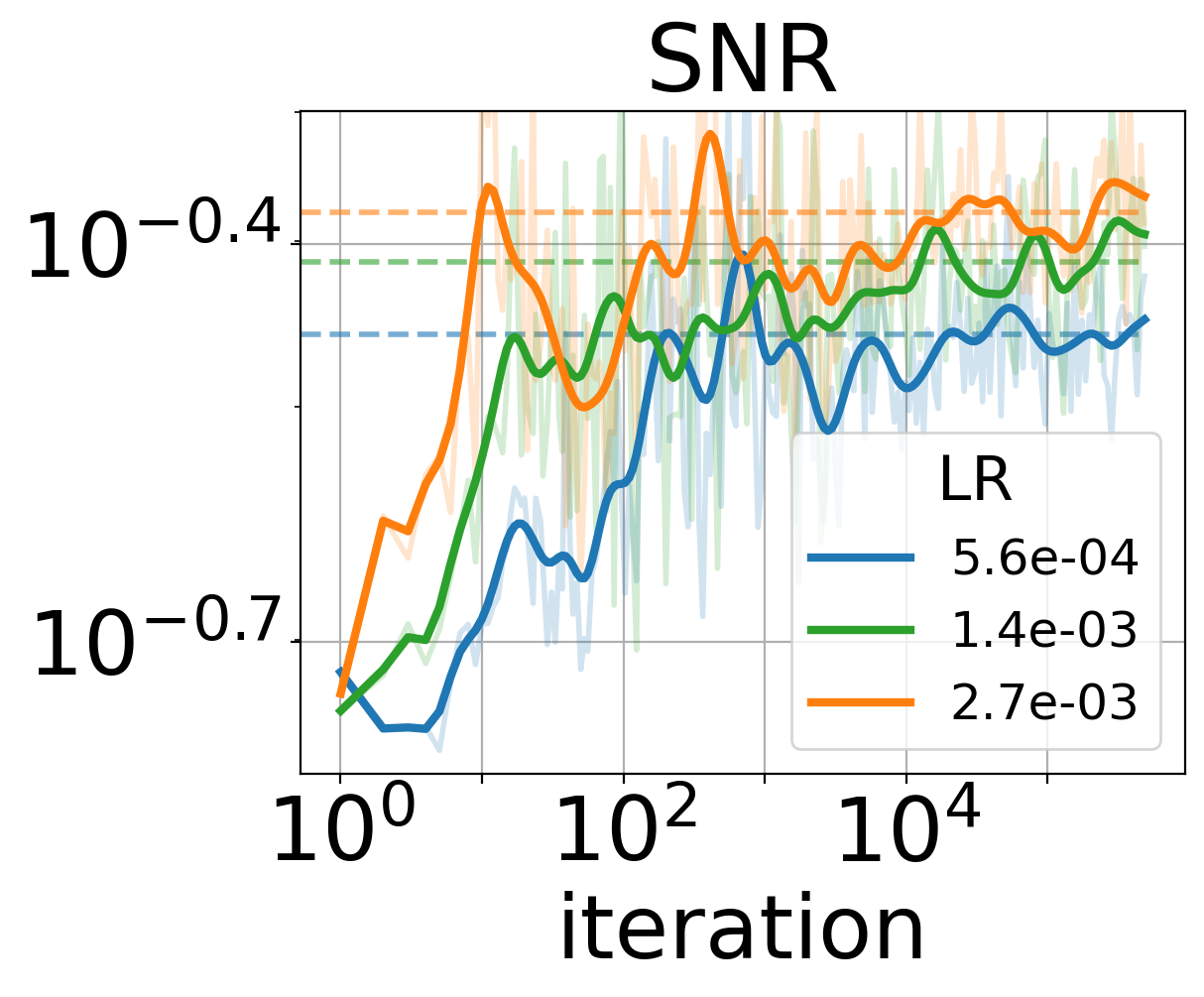} \\
        \multicolumn{3}{c}{OP ResNet-18 on CIFAR-100} \\
        \includegraphics[width=0.31\textwidth]{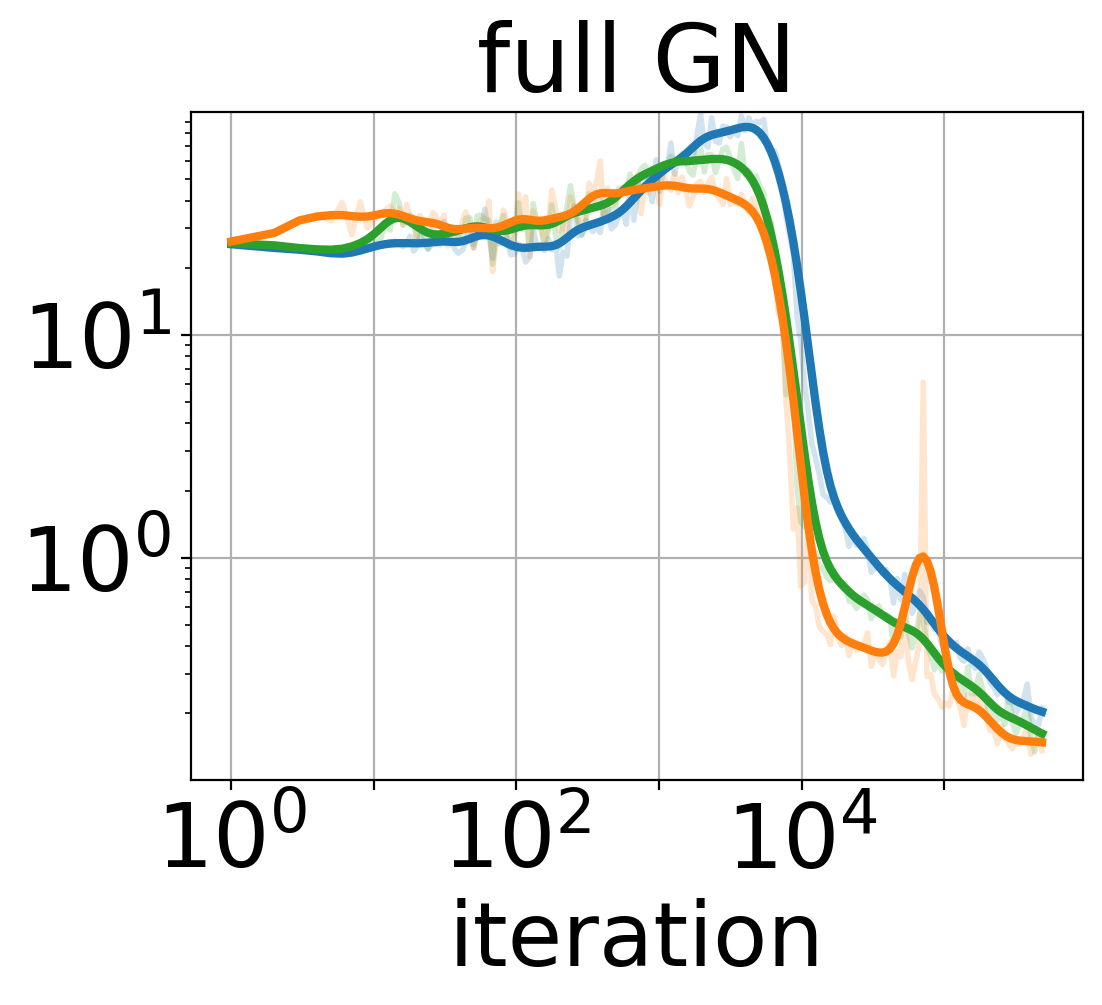} &
        \includegraphics[width=0.31\textwidth]{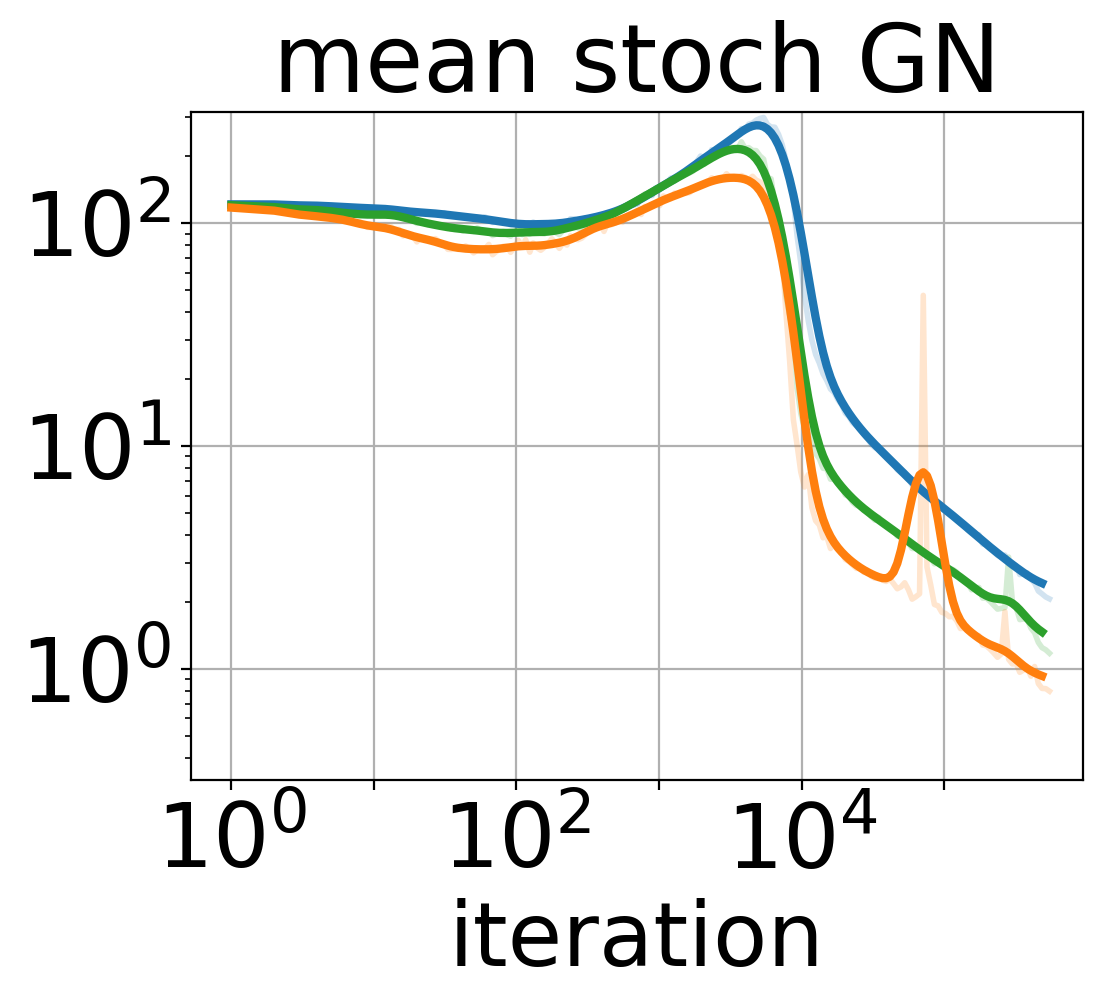} & 
        \includegraphics[width=0.335\textwidth]{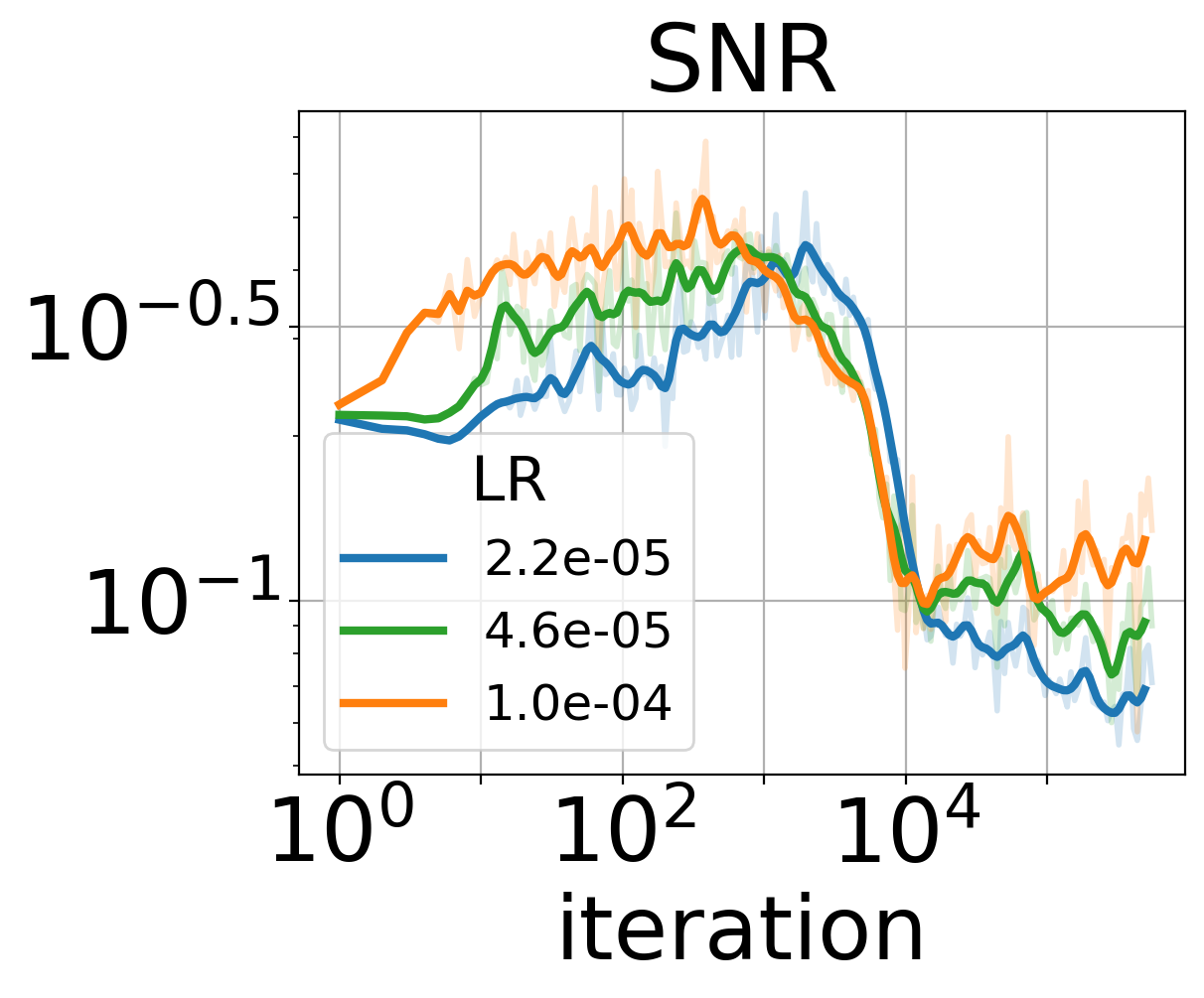}  \\
    \end{tabular}
    \caption{Norm of full gradient (left), mean norm of stoch. gradient (center) and SNR (right) for small LRs of UP and OP setups for ResNet-18 on CIFAR-10 and CIFAR-100 datasets. Dashed lines indicate stationary values for the UP setup. This figure extends Figure~\ref{fig:snr_convnet} from the main text.}
    \label{fig:app_snr_resnet}
\end{figure}

\begin{figure}
\centering
    \addtolength{\tabcolsep}{-0.4em}
    \begin{tabular}{ccc}
        \multicolumn{3}{c}{UP ConvNet on CIFAR-100} \\
        \includegraphics[width=0.31\textwidth]{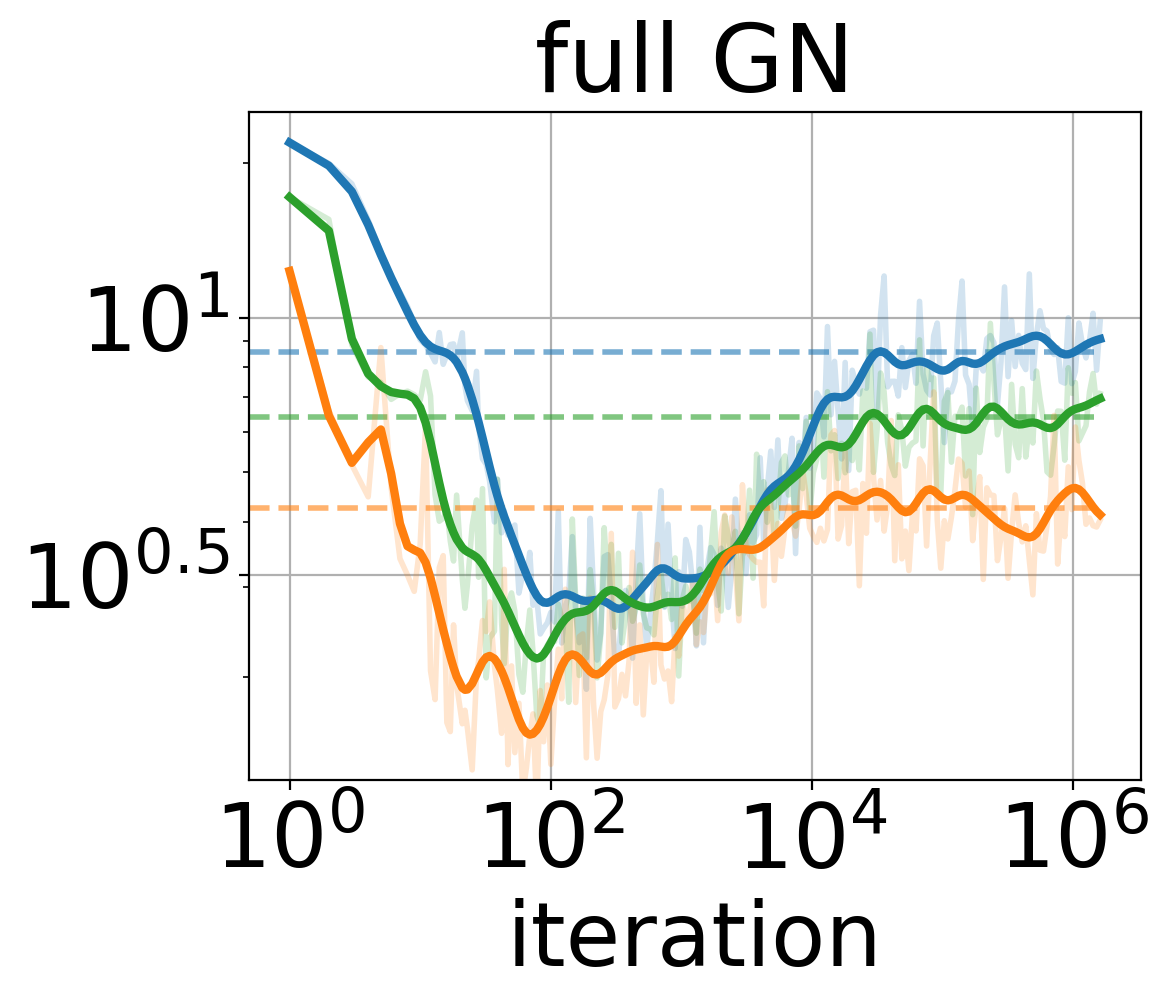} &
        \includegraphics[width=0.31\textwidth]{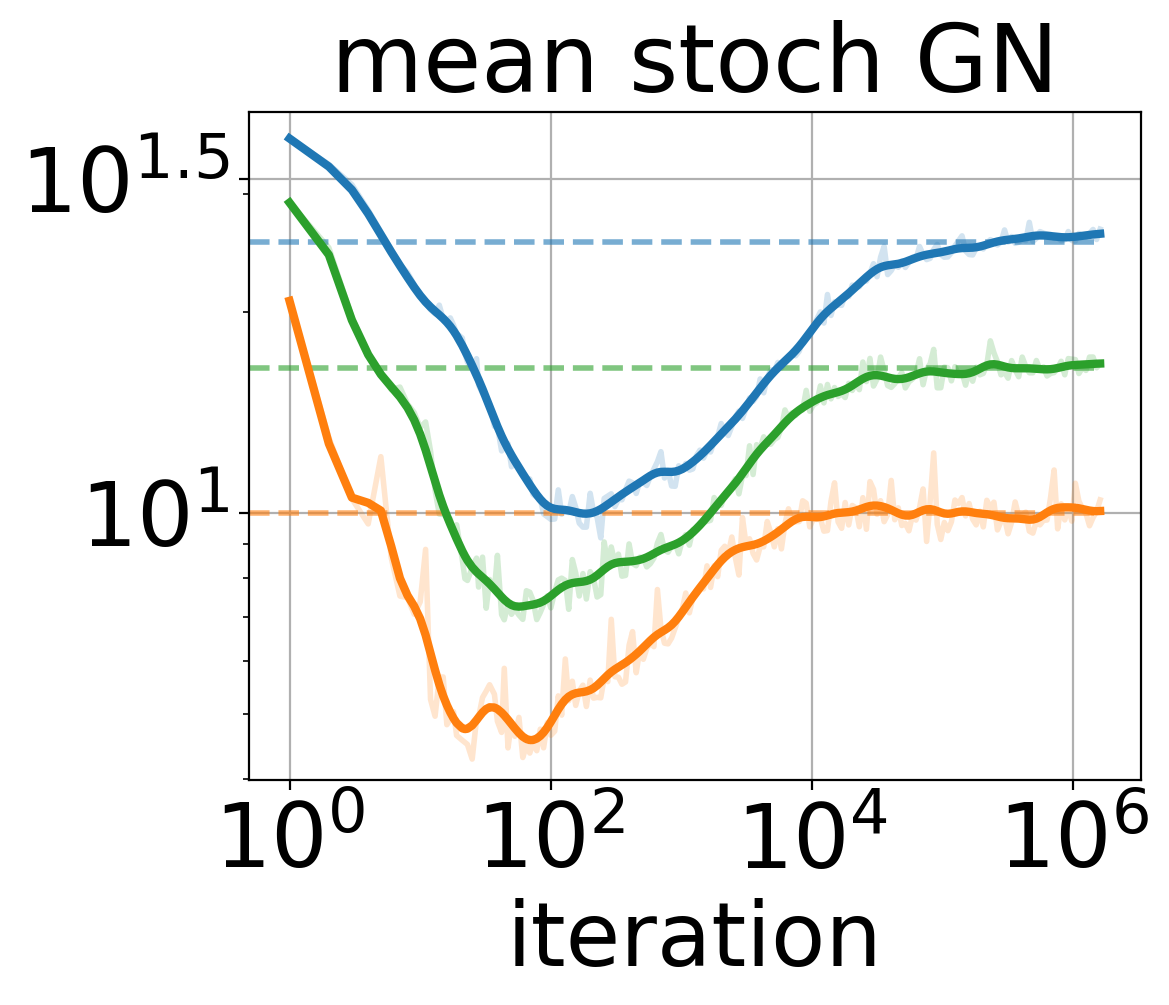} & 
        \includegraphics[width=0.325\textwidth]{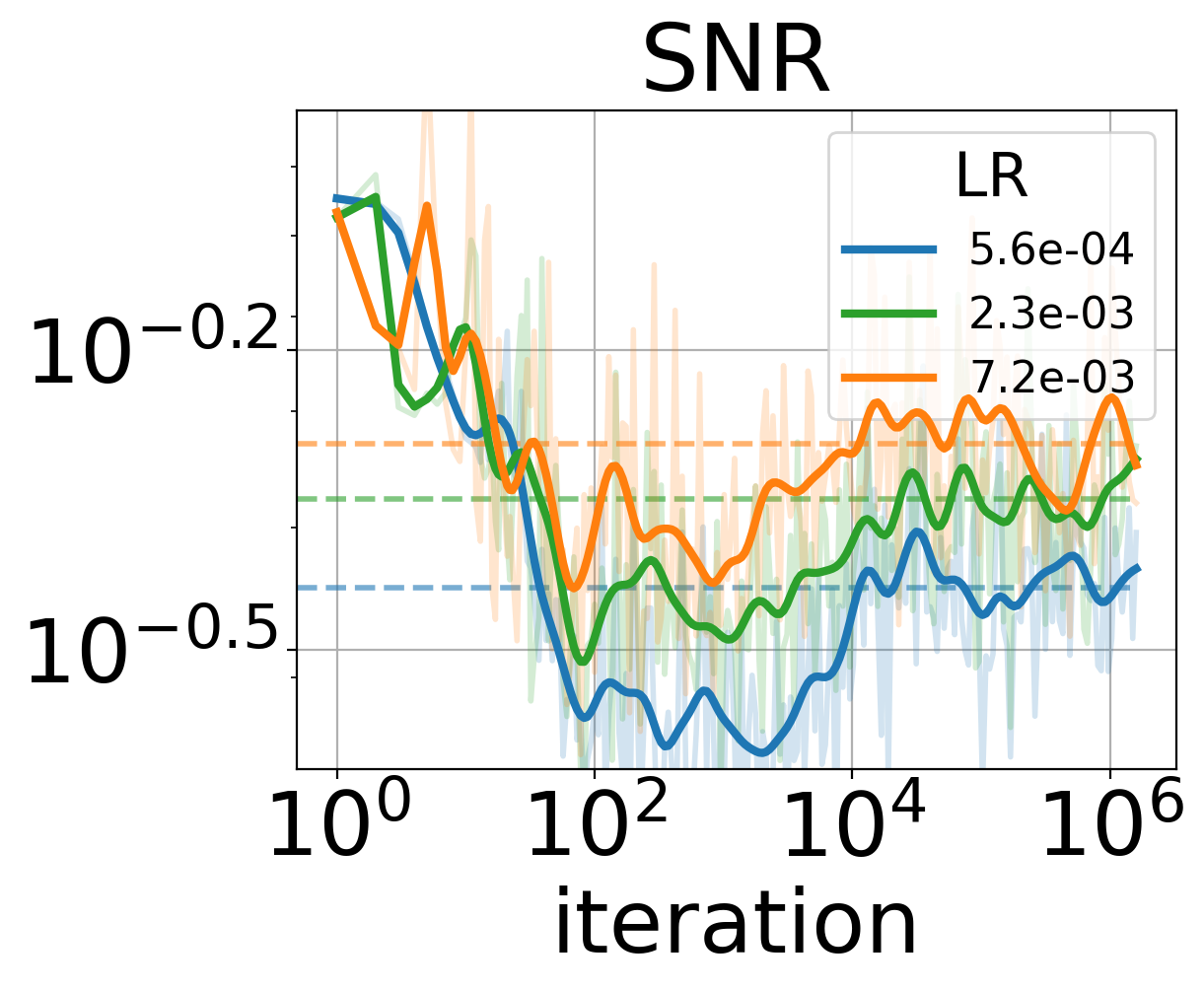} \\
        \multicolumn{3}{c}{OP ConvNet on CIFAR-100} \\
        \includegraphics[width=0.31\textwidth]{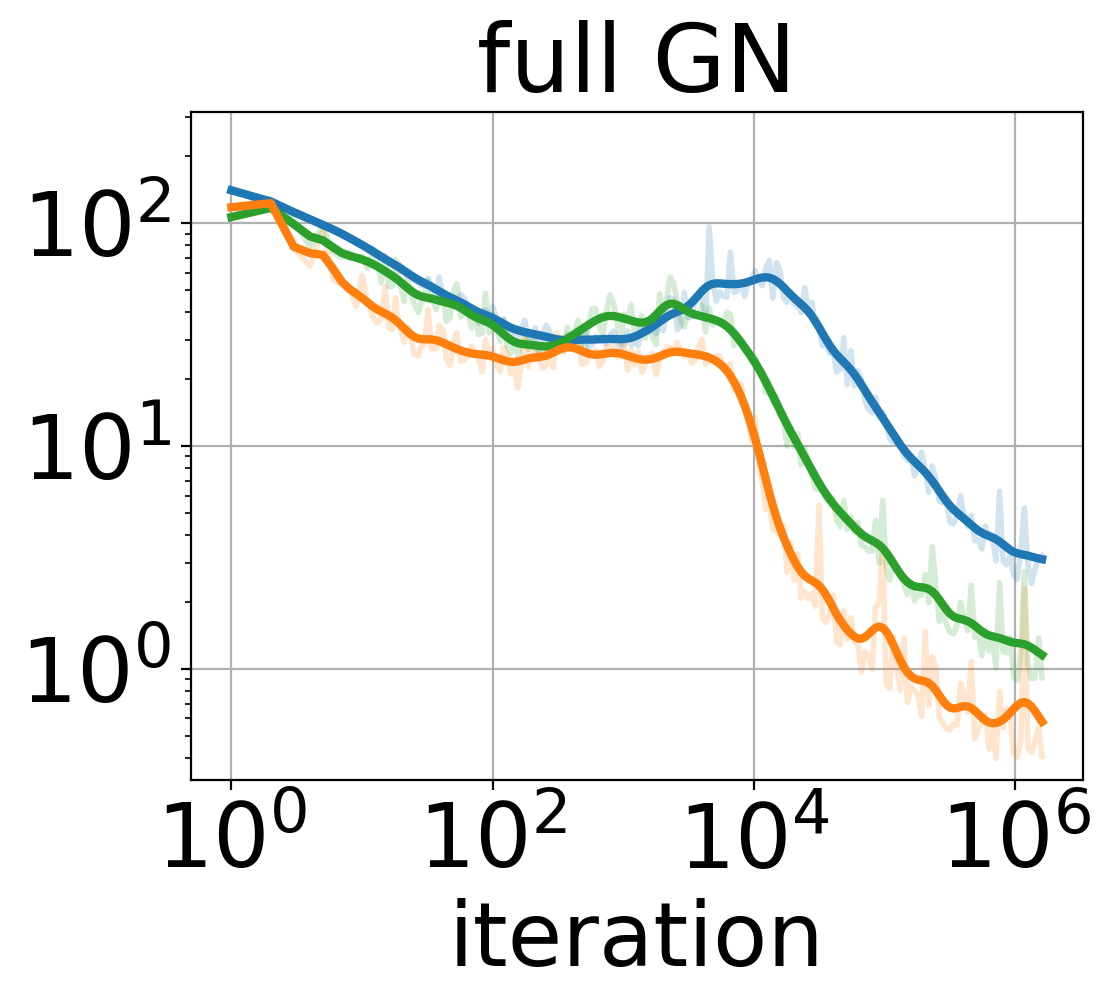} &
        \includegraphics[width=0.31\textwidth]{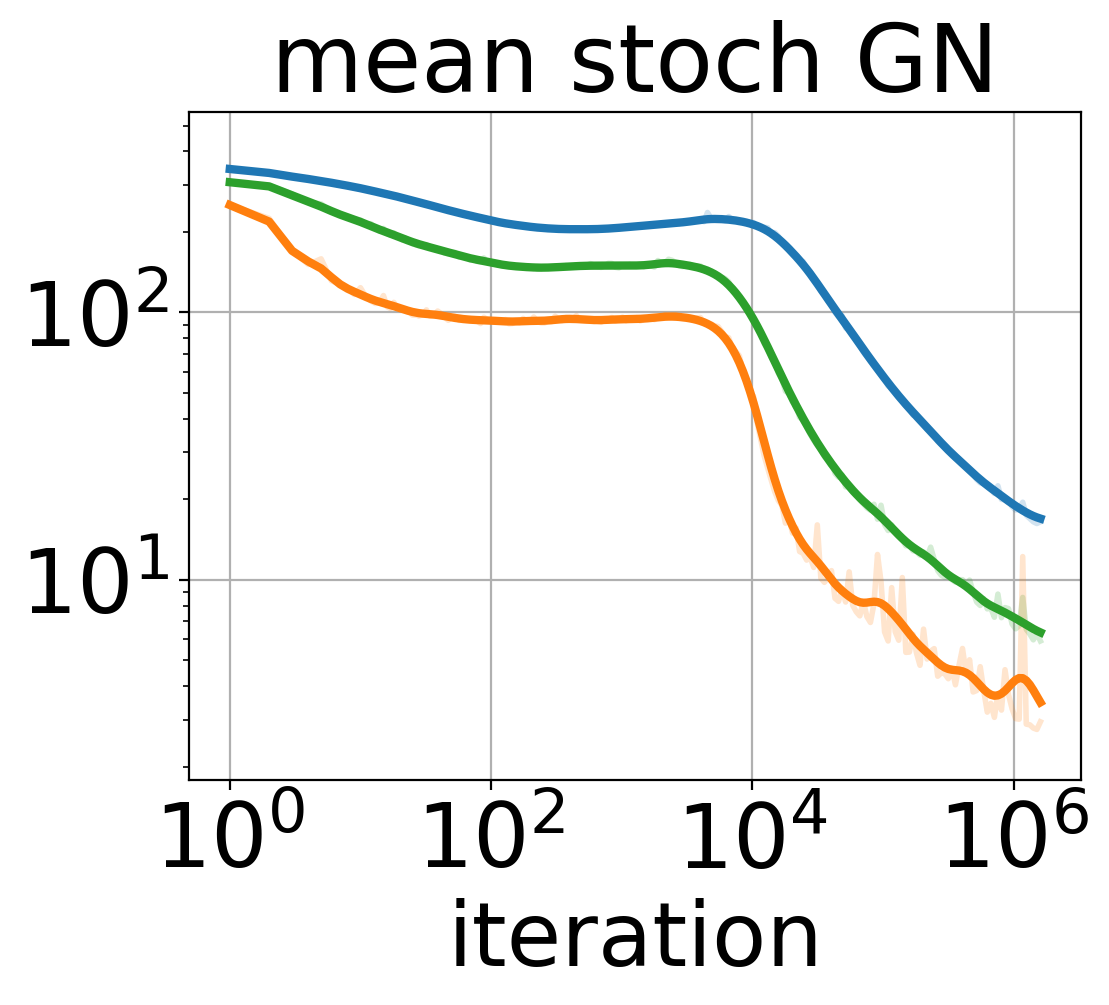} & 
        \includegraphics[width=0.335\textwidth]{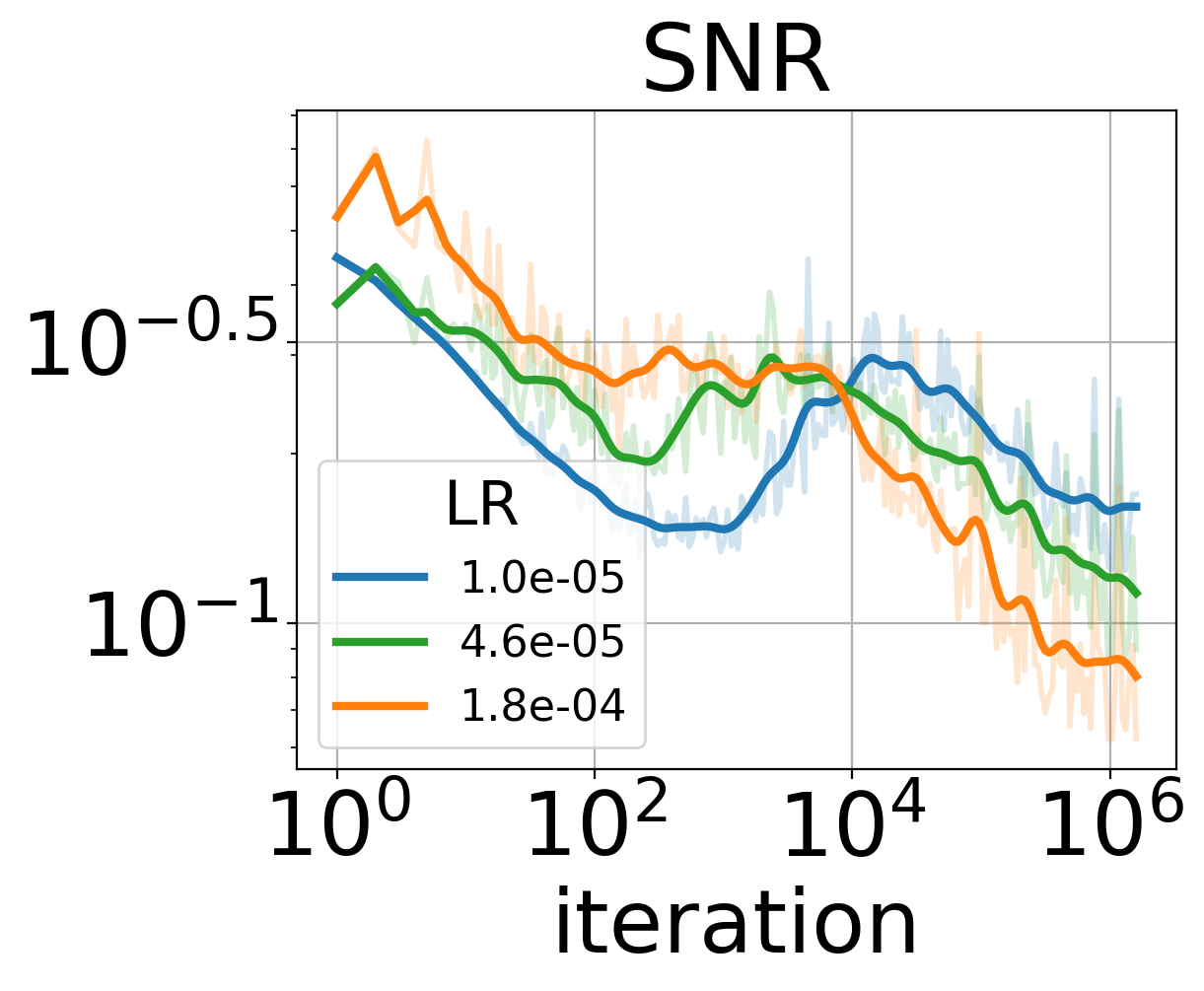} 
    \end{tabular}
    \caption{Norm of full gradient (left), mean norm of stoch. gradient (center) and SNR (right) for small LRs of UP and OP setups for ConvNet on CIFAR-100. Dashed lines indicate stationary values for the UP setup. This figure extends Figure~\ref{fig:snr_convnet} from the main text.}
    \label{fig:app_snr_convnet}
\end{figure}


\end{document}